\documentclass[11pt,a4paper]{article}%
\usepackage[T1]{fontenc}
\usepackage{libertine}
\usepackage[table]{xcolor}
\usepackage[american]{babel}
\usepackage{xfrac}
\usepackage{bbm}
\usepackage{authblk}
\usepackage{amsthm}
\usepackage{bm}
\usepackage[format=plain,font=it]{caption}
\usepackage{algorithm}
\usepackage{algpseudocode}
\usepackage{pdflscape}

\usepackage{geometry}
\usepackage[round]{natbib} 
\usepackage{mathtools} 
\usepackage{booktabs} 
\usepackage{tikz} 
\usepackage{multirow}

\usepackage{float}
\usepackage{subfig}
\usepackage{amsmath}
\usepackage{amsfonts}
\usepackage{comment}
\usepackage{dsfont}

\DeclareMathOperator{\g}{g}

\usepackage{array}
\newcolumntype{C}[1]{>{\centering\arraybackslash}m{#1}}
\newcolumntype{R}[1]{>{\raggedleft\arraybackslash}m{#1}}
\newcolumntype{L}[1]{>{\raggedright\arraybackslash}m{#1}}

\title{Probabilistic Scores of Classifiers, \\ Calibration is not Enough\thanks{Emmanuel Flachaire and Ewen Gallic acknowledge that the project leading to this publication has received funding from the French government under the ``France 2030'' investment plan managed by the French National Research Agency (reference: ANR-17-EURE-0020) and from Excellence Initiative of Aix-Marseille University -- A*MIDEX.\\Replication codes and companion e-book: \href{https://github.com/fer-agathe/scores-classif-calibration}{https://github.com/fer-agathe/scores-classif-calibration}}}

\theoremstyle{plain}
\newtheorem{theorem}{Theorem}[section]
\newtheorem{proposition}[theorem]{Proposition}
\newtheorem{lemma}[theorem]{Lemma}

\theoremstyle{definition}

\theoremstyle{remark}

\usepackage{graphicx,pstricks,psfrag}
\definecolor{bleu}{RGB}{0,101,189}
\definecolor{vert}{HTML}{004D40}
\definecolor{rose}{HTML}{D81B60}
\definecolor{bleuTOL}{HTML}{332288}
\definecolor{wongBlack}{RGB}{0,0,0}
\definecolor{wongGold}{RGB}{230, 159, 0}
\definecolor{wongLightBlue}{RGB}{86, 180, 233}
\definecolor{wongGreen}{RGB}{0, 158, 115}
\definecolor{wongYellow}{RGB}{240, 228, 66}
\definecolor{wongBlue}{RGB}{0, 114, 178}
\definecolor{wongOrange}{RGB}{213, 94, 0}
\definecolor{wongPurple}{RGB}{204, 121, 167}
\definecolor{colUncalibrated}{RGB}{191, 191, 191}
\definecolor{colRecalibrated}{RGB}{197, 214, 231}

\definecolor{bleuTOL}{HTML}{332288}
\definecolor{vertTOL}{HTML}{117733}
\definecolor{vertClairTOL}{HTML}{44AA99}
\definecolor{bleuClairTOL}{HTML}{88CCEE}
\definecolor{sableTOL}{HTML}{DDCC77}
\definecolor{parmeTOL}{HTML}{CC6677}
\definecolor{magentaTOL}{HTML}{AA4499}
\definecolor{roseTOL}{HTML}{882255}

\usepackage{hyperref}
\hypersetup{
    plainpages=false,
    bookmarksopen,
    bookmarksnumbered,
    unicode=true,          
    pdftoolbar=true,        
    pdfmenubar=true,        
    pdffitwindow=false,     
    pdfstartview={FitH},    
    pdftitle={Probabilistic Scores of Classifiers, Calibration is not Enough},    
    pdfauthor={},     
    pdfkeywords={calibration, epistemic uncertainty, predictive uncertainty, scores, divergence}, 
    pdfnewwindow=true,      
    colorlinks=true,       
    linkcolor=rose,          
    citecolor={wongBlue},        
    filecolor={rose},      
    urlcolor={rose},           
}

\author[1]{Agathe~Fernandes~Machado\thanks{Corresponding author: \href{mailto:fernandes_machado.agathe@courrier.uqam.ca}{fernandes\_machado.agathe@courrier.uqam.ca}}}
\author[1]{Arthur~Charpentier}
\author[2]{Emmanuel~Flachaire}
\author[2]{Ewen~Gallic}
\author[3]{Fran\c{c}ois Hu}

\affil[1]{%
    \footnotesize Département de Mathématiques\\
    Université du Québec à Montréal\\
    Montréal, Québec, Canada
}
\affil[2]{%
    \footnotesize Aix Marseille Univ, CNRS, AMSE\\
    Marseille, France
}
\affil[3]{%
    \footnotesize Milliman France\\
    14 Av. de la Grande Armée, 75017 Paris, France
}

\usepackage[misc]{ifsym}
\makeatletter
\def\@fnsymbol#1{%
   \ifcase#1\or
   \TextOrMath ~ \dagger\or
   \TextOrMath {\footnotesize\Letter} \dagger\or
   \TextOrMath \textdaggerdbl \ddagger \or
   \TextOrMath \textsection  \mathsection\or
   \TextOrMath \textparagraph \mathparagraph\or
   \TextOrMath \textbardbl \|\or
   \TextOrMath {\textdagger\textdagger}{\dagger\dagger}\or
   \TextOrMath {\textdaggerdbl\textdaggerdbl}{\ddagger\ddagger}\else
   \@ctrerr \fi
}
\makeatother

\usepackage{fancyhdr} 
\newcommand{\authornames}{\footnotesize\textsc{Fernandes Machado, Charpentier, Flachaire, Gallic, Hu}}
\usepackage{etoolbox}
\makeatletter
\patchcmd{\NAT@test}{\else \NAT@nm}{\else \NAT@nmfmt{\NAT@nm}}{}{}

\DeclareRobustCommand\citepos
  {\begingroup
   \let\NAT@nmfmt\NAT@posfmt
   \NAT@swafalse\let\NAT@ctype\z@\NAT@partrue
   \@ifstar{\NAT@fulltrue\NAT@citetp}{\NAT@fullfalse\NAT@citetp}}

\let\NAT@orig@nmfmt\NAT@nmfmt
\def\NAT@posfmt#1{\NAT@orig@nmfmt{#1's}}
\makeatother

\begin{document}

\maketitle
\thispagestyle{empty}

\begin{abstract}
In binary classification tasks, accurate representation of probabilistic predictions is essential for various real-world applications such as predicting payment defaults or assessing medical risks. The model must then be well-calibrated to ensure alignment between predicted probabilities and actual outcomes. However, when score heterogeneity deviates from the underlying data probability distribution, traditional calibration metrics lose reliability, failing to align score distribution with actual probabilities.
In this study, we highlight approaches that prioritize optimizing the alignment between predicted scores and true probability distributions over minimizing traditional performance or calibration metrics. When employing tree-based models such as Random Forest and XGBoost, our analysis emphasizes the flexibility these models offer in tuning hyperparameters to minimize the Kullback-Leibler (KL) divergence between predicted and true distributions. Through extensive empirical analysis across 10 UCI datasets and simulations, we demonstrate that optimizing tree-based models based on KL divergence yields superior alignment between predicted scores and actual probabilities without significant performance loss. In real-world scenarios, the reference probability is determined \textit{a priori} as a Beta distribution estimated through maximum likelihood. Conversely,  minimizing traditional calibration metrics may lead to suboptimal results, characterized by notable performance declines and inferior KL values. Our findings reveal limitations in traditional calibration metrics, which could undermine the reliability of predictive models for critical decision-making.
\end{abstract}

\section{Introduction}

When quantifying the risks associated with decisions made using a classifier, it is essential that the scores returned by the classifier accurately reflect the underlying probability of the event in question. The model must then be well-calibrated \citep{brier_1950,murphy1972scalar,dawid1982well}. This is particularly relevant in contexts such as assessing risks, in medicine, or in economic applications. Tree-based machine learning techniques like Random Forests (RF) \citep{ho_1995_random,Breiman_2001_rf} and XGBoost \citep{friedman_2001_xgb} are increasingly popular for risk estimation by practitioners, though these models are not inherently well-calibrated \citep{caliForest2020}. Even when they are calibrated, it may not be sufficient because the score distribution might not align with actual probabilities. This misalignment can render calibration metrics unreliable since they are assessed only within the prediction range, which may not represent the full range of the underlying data distribution.

Scholars, as in \citet{rodgers2000inverse}, acknowledge the difficulty of knowing the "true distribution" of some underlying risk factor, recommending that practitioners make a "{\em reasonable estimate of a probability density function consistent with all our knowledge, one that is least committal about the state but consistent with whatever more or less detailed understanding we may have of the state vector prior to the measurement(s)}". In genetics, it could be possible to have some (biological) knowledge about the ``true prior distribution" of some quantity, capturing its distribution in the entire heterogeneous population, as in \citet{galanti2021pheniqs}, or on spectral measures in physics applications, as in \citet{cressie2018mission}. 

\paragraph{Related Work and Metrics}Traditionally, binary classifiers are evaluated using metrics like AUC-ROC (or AUC). However, these metrics ensure neither that the classifier's scores are well-calibrated, nor reflect the underlying probability distribution \citep{widmann2020calibration}. 
To address the first challenge, machine learning literature has shifted towards introducing calibration metrics and methodologies to recalibrate these scores \citep{brier_1950, dawid1982well, naeini2015obtaining, kull2017EJS}. 
And despite the abundance of calibration metrics, none, to our knowledge, evaluate their effectiveness using synthetic data with known probabilities. To address this gap, this paper demonstrates through empirical simulations that optimizing KL divergence could be a more effective model selection approach \citep{der2009aleatory,shaker2020aleatoric,schweighofer2023quantification,gruber2023sources}. In particular, the model selected by optimizing KL divergence does not necessarily correspond to the one that minimizes existing accuracy and calibration metrics, confirming that these metrics might not be a sufficient goal in practice.

In a real-world context where the distribution of underlying probabilities is no longer directly observable, we suppose that some expert based opinion can be considered. For convenience, in this paper, we adopt an approach where a Beta distribution is selected as a reference, and study accuracy, calibration and score dispersion over 10 UCI datasets. 

\paragraph{Contributions and Findings} Let us now summarize our main contributions:
\begin{itemize}
    \item We show that the proposed calibration metrics alone are insufficient in practice to ensure the reliability of scores for crucial decision-making particularly when the score distribution lacks heterogeneity compared to the underlying data distribution.
    \item We demonstrate through controlled simulations that optimizing KL divergence is a more effective model selection approach.
    \item Our simulations in both synthetic and real-world data reveal that minimizing traditional performance and calibration metrics for the target variable may result in suboptimal outcomes, marked by significant performance declines and inferior KL values.
\end{itemize}

We start by introducing classical calibration metrics for binary classification tasks and discussing score heterogeneity. We then highlight deficiencies in these metrics using numerical simulations and synthetic datasets. Finally, we conduct extensive experiments to demonstrate the effectiveness of KL-based metrics on both synthetic and real-world datasets.

\section{Calibration of a Binary Classifier}\label{sec:measure-calibration}

Consider a binary variable $y$ indicating event occurrence (1 for occurrence, 0 otherwise). In this context, the event probability $p_i$ depends on individual characteristics $\mathbf{x}_i$, represented as $p_i = s(\mathbf{x}_i)$, where $i=1,\ldots, n$ denotes individuals in a sample of size $n$. The goal is to estimate this probability $\hat{s}(\mathbf{x}_i)\in [0,1]$ using models like Generalized Linear Models (GLMs) or Machine Learning (ML) models such as Trees or ensemble models like Random Forest (RF) or Boosting (XGBoost). Given our interest in predicting a continuous variable $s(\boldsymbol{x})$, alongside a binary variable $y$, we adopt a regression framework for estimating these scores, denoted as $\hat{s}(\boldsymbol{x})$.

However, if the model is not well-calibrated, the score cannot be interpreted as ``the true underlying probability'' \citep{van1995fine,kull2014reliability,konek2016probabilistic}. For a binary variable $Y$, a model is well-calibrated when \citep{Schervish_1989_AS}%
\begin{equation}
\mathbb{E}[Y \mid \hat{s}(\bm{X}) = p] = p, \quad \forall p \in [0,1].\label{eq-calib-E}
\end{equation}

For example, if \(y\) represents a credit default where \(y=1\) indicates default and \(y=0\) indicates no default, a model predicting credit default is well-calibrated if the predicted probabilities \(\hat{s}(\mathbf{x})\) match the actual observed probabilities of default. If the model predicts a probability of 0.8 for an individual, the actual default frequency for individuals with this predicted probability should be close to 80\%.

\paragraph{Calibration of Well-Specified Logistic Regression}
It should be mentioned that conditioning by \(\{\hat{s}(\mathbf{X})=p\}\) leads to the concept of (local) calibration; however, as discussed by \citet{bai2021don}, \(\{\hat{s}(\mathbf{X})=p\}\) is \textit{a.s.} a null mass event for standard regression models, such as a logistic regression. Thus, calibration should be understood in the sense that 
\begin{equation}
\mathbb{E}[Y \mid \hat{s}(\mathbf{X}) = p]\overset{a.s.}{\to} p\text{ when }n\to\infty,
\end{equation}
meaning that, asymptotically, the model is well-calibrated, or locally well-calibrated in \(p\), for any \(p\). By the dominated convergence theorem, it also indicates that the model is ``on average'' well-calibrated. To account for this, we consider multiple replications of finite samples in the simulation study in Section~\ref{sec:scores-heterogeneity}. In a well-specified logistic regression model, this property is satisfied.

\begin{proposition}\label{prop:logistic}
Consider a dataset $\{(y_i,\mathbf{x{_i}})\}$, where 
{$\mathbf{x}$}
are $k$ features ($k$ being fixed), so that $Y|\boldsymbol{X}=\mathbf{x} \sim \mathcal{B}\big(s(\mathbf{x})\big)$ where
$
s(\mathbf{x})={\big[1+\exp[-(\beta_0+\mathbf{x}^\top\boldsymbol{\beta})]\big]}
$. Let $\widehat{\beta}_0$ and $\widehat{\boldsymbol{\beta}}$ denote maximum likelihood estimators. Then, for any $\mathbf{x}$, the score is defined as 
$
\hat{s}(\mathbf{x})={\big[1+\exp[-(\hat\beta_0+\mathbf{x}^\top\hat{\boldsymbol{\beta}})]\big]^{-1}}
$
is well-calibrated in the sense that
$
\mathbb{E}[Y \mid \hat{s}(\mathbf{X}) = p]\overset{a.s.}{\to} p\text{ as }n\to\infty.
$
\end{proposition}
\begin{proof}
    Proof in Appendix~\ref{proof:prop}. If $k$ were increasing with $n$, \cite{bai2021don} showed that logistic regression is over-confident. However, assuming $k$ is fixed provides a complementary perspective.
\end{proof}

To evaluate model calibration, the literature offers both graphical techniques and metrics \citep{brier_1950, NEURIPS2019_f8c0c968, naeini2015obtaining, NEURIPS2022_33d6e648, pmlr-v119-zhang20k}. Given the continuous nature of the score with a null mass event, various methods have emerged. Graphical methods mainly involve estimating a calibration curve.

\paragraph{Calibration Curve}
In the binary case, based on Eq.~\ref{eq-calib-E}, consider
\begin{equation}
\g :
\begin{cases}
[0,1] \rightarrow [0,1]\\
p \mapsto \g(p) := \mathbb{E}[Y \mid \hat{s}(\mathbf{X}) = p]
\end{cases}.
\end{equation}
The $\g$ function for a well-calibrated model is the identity function $\g(p) = p$. In practice, from this real-valued setting, it is challenging to have a sufficient number of observations in the training dataset with identical scores to effectively measure calibration defined in Eq.~\ref{eq-calib-E}, resulting in a lack of robustness in the estimation process of these probabilities \citep{Dimitriadis_2021}. A common method for estimating calibration is to group observations into $B$ bins \citep{NEURIPS2019_f8c0c968, naeini2015obtaining}, defined by the empirical quantiles of predicted values $\hat{s}(\mathbf{x})$. The average of observed values, denoted $\bar{y}_b$ with $b\in \{1, \ldots, B\}$, in each bin $b$ can then be compared with the central value of the bin. Thus, a calibration curve can be constructed by plotting the middle of each bin on the x-axis and the averages of corresponding observations on the y-axis, also referred to as ``reliability diagrams'' in \cite{Wilks_1990_WF}.

\paragraph{Calibration Metrics}
A commonly used metric for assessing calibration is the Brier score \citep{brier-use-2021, kull2017EJS, platt1999probabilistic, NEURIPS2020_9bc99c59}. It is expressed as \citep{brier_1950}
\begin{equation}
\text{BS} = \frac{1}{n}\sum_{i=1}^{n} \big(\hat{s}(\mathbf{x}_i) - y_i\big)^{2}\enspace.\label{eq-brier-score}
\end{equation}

\citet{Austin_2019} proposes the Integrated Calibration Index (ICI) based on the calibration curve, using smoothing methods instead of bins, similar to \citet{pmlr-v119-zhang20k}. Smoothing methods avoid the need for hyperparameter selection required in bin-based methods, which depend on the number of observations. However, they require a neighborhood parameter. Here, the binary event is regressed on predicted scores using LOESS for small datasets (\(n < 1000\)) or cubic regression splines for larger datasets, and $\hat{\g}$ denote the estimate. The ICI writes:
\begin{equation}
    \text{ICI} = \frac{1}{n}\sum_{i=1}^{n} \big\vert\hat{s}(\mathbf{x}_i) -\hat{\g}\big(\hat{s}(\mathbf{x}_i)\big)\big\vert,
\end{equation}
which is the empirical version of \citet{Austin_2019}, based on the absolute difference between the calibration curve (estimated using a local regression on sample $\{(\hat{s}(\mathbf{x}_i),y_i)\}$) and the identity function.

\section{Scores Heterogeneity}\label{sec:scores-heterogeneity}

As explained in \citet{dietrich2013reasons}, Bayesian epistemology guides us on how to update our beliefs in response to new evidence or information. It delineates the influence any specific piece of evidence should have on our beliefs and the posterior beliefs we should adopt after assimilating such evidence (for an introduction, see, \textit{e.g.}, \citealp{bovens2004bayesian}). This framework of rational belief updating presupposes the existence of prior beliefs. However, a notable limitation of Bayesianism is its silence on the origin of these priors. It lacks the means to justify certain prior beliefs as rational or appropriate and to critique others as irrational or unreasonable. Nonetheless, we naturally consider some prior beliefs to be more reasonable than others. In some applications, we could have some knowledge about the "true prior distribution" (in the sense given in \citealp{andrewgellman}), \textit{i.e.}, the distribution of the parameter across the population), corresponding to the ``{\em credence function}'' in \citet{lance1995subjective} and \citet{moss2018probabilistic}, while \citet{miller1975propensity} suggested the term ``propensity'', or ``{\em propensity theory of probability}'', following \citet{peirce1910notes}.

When a binary classification model is well-calibrated, the distribution of its score \(\hat{s}(\mathbf{x})\) according to Eq.~\ref{eq-calib-E} is identical to that of the underlying probability. In practice, however, the underlying probability is not observed. Consequently, calibration measures cannot rely on this probability. Whether based on bin creation or a local regression approach, calibration measures cannot assess discrepancies between the distribution of scores and the distribution of underlying probabilities of the event. Here, we aim to study the potential links between score heterogeneity and calibration measures in a controlled environment where we observe true probabilities through simulations. To achieve this, we train regression trees while varying their complexity. The fewer leaves a tree has, the less heterogeneous the scores it returns will be. 
To better understand the relationship between score heterogeneity and calibration, the decomposition of prediction errors proves to be useful.

\subsection{Decomposition}

Let \(\hat{s}\) denote a scoring classifier, \(\mathcal{X}\to[0,1]\), then set \(\widehat{S}:=\hat{s}(\bm{X})\) the score output by \(\hat{s}\) on instance \(\bm{X}\). Let \(P\) denote the true (posterior) probability, in the sense that \(S:=s(\bm{X})=\mathbb{E}[Y|\bm{X}]\). \citet{murphy1972scalar}
\citet{degroot1983comparison} and \citet{brocker2009reliability} suggested the following decomposition. Let \(C:=\mathbb{E}[Y|\widehat{S}]\), corresponding to the true proportion of 1's among the instances for
which the model has output the same estimate \(\widehat{S}\). The ``calibration-refinement'' decomposition considered in \citet{brocker2009reliability} is:

\begin{lemma}[Adapted from \citealp{brocker2009reliability}]
The expected loss corresponding to any proper scoring rule is the sum of expected divergence of \(\widehat{S}\) from \(C\) and the expected
divergence of \(C\) from \(Y\), denoted\footnote{with notations based on random variables instead of distributions for convenience and interpretability.} 
\[
\mathbb{E}\big[d(\widehat{S},Y)\big]=\underbrace{
\mathbb{E}\big[d(\widehat{S},C)\big]}_{\text{calibration loss}}+
\underbrace{\mathbb{E}\big[d(C,Y)\big]}_{\text{refinement loss}}.
\]
\end{lemma}

Here, the calibration loss is due to the difference between
the model output score \(\widehat{S}\) and the proportion of 1's among instances with the same output.
An alternative decomposition can be considered

\begin{lemma}[Adapted from \citealp{kull2015novel}]
The expected loss corresponding to any proper scoring rule is the sum of expected divergence of \(\widehat{S}\) from \(S\) and the expected
divergence of \(S\) from \(Y\),  
\[
\mathbb{E}\big[d(\widehat{S},Y)\big]=\underbrace{
\mathbb{E}\big[d(\widehat{S},S)\big]}_{\text{epistemic loss}}+
\underbrace{\mathbb{E}\big[d(S,Y)\big]}_{\text{irreducible loss}}.
\]
\end{lemma}

Here, the irreducible loss represents inherent uncertainty in the classification task (and is the same whatever the model), and is called ``aleatoric'' in \citet{hora1996aleatory, der2009aleatory, senge2014reliable}; the  epistemic uncertainty, as coined in \citet{der2009aleatory}, measures the distance between the distribution of estimated scores (denoted \(\widehat{S}\)) and the true probability \(S\). 
If the scoring rule is the log-loss (that is a proper scoring rule, \citealp{good1952rational}), the associated discrepancy $d$ is KL divergence (or relative entropy). If \(\widehat{S}\overset{a.s.}{=}S\), \citet{kull2015novel} says that \(\widehat{s}\) is ``optimal'' (w.r.t. any proper scoring rule). Unlike aleatoric uncertainty, which refers to an intrinsic notion of randomness and uncertainty (and is irreducible by nature), the epistemic uncertainty can potentially be reduced with additional information, as discussed in \citet{banerji2023clinical}, such as the range of the underlying probabilities or their distribution in the population (is the prior probability to have a car accident for individuals concentrated around 10\%, or is it on a larger range, e.g., \([5\%;40\%]\)?). To capture this heterogeneity of probabilities in the population, we use KL divergence between distributions of \(\widehat{S}\) and \(S=\mathbb{E}[Y|\boldsymbol{X}]\), as in \citet{shaker2020aleatoric}: ``{\em the epistemic uncertainty is measured in terms of the mutual information between hypotheses and outcomes (\textnormal{i.e.}, the KL divergence).}'' Since the distribution of \(S\) is always unknown in real applications, \citet{apostolakis1990concept} suggests to consider the distance between the distribution of predicted scores \(\hat{s}_i\)'s and some expert knowledge, or prior opinion on the distribution of true probabilities. \citet{zidek2003uncertainty} suggests to compare simply the dispersion (quantified using the variance) of predicted scores.

\subsection{Synthetic Data}\label{sec:synth-data}

We simulate data to compare estimated scores with the underlying probabilities, considering four types of data generating processes (DGP) using a logistic link function. The first three DGPs follow \citet{Ojeda_2023}, and we introduce a fourth DGP with interaction terms and non-linearities. %
Each scenario uses a logistic model to generate the outcome. Let \(Y_i\) be a binary variable following a Bernoulli distribution: \(Y_i \sim B(p_i)\), where \(p_i\) is the probability of observing \(Y_i = 1\). The probability \(p_i\) is defined by:
\begin{equation}
p_i = \mathbb{P}(Y = 1 \mid \mathbf{x}_i)=\mathbb{E}[Y \mid \mathbf{x}_i] = \big[{1+\exp(-\eta_i)}\big]^{-1}.\enspace\label{eq-true-propensity}
\end{equation}
For the second DGP, to introduce non-linearities, \(p^3\) is used as true probabilities instead of $p$. For all DGPs, \(\eta_i =\mathbf{x}_i^\top \boldsymbol{\beta}\), %
where \(\mathbf{x}_i\) is a vector of covariates and \(\boldsymbol{\beta}\) is a vector of arbitrary scalars. The covariate vector includes two continuous predictors for DGPs 1 and 2. For DGP 3, it includes five continuous and five categorical predictors. For DGP 4, it contains three continuous variables, the square of the first variable, and an interaction term between the second and third variables.\footnote{\(\eta_i = \beta_1 x_{1,i} + \beta_2 x_{2,i} + \beta_3 x_{3,i} + \beta_4 x_{1,i}^2 + \beta_5 x_{2,i} \times x_{3,i}\).} Continuous predictors are drawn from \(\mathcal{N}(0,1)\). Categorical predictors consist of two variables with two categories, one with three categories, and one with five categories, all uniformly distributed.

For each DGP, we generate data considering four scenarios with varying numbers of noise variables: 0, 10, 50, or 100 variables drawn from \(\mathcal{N}(0,1)\) (see Table~\ref{tab:dgp} for simulation settings). The distribution of underlying probabilities for each DGP category is shown in Fig.~\ref{fig:hist-underlying-prob}, based on a single draw of 10,000 observations. For the fourth DGP, to achieve a similar probability distribution to DGP 1, we perform resampling using a rejection algorithm (details in Section~\ref{sec-appendix-resampling-algo}).

\begin{figure}[ht!]
\centering
\includegraphics[width=\columnwidth]{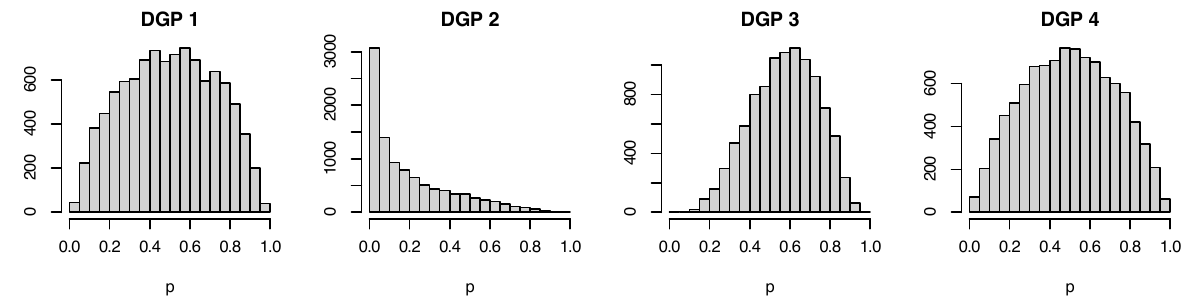}
\caption{Distribution of the underlying probabilities in the different categories of scenarios.}
\label{fig:hist-underlying-prob}
\end{figure}

\subsection{Illustrative Example: Regression Trees}\label{sec:heterogeneity-trees}

Tree-based methods like RF and XGBoost have become popular for estimating binary events \citep{insu-tree}. Understanding their core component, decision trees, is therefore essential. In decision trees, the heterogeneity of scores and tree size are interconnected. The number of leaves in a decision tree determines the partitioning of the data space into subspaces. As the number of leaves increases, the prediction function becomes more segmented. To illustrate this relationship, we generate a dataset of 30,000 observations (10,000 each for training, validation and test) according to DGP~1, without adding noise variables. Simulations with varying numbers of noise variables and different DGPs are performed subsequently. The trees are estimated using the \texttt{R} package \texttt{rpart} \citep{r_package_rpart}, which implements the CART algorithm \citep{breimann1984classification}. Since the CART algorithm does not allow direct control over the number of leaves, we adjust the parameter controlling the minimum number of observations per terminal leaf.

Among the estimated trees, we isolate the smallest one (with the fewest leaves), and the largest one  (with the most leaves). The distribution of scores for these two trees on the test sample is shown in Fig.~\ref{fig:tree-scenario1-hist-scores}. The smallest tree, with five leaves, produces five distinct scores. In contrast, the largest tree, with 1908 terminal leaf nodes, exhibits a more heterogeneous score distribution, with a strong concentration around 0 and 1 due to the model's objective of minimizing classification error. As anticipated, the tree with numerous leaves significantly overfits compared to the true probabilities.

\begin{figure}[ht!]
\centering
\includegraphics[width=\columnwidth]{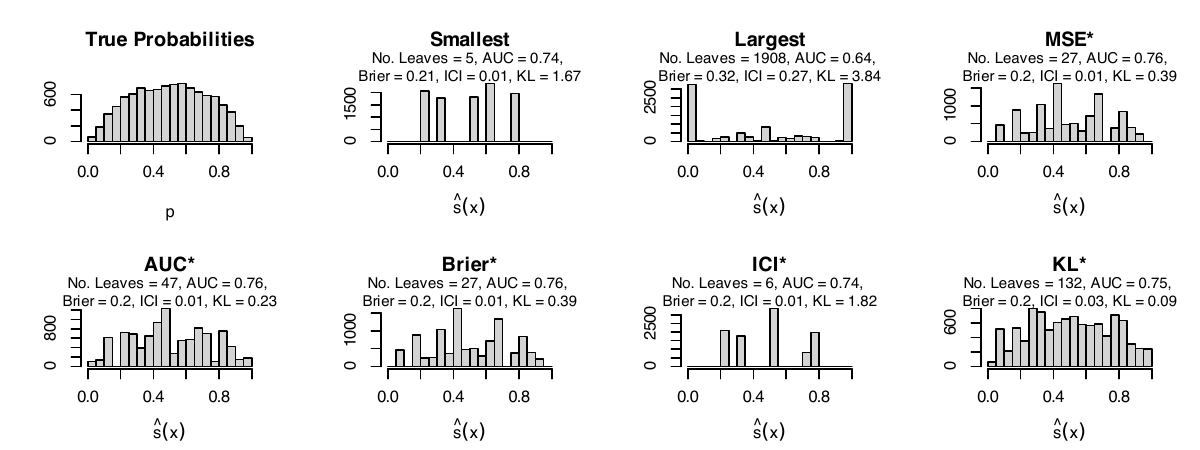}
\caption{Distribution of scores on test set for trees (DGP~1, without noise variables).}
\label{fig:tree-scenario1-hist-scores}
\end{figure}

In practice, hyperparameters are typically selected through grid search, choosing the model that performs best according to a criterion on a validation set.  For our binary classifier, we might select the tree that maximizes the AUC, for example. The score distribution for this tree, labeled AUC* in Fig.~\ref{fig:tree-scenario1-hist-scores}, may be more heterogeneous than that of the smallest tree but tends to deviate from the true probability distribution. To quantify this deviation, we first discretize the probability and score distributions by dividing the interval \([0,1]\) into \(m=20\) bins (histograms). We then compute the KL divergence \citep{Kullback_1951} between the two distributions, using the true probabilities as the reference. Let \(h_\phi\) be the bin of the scores and \(h_g\) the bin of the true probabilities, the KL divergence of \(\phi\) w.r.t. \(g\) is:%
\begin{equation}
D_{KL}(\phi || g) = \sum_{i=1}^m h_\phi(i) \log \frac{h_\phi(i)}{h_g(i)}.
\end{equation}

The Kullback-Leibler divergence on the test sample of the tree whose hyperparameters were chosen by optimizing this metric, denoted KL*, is much lower than that of the tree maximizing AUC (0.09 vs. 0.23), while the AUC remains almost unchanged.

Since our primary interest is in the predicted scores rather than the predicted labels, it might be tempting to choose model hyperparameters by optimizing a calibration distance. However, the results obtained by minimizing the Brier score (denoted Brier*) or the ICI (denoted ICI*) show that model calibration is not sufficient to achieve this goal. The distribution of obtained scores, similar to when maximizing the AUC, remains distant from the probability distribution.

\paragraph{Replications}
The experiment conducted with a single simulation under DGP 1 warrants replication with other samples. We generate 100 replications of the experiment (replications under other DGPs are presented in Section~\ref{sec:appendix-sim-trees}). Table~\ref{tbl:trees-metrics} reports the average values for performance metrics (AUC), calibration (Brier, ICI), and distribution divergence (KL Div.). In practical applications, as previously mentioned, the focus is often on optimizing performance metrics, typically AUC. However, maximizing AUC (row AUC*) reduces the ICI and Brier scores but increases the KL divergence compared to the tree optimized for KL (row KL*). This suggests that if decision-makers prioritize model calibration alongside classification performance, maximizing AUC may not yield score distributions that closely resemble the underlying data distribution. In addition, if decision-makers focus on optimizing calibration metrics like ICI or Brier instead of AUC, the performance decrease will be more pronounced than when optimizing KL. This also results in substantial discrepancies between the true probability distribution and the predicted scores, as measured by KL divergence.

The trees obtained by optimizing AUC have fewer terminal leaves than those obtained by minimizing KL divergence, resulting in lower score heterogeneity in the former. The Quantile Ratio (QR in Table~\ref{tbl:trees-metrics}) quantifies this variability.
QR quantifies the ratio between the difference of the 90th and 10th percentiles of the scores and the difference of the 90th and 10th percentiles of the true probabilities. A ratio of 1 indicates score spread equals true probability spread. A ratio greater than 1 implies wider score distribution than true probabilities, indicating more variability. Conversely, a ratio less than 1 suggests scores are more concentrated around central values than true probabilities, indicating less variability. Regardless of whether AUC or KL divergence is optimized, the QR remains very close to 1, slightly exceeding it. For DGPs~3 and 4 (Tables~\ref{tbl:trees-metrics-dgp-3} and \ref{tbl:trees-metrics-dgp-4}), minimizing KL divergence yields a QR close to 1, while maximizing AUC results in a QR significantly less than 1. This discrepancy reflects lower score variability, as shown in Figs.~\ref{fig:bp_synthetic_trees_3} and \ref{fig:bp_synthetic_trees_4}.

{
\setlength{\tabcolsep}{1.9pt}

\begin{table*}[ht!]
    \caption{Average performance and calibration metrics, for regression trees, computed on test set over 100 replications under DGP~1, without noise variables.}\label{tbl:trees-metrics}
    \centering\small
    \begin{tabular}{cccccccc}
\toprule
Tree & No. Leaves & MSE & AUC & Brier & ICI & KL Div. & QR \\
\midrule
Smallest & 5 (0.603) & 0.012 (0.002) & 0.720 (0.007) & 0.212 (0.002) & 0.012 (0.004) & 1.807 (0.164) & 0.930 (0.055)\\
Largest & 1 999 (44.796) & 0.134 (0.004) & 0.629 (0.006) & 0.333 (0.005) & 0.289 (0.008) & 4.190 (0.180) & 1.627 (0.010)\\
MSE* & 30 (4.800) & 0.003 (0.000) & 0.751 (0.005) & 0.203 (0.002) & 0.014 (0.004) & 0.300 (0.120) & 1.021 (0.048)\\
AUC* & 33 (6.639) & 0.004 (0.000) & 0.750 (0.005) & 0.203 (0.002) & 0.015 (0.004) & 0.283 (0.117) & 1.022 (0.047)\\
Brier* & 31 (6.170) & 0.004 (0.000) & 0.750 (0.005) & 0.203 (0.002) & 0.014 (0.004) & 0.284 (0.128) & 1.020 (0.046)\\
ICI* & 10 (6.782) & 0.008 (0.003) & 0.733 (0.012) & 0.208 (0.004) & 0.012 (0.004) & 1.262 (0.545) & 0.976 (0.078)\\
KL* & 97 (38.855) & 0.008 (0.003) & 0.741 (0.007) & 0.207 (0.003) & 0.029 (0.010) & 0.086 (0.019) & 1.078 (0.038)\\
\bottomrule
\end{tabular}
\end{table*}
}

Fig.~\ref{fig:trees-kl-calib-leaves-1} visualizes the relationship between KL divergence and calibration, measured by ICI, according to the number of terminal leaves in the tree for data simulated under DGP 1. Excluding the potential case of a single-leaf tree, which would be perfectly calibrated, an increase in the number of leaves worsens calibration on the validation sample. Concurrently, KL divergence exhibits a convex relationship with the number of leaves. Similar observations are reported for other DGPs and when using the Brier score instead of ICI as the calibration measure (see Figs.~\ref{fig:trees-kl-calib-brier-leaves-all} and \ref{fig:trees-kl-calib-ici-leaves-all} in the appendix).

\begin{figure}[ht!]
    \centering
    \includegraphics[width = \columnwidth]{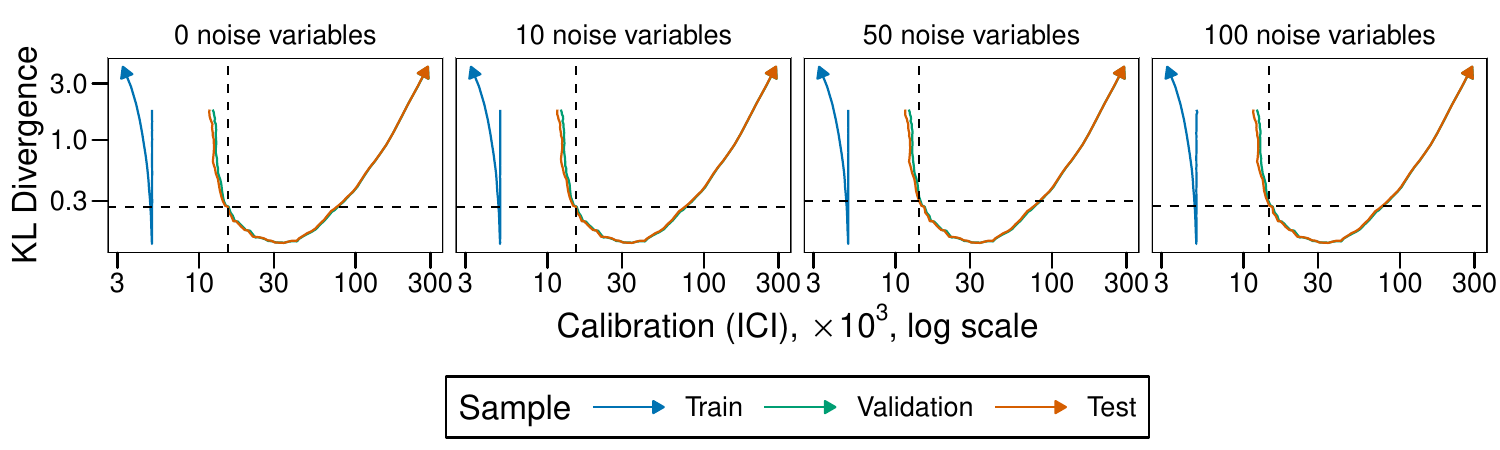}
    \caption{KL Divergence and Calibration (DGP 1) across increasing average number of tree leaves. Dashed lines represent values when maximizing AUC. Arrows indicate increasing number of leaves.}
    \label{fig:trees-kl-calib-leaves-1}
\end{figure}

Overall, the results show that increasing tree leaves tends to worsen calibration but may also decrease the divergence between predicted scores and the true probability distribution. This highlights the balance between model complexity and the fidelity of score distributions to underlying probabilities. Our findings suggest that optimizing for performance metrics, calibration metrics, and divergence metrics often leads to different outcomes, indicating that these objectives may not be aligned.

\section{Simulations}

Simulations using a decision tree on synthetic data to estimate the probabilities of a binary event showed that optimizing KL distance closely approximates the probability distribution of predicted scores with low impact on predictive performance. Our findings\footnote{\textbf{Reproducibility}: All the \texttt{R} code required to reproduce the results is available in an archive. Explanations and outputs are provided in an accompanying e-book.} highlight that calibration is meaningful only when the distribution of estimated scores mirrors the true underlying probabilities. Otherwise, a model may seem well-calibrated but fail to reflect the actual event probabilities. We extend this analysis to include other tree-based models, such as RF and XGBoost and run simulations on both synthetic and real-world data. Within a hyperparameter optimization framework targeting specific criteria, we explore discrepancies between AUC performance, calibration (Brier Score and ICI), and alignment with true probabilities (KL divergence). By comparing the outcomes from GLM and two GAM models---one with all explanatory variables and the other with Lasso penalization---we aim to show that calibration metrics alone are inadequate for decision-making. 

\subsection{Synthetic Data}

Comparisons of performance, calibration, and divergence metrics are conducted while varying key hyperparameters of RF and XGBoost on simulations, using 100 replications for each of the four DGPs presented in Section \ref{sec:synth-data}. A controlled environment is used where the KL divergence is calculated using the true simulated probabilities for each scenario. The binary outcome is estimated using two machine learning models (RF and XGBoost), with hyperparameters selected by optimizing either AUC (benchmark) or KL divergence on the validation set. For RF, we vary the minimal number of observations in the terminal leaves as well as the number of candidate variables for splitting. For XGBoost, we vary the tree depth and the number of boosting iterations. Further details on the hyperparameter optimization grids are provided in Sections~\ref{sec:appendix-rf} (RF) and \ref{sec:appendix-xgb} (XGBoost). For comparison, three other models, GLM \citep{nelder_1972_glm}, GAM \citep{hastie_gam_1986}, and GAMSEL \citep{chouldechova2015generalized}, are also estimated. In these models, all available predictors in the dataset are included. No interaction terms are added (therefore, the models estimated under DGP 4 are misspecified).

Table~\ref{tab:synthetic-results-dgp1-short} reports the metrics used previously to measure performance (AUC), calibration (Brier Score and ICI), divergence between the distribution of scores and that of probabilities (KL), as well as the ratio of the interquantile difference (90th and 10th) of the scores over that of the reference prior (QR). These metrics are computed on the test set. For space considerations, only the results from simulations conducted under DGP 1 are presented here, for various numbers of noise variables introduced in the datasets (See Tables~\ref{tab:synthetic-results-dgp1}, \ref{tab:synthetic-results-dgp2}, \ref{tab:synthetic-results-dgp3}, and \ref{tab:synthetic-results-dgp4} in the appendix for results for all the DGPs, also including metrics computed for the GLM, GAM, and GAMSEL models). Regardless of the number of noise variables introduced in the dataset, selecting hyperparameters to minimize KL divergence rather than maximize AUC leads to a slight reduction in performance (the variation in AUC, $\Delta \text{AUC}$, ranges from 0 to 0.01). At the same time, calibration measures (Brier Score, ICI) show little or no variation. The variability of the scores increases when the objective is the minimization of KL divergence. Specifically, the interquantile range ratios increase more substantially. These results are also evident in the score histograms presented in the appendix (Figs.~\ref{fig:bp_synthetic_forest_1}, \ref{fig:bp_synthetic_forest_2}, \ref{fig:bp_synthetic_forest_3}, and \ref{fig:bp_synthetic_forest_4} for RF, and \ref{fig:bp_synthetic_xgb_1}, \ref{fig:bp_synthetic_xgb_2}, \ref{fig:bp_synthetic_xgb_3}, and \ref{fig:bp_synthetic_xgb_4} for XGBoost). However, it should be noted that, for RF, when the number of noise variables is high (50 or 100), the interquantile range ratio remains below 1, indicating that in these cases, the tree-based models underestimate the variability of the probabilities. Finally, concerning KL divergence, an improvement is observed when optimization is based on this criterion, with this relative improvement increasing with the number of noise variables. The improvement is more pronounced for RF than for XGBoost.

To summarize, for all DGPs, the RF and XGBoost models yielding the optimal KL do not lead to significant performance loss. Opting for the best KL divergence does not cause a large change in calibration but reduces the distance between the distribution of estimated scores and that of the underlying probabilities, the reduction being more substantial in the presence of noise in the dataset.

{
\setlength{\tabcolsep}{2pt}

\begin{table*}[t]
    \caption{Metrics computed on the test set for models selected based on KL divergence across 100 replications of DGP 1, with variations (\(\Delta\)) compared to metrics when models are selected based on AUC. Standard errors are provided in parentheses.} 
    \label{tab:synthetic-results-dgp1-short}
    \centering\small
    \begin{tabular}{clcccccccccc}
\toprule
Noise & Model & AUC & Brier & ICI & KL & QR & \(\Delta\)AUC & \(\Delta\)Brier & \(\Delta\)ICI & \(\Delta\)KL & \(\Delta\)QR\\
\midrule
 0 & RF & 0.75 (0.01) & 0.20 (0.00) & 0.02 (0.01) & 0.03 (0.01) & 1.04 (0.02) & \cellcolor[HTML]{FFCCCC}{\textcolor[HTML]{333333}{-0.01}} & \cellcolor[HTML]{FFD6D6}{\textcolor[HTML]{333333}{0.00}} & \cellcolor[HTML]{FFD6D6}{\textcolor[HTML]{333333}{0.01}} & \cellcolor[HTML]{E9F6E9}{\textcolor[HTML]{1A1A1A}{-0.03}} & 0.04\\
 & XGB & 0.75 (0.01) & 0.20 (0.00) & 0.01 (0.00) & 0.02 (0.01) & 1.01 (0.02) & \cellcolor[HTML]{FFD6D6}{\textcolor[HTML]{333333}{0.00}} & \cellcolor[HTML]{FFD6D6}{\textcolor[HTML]{333333}{0.00}} & \cellcolor[HTML]{FFD6D6}{\textcolor[HTML]{333333}{0.00}} & \cellcolor[HTML]{E9F6E9}{\textcolor[HTML]{1A1A1A}{-0.03}} & 0.04\\
 \cmidrule{1-11}
10 & RF &  0.75 (0.01) & 0.20 (0.00) & 0.01 (0.00) & 0.01 (0.00) & 1.00 (0.02) & \cellcolor[HTML]{FFC2C2}{\textcolor[HTML]{2B2B2B}{-0.01}} & \cellcolor[HTML]{FFCCCC}{\textcolor[HTML]{333333}{0.00}} & \cellcolor[HTML]{FFD6D6}{\textcolor[HTML]{333333}{0.00}} & \cellcolor[HTML]{E9F6E9}{\textcolor[HTML]{1A1A1A}{-0.05}} & 0.02\\
 & XGB &  0.75 (0.01) & 0.20 (0.00) & 0.02 (0.00) & 0.01 (0.00) & 1.02 (0.02) & \cellcolor[HTML]{FFCCCC}{\textcolor[HTML]{333333}{-0.01}} & \cellcolor[HTML]{FFD6D6}{\textcolor[HTML]{333333}{0.00}} & \cellcolor[HTML]{FFD6D6}{\textcolor[HTML]{333333}{0.00}} & \cellcolor[HTML]{E9F6E9}{\textcolor[HTML]{1A1A1A}{-0.06}} & 0.08\\
\cmidrule{1-11}
50 & RF &  0.75 (0.01) & 0.21 (0.00) & 0.03 (0.01) & 0.14 (0.02) & 0.81 (0.02) & \cellcolor[HTML]{FFD6D6}{\textcolor[HTML]{333333}{0.00}} & \cellcolor[HTML]{FFD6D6}{\textcolor[HTML]{333333}{0.00}} & \cellcolor[HTML]{E9F6E9}{\textcolor[HTML]{1A1A1A}{-0.01}} & \cellcolor[HTML]{E9F6E9}{\textcolor[HTML]{1A1A1A}{-0.12}} & 0.04\\
 & XGB & 0.74 (0.01) & 0.21 (0.00) & 0.02 (0.01) & 0.01 (0.00) & 1.02 (0.02) & \cellcolor[HTML]{FFC2C2}{\textcolor[HTML]{2B2B2B}{-0.01}} & \cellcolor[HTML]{FFCCCC}{\textcolor[HTML]{333333}{0.00}} & \cellcolor[HTML]{FFD6D6}{\textcolor[HTML]{333333}{0.00}} & \cellcolor[HTML]{E9F6E9}{\textcolor[HTML]{1A1A1A}{-0.08}} & 0.10\\
\cmidrule{1-11}
 100 & RF  & 0.74 (0.01) & 0.21 (0.00) & 0.06 (0.01) & 0.45 (0.03) & 0.61 (0.02) & \cellcolor[HTML]{FFD6D6}{\textcolor[HTML]{333333}{0.00}} & \cellcolor[HTML]{E9F6E9}{\textcolor[HTML]{1A1A1A}{0.00}} & \cellcolor[HTML]{E9F6E9}{\textcolor[HTML]{1A1A1A}{-0.01}} & \cellcolor[HTML]{D4F2D4}{\textcolor[HTML]{1A1A1A}{-0.13}} & 0.04\\
 & XGB & 0.74 (0.01) & 0.21 (0.00) & 0.02 (0.01) & 0.01 (0.00) & 1.02 (0.01) & \cellcolor[HTML]{FFB8B8}{\textcolor[HTML]{2B2B2B}{-0.01}} & \cellcolor[HTML]{FFCCCC}{\textcolor[HTML]{333333}{0.00}} & \cellcolor[HTML]{FFD6D6}{\textcolor[HTML]{333333}{0.00}} & \cellcolor[HTML]{E9F6E9}{\textcolor[HTML]{1A1A1A}{-0.08}} & 0.11\\
\bottomrule
\end{tabular}
\end{table*}
}

\subsection{Real-World Data}\label{sec:real-world-data}

In the context of real-world data, the probability distributions corresponding to the response variable \(Y\) are not directly observable. Therefore, we employ a Bayesian framework with a Beta prior to model the underlying data distributions. While expert opinions could inform the parameters of this prior distribution, we estimate them using a statistical learning model---either a GLM or a GAM---for illustrative purposes. Using these prior distributions, we can replicate the procedure previously applied to simulated data and compare the performance of tree-based models when the objective is to maximize the AUC or minimize the KL divergence between the scores and the prior distribution. 
The data is split into 64\% for training, 16\% for validation, and 20\% for testing the fine-tuned models. Note that for the GAM, GLM, and GAMSEL models, the validation samples are not used.

We use ten datasets available on the UCI machine learning repository: (i) \texttt{abalone}, (ii) \texttt{adult}, (iii) \texttt{bank} (bank marketing), (iv) \texttt{default} (German credit), (v) \texttt{drybean}, (vi) \texttt{coupon} (in-vehicle coupon recommendation), (vii) \texttt{mushroom}, (viii) \texttt{occupancy} (occupancy Detection), (ix) \texttt{winequality} (wine quality), and (x) \texttt{spambase}. For most of these datasets, the initial objective is the classification of a binary variable. For the \texttt{adult} dataset, this binary variable is constructed from a continuous variable indicating annual salary. For the \texttt{drybean} and \texttt{winequality} datasets, the initial target variable is multi-class. In that case, we define the target as the most frequent level of this multi-class variable. More details on the datasets and the target variables are given in Section~\ref{sec:appendix-datasets} of the appendix.

Table~\ref{tab:real-data-results-gamsel-short} displays changes in metrics on the test sets when optimizing KL divergence instead of AUC, assuming the target follows a Beta distribution with parameters estimated from the GAMSEL model. See Tables~\ref{tab:real-data-results-glm}, \ref{tab:real-data-results-gam}, and \ref{tab:real-data-results-gamsel} for additional metrics with different priors, and Table~\ref{tab:results-real-glm} for metrics of GLM, GAM, and GAMSEL models.

Similar to what was observed with simulated data, when the objective becomes minimizing the KL divergence between the scores and the prior distribution of probabilities, this leads to a slight reduction in the AUC in most cases, for both RFs and XGBoost. The performance loss due to prioritizing KL divergence (\(\Delta \text{AUC}\)) is more pronounced for the \texttt{coupons} dataset (-0.07 with XGBoost) and the \texttt{winequality} dataset (-0.06 and -0.08 for RF and XGBoost, respectively). For these datasets, this performance loss is accompanied by a significant decrease in KL divergence (\(\Delta \text{KL}\)) when using XGBoost. This suggests that selecting XGBoost with optimal KL divergence enhances the resemblance between the prior Beta distribution and the predicted scores. It implies that if we have confidence in the prior distribution, using event probabilities directly from the model selected by optimizing AUC might not be advisable compared to using hyperparameters that yield optimal KL. While optimizing KL divergence consistently leads to a reduction in this divergence compared to the benchmark, we observed that it often results in a slight deterioration of calibration as measured by the Brier score (\(\Delta \text{Brier}\)) or ICI (\(\Delta \text{ICI}\)). When optimizing for the Brier score (Brier* in Table~\ref{tab:real-data-results-gamsel}), the AUC remains unchanged compared to the benchmark, but the score distribution significantly deviates from the prior distribution. Similarly, with the Integrated Calibration Index (ICI* in Table~\ref{tab:real-data-results-gamsel}), the KL divergence increases and the AUC significantly decreases.

To summarize, in cases where strong belief is placed in the prior, our results suggest prioritizing the minimization of KL divergence for selecting hyperparameters. The exercise hints that, in comparison to optimizing models based solely on AUC maximization or ICI minimization, choosing models that minimize KL divergence ensures a better representation of the true probability distributions with only a marginal decrease in AUC. Therefore, in decision-making contexts where accurate probability estimation is crucial, minimizing KL should be considered a more effective strategy than maximizing AUC or minimizing Brier score or ICI.

\begin{table}[htb]
    \caption{Changes in AUC, Brier score, ICI, and KL divergence with Beta-distributed priors when models are optimized based on KL divergence instead of AUC.}
    \label{tab:real-data-results-gamsel-short}
    \centering\scriptsize
    \begin{tabular}{llllll}
\toprule
Dataset & Model & \(\Delta \text{AUC}\) & \(\Delta \text{Brier}\) & \(\Delta \text{ICI}\) & \(\Delta \text{KL}\)\\
\midrule
 & RF & \cellcolor[HTML]{FFD6D6}{\textcolor[HTML]{333333}{0.00}} & \cellcolor[HTML]{E9F6E9}{\textcolor[HTML]{1A1A1A}{0.00}} & \cellcolor[HTML]{FFD6D6}{\textcolor[HTML]{333333}{0.00}} & \cellcolor[HTML]{E9F6E9}{\textcolor[HTML]{1A1A1A}{-0.01}}\\
\cmidrule{2-6}
\multirow[t]{-2}{*}[1\dimexpr\aboverulesep+\belowrulesep+\cmidrulewidth]{\raggedright\arraybackslash abalone} & XGB & \cellcolor[HTML]{FFD6D6}{\textcolor[HTML]{333333}{0.00}} & \cellcolor[HTML]{FFD6D6}{\textcolor[HTML]{333333}{0.00}} & \cellcolor[HTML]{FFC2C2}{\textcolor[HTML]{2B2B2B}{0.02}} & \cellcolor[HTML]{E9F6E9}{\textcolor[HTML]{1A1A1A}{-0.16}}\\
\cmidrule{1-6}
 & RF & \cellcolor[HTML]{FFD6D6}{\textcolor[HTML]{333333}{-0.01}} & \cellcolor[HTML]{FFCCCC}{\textcolor[HTML]{333333}{0.01}} & \cellcolor[HTML]{FFC2C2}{\textcolor[HTML]{2B2B2B}{0.02}} & \cellcolor[HTML]{E9F6E9}{\textcolor[HTML]{1A1A1A}{-0.04}}\\
\cmidrule{2-6}
\multirow[t]{-2}{*}[1\dimexpr\aboverulesep+\belowrulesep+\cmidrulewidth]{\raggedright\arraybackslash adult} & XGB & \cellcolor[HTML]{FFC2C2}{\textcolor[HTML]{2B2B2B}{-0.02}} & \cellcolor[HTML]{FFB8B8}{\textcolor[HTML]{2B2B2B}{0.01}} & \cellcolor[HTML]{FFB8B8}{\textcolor[HTML]{2B2B2B}{0.02}} & \cellcolor[HTML]{E9F6E9}{\textcolor[HTML]{1A1A1A}{-0.26}}\\
\cmidrule{1-6}
 & RF & \cellcolor[HTML]{FFB8B8}{\textcolor[HTML]{2B2B2B}{-0.03}} & \cellcolor[HTML]{FFC2C2}{\textcolor[HTML]{2B2B2B}{0.01}} & \cellcolor[HTML]{FFB8B8}{\textcolor[HTML]{2B2B2B}{0.03}} & \cellcolor[HTML]{E9F6E9}{\textcolor[HTML]{1A1A1A}{-0.25}}\\
\cmidrule{2-6}
\multirow[t]{-2}{*}[1\dimexpr\aboverulesep+\belowrulesep+\cmidrulewidth]{\raggedright\arraybackslash bank} & XGB & \cellcolor[HTML]{FFA3A3}{\textcolor[HTML]{1F1F1F}{-0.04}} & \cellcolor[HTML]{FFC2C2}{\textcolor[HTML]{2B2B2B}{0.01}} & \cellcolor[HTML]{FFC2C2}{\textcolor[HTML]{2B2B2B}{0.02}} & \cellcolor[HTML]{D4F2D4}{\textcolor[HTML]{1A1A1A}{-0.50}}\\
\cmidrule{1-6}
 & RF & \cellcolor[HTML]{FFD6D6}{\textcolor[HTML]{333333}{0.00}} & \cellcolor[HTML]{FFD6D6}{\textcolor[HTML]{333333}{0.00}} & \cellcolor[HTML]{FFCCCC}{\textcolor[HTML]{333333}{0.01}} & \cellcolor[HTML]{E9F6E9}{\textcolor[HTML]{1A1A1A}{-0.02}}\\
\cmidrule{2-6}
\multirow[t]{-2}{*}[1\dimexpr\aboverulesep+\belowrulesep+\cmidrulewidth]{\raggedright\arraybackslash default} & XGB & \cellcolor[HTML]{2F5D2F}{\textcolor[HTML]{E6E6E6}{0.00}} & \cellcolor[HTML]{E9F6E9}{\textcolor[HTML]{1A1A1A}{0.00}} & \cellcolor[HTML]{FFD6D6}{\textcolor[HTML]{333333}{0.00}} & \cellcolor[HTML]{E9F6E9}{\textcolor[HTML]{1A1A1A}{0.00}}\\
\cmidrule{1-6}
 & RF & \cellcolor[HTML]{FFD6D6}{\textcolor[HTML]{333333}{0.00}} & \cellcolor[HTML]{FFB8B8}{\textcolor[HTML]{2B2B2B}{0.01}} & \cellcolor[HTML]{FFA3A3}{\textcolor[HTML]{1F1F1F}{0.04}} & \cellcolor[HTML]{E9F6E9}{\textcolor[HTML]{1A1A1A}{-0.44}}\\
\cmidrule{2-6}
\multirow[t]{-2}{*}[1\dimexpr\aboverulesep+\belowrulesep+\cmidrulewidth]{\raggedright\arraybackslash drybean} & XGB & \cellcolor[HTML]{FFD6D6}{\textcolor[HTML]{333333}{0.00}} & \cellcolor[HTML]{FFCCCC}{\textcolor[HTML]{333333}{0.00}} & \cellcolor[HTML]{FFC2C2}{\textcolor[HTML]{2B2B2B}{0.02}} & \cellcolor[HTML]{E9F6E9}{\textcolor[HTML]{1A1A1A}{-0.25}}\\
\cmidrule{1-6}
 & RF & \cellcolor[HTML]{FFCCCC}{\textcolor[HTML]{333333}{-0.01}} & \cellcolor[HTML]{FFCCCC}{\textcolor[HTML]{333333}{0.01}} & \cellcolor[HTML]{FFD6D6}{\textcolor[HTML]{333333}{0.00}} & \cellcolor[HTML]{E9F6E9}{\textcolor[HTML]{1A1A1A}{-0.01}}\\
\cmidrule{2-6}
\multirow[t]{-2}{*}[1\dimexpr\aboverulesep+\belowrulesep+\cmidrulewidth]{\raggedright\arraybackslash coupon} & XGB & \cellcolor[HTML]{FF8585}{\textcolor[HTML]{101010}{-0.07}} & \cellcolor[HTML]{FFA3A3}{\textcolor[HTML]{1F1F1F}{0.02}} & \cellcolor[HTML]{426F42}{\textcolor[HTML]{E6E6E6}{-0.07}} & \cellcolor[HTML]{6CA56C}{\textcolor[HTML]{E6E6E6}{-3.10}}\\
\cmidrule{1-6}
 & RF & \cellcolor[HTML]{FFD6D6}{\textcolor[HTML]{333333}{0.00}} & \cellcolor[HTML]{FF9999}{\textcolor[HTML]{1A1A1A}{0.02}} & \cellcolor[HTML]{FFC2C2}{\textcolor[HTML]{2B2B2B}{0.02}} & \cellcolor[HTML]{E9F6E9}{\textcolor[HTML]{1A1A1A}{-0.24}}\\
\cmidrule{2-6}
\multirow[t]{-2}{*}[1\dimexpr\aboverulesep+\belowrulesep+\cmidrulewidth]{\raggedright\arraybackslash mushroom} & XGB & \cellcolor[HTML]{FFD6D6}{\textcolor[HTML]{333333}{0.00}} & \cellcolor[HTML]{FFADAD}{\textcolor[HTML]{232323}{0.02}} & \cellcolor[HTML]{FF9999}{\textcolor[HTML]{1A1A1A}{0.05}} & \cellcolor[HTML]{D4F2D4}{\textcolor[HTML]{1A1A1A}{-0.83}}\\
\cmidrule{1-6}
 & RF & \cellcolor[HTML]{FFD6D6}{\textcolor[HTML]{333333}{0.00}} & \cellcolor[HTML]{FFB8B8}{\textcolor[HTML]{2B2B2B}{0.02}} & \cellcolor[HTML]{FF8585}{\textcolor[HTML]{101010}{0.07}} & \cellcolor[HTML]{D4F2D4}{\textcolor[HTML]{1A1A1A}{-0.65}}\\
\cmidrule{2-6}
\multirow[t]{-2}{*}[1\dimexpr\aboverulesep+\belowrulesep+\cmidrulewidth]{\raggedright\arraybackslash occupancy} & XGB & \cellcolor[HTML]{FFD6D6}{\textcolor[HTML]{333333}{0.00}} & \cellcolor[HTML]{FFD6D6}{\textcolor[HTML]{333333}{0.00}} & \cellcolor[HTML]{FFADAD}{\textcolor[HTML]{232323}{0.04}} & \cellcolor[HTML]{E9F6E9}{\textcolor[HTML]{1A1A1A}{-0.24}}\\
\cmidrule{1-6}
 & RF & \cellcolor[HTML]{FF8F8F}{\textcolor[HTML]{141414}{-0.06}} & \cellcolor[HTML]{FF7A7A}{\textcolor[HTML]{0A0A0A}{0.04}} & \cellcolor[HTML]{FFCCCC}{\textcolor[HTML]{333333}{0.01}} & \cellcolor[HTML]{E9F6E9}{\textcolor[HTML]{1A1A1A}{-0.38}}\\
\cmidrule{2-6}
\multirow[t]{-2}{*}[1\dimexpr\aboverulesep+\belowrulesep+\cmidrulewidth]{\raggedright\arraybackslash winequality} & XGB & \cellcolor[HTML]{FF7A7A}{\textcolor[HTML]{0A0A0A}{-0.08}} & \cellcolor[HTML]{FF8F8F}{\textcolor[HTML]{141414}{0.03}} & \cellcolor[HTML]{2F5D2F}{\textcolor[HTML]{E6E6E6}{-0.07}} & \cellcolor[HTML]{2F5D2F}{\textcolor[HTML]{E6E6E6}{-4.53}}\\
\cmidrule{1-6}
 & RF & \cellcolor[HTML]{FFCCCC}{\textcolor[HTML]{333333}{-0.01}} & \cellcolor[HTML]{FF9999}{\textcolor[HTML]{1A1A1A}{0.02}} & \cellcolor[HTML]{FFC2C2}{\textcolor[HTML]{2B2B2B}{0.02}} & \cellcolor[HTML]{E9F6E9}{\textcolor[HTML]{1A1A1A}{-0.10}}\\
\cmidrule{2-6}
\multirow[t]{-2}{*}[1\dimexpr\aboverulesep+\belowrulesep+\cmidrulewidth]{\raggedright\arraybackslash spambase} & XGB & \cellcolor[HTML]{FFD6D6}{\textcolor[HTML]{333333}{-0.01}} & \cellcolor[HTML]{FFADAD}{\textcolor[HTML]{232323}{0.02}} & \cellcolor[HTML]{FFC2C2}{\textcolor[HTML]{2B2B2B}{0.02}} & \cellcolor[HTML]{D4F2D4}{\textcolor[HTML]{1A1A1A}{-0.75}}\\
\bottomrule
\end{tabular}
\end{table}

\section{Conclusion}

Traditionally, in binary classification, model hyperparameters are chosen to optimize performance. However, when the focus is on both the discriminative ability of the classifier and the quantification of underlying risks, it is important to ensure that the model's estimated scores reflect true probabilities. Common calibration metrics are insufficient for this purpose. Using synthetic data, we demonstrated that the goals of performance, calibration, and representativeness of scores w.r.t. true probabilities are not necessarily aligned. Optimizing model calibration increases the divergence between score and probability distributions. Our simulations suggest that optimizing KL divergence is more effective, although it results in a small loss of accuracy. With real data, minimizing KL divergence directly is not feasible since the true probabilities are not observable. However, if decision-makers have prior knowledge about the distribution of true probabilities, it can be used to select the model. Results from real data simulations align with those obtained from synthetic data.

{
\small
\bibliography{biblio}
}

\appendix

\begin{center}
\vspace{4mm}
{\scshape\Large Appendix}\\
\end{center}

\setcounter{figure}{0}  
\setcounter{table}{0}  
\renewcommand{\thetable}{\thesection\arabic{table}}
\renewcommand{\thefigure}{\thesection\arabic{figure}}

\section{Proof of Proposition \ref{prop:logistic}}\label{proof:prop}

As a starting point, suppose that there is single feature, so that $\mathbf{x}$ can be denoted $x$. Suppose that $Y|{X}={x} \sim \mathcal{B}\big(s({x})\big)$ where
$$s(\mathbf{x})={\big[1+\exp[-(\beta_0+\beta_1 x)]\big]}
,$$
with $\beta_1\neq 0$, and let $\widehat{\beta}_0$ and $\widehat{{\beta}}_1$, denote maximum likelihood estimators, so that the model is well specified. Then, for any ${x}$, the score is
$$
\hat{s}(\mathbf{x})={\big[1+\exp[-(\hat\beta_0+\beta_1 x)]\big]^{-1}}.
$$
Here both $s$ and $\hat{s}$ are continuous and invertible. Given $p\in(0,1)$, $\{\hat{s}({x})=p\}$ corresponds to $
\{x=\hat{s}^{-1}(p)\},$
since mapping $\hat{s}$ is one-to-one. Thus, $Y$ conditional on $\{\hat{s}({x})=p\}$ is therefore a Bernoulli variable with mean $s(\hat{s}^{-1}(p))$.
And because $\hat{s}$ and $s$ are continuous and bijective functions
$$
\begin{cases}
    \hat{\beta}_0\to\beta_0\\
    \hat{\beta}_1\to\beta_1
\end{cases}
\text{ as }n\to\infty
~\Longrightarrow~\forall x,p~
\begin{cases}
    \hat{s}(x)\to s(x)\\
    \hat{s}^{-1}(p)\to s^{-1}(p)
\end{cases}
\text{ as }n\to\infty
$$
since the model is well specified, where the convergence is in probability here. 
Thus
$$
s(\hat{s}^{-1}(p))\to p\text{ as }n\to\infty
$$
and
$$
\mathbb{E}[Y|\hat{s}(x)=y]=p(\hat{s}^{-1}(p))\to p\text{ as }n\to\infty.
$$
This property holds in higher (fixed) dimensions, if the model is well specified, since functions are continuous and invertible (linear models with continuous and invertible link functions). The case where the dimension of $\mathbf{x}$ increases with $n$ is discussed in \citet{bai2021don}.

\section{Resampling With a Target Distribution for the Predicted Scores}\label{sec-appendix-resampling-algo}

In our generated sample, $\mathcal{D}=\{(\boldsymbol{x}_i,y_i,{s}_i),i\in\{1,\cdots,n\}\}$, let $\widehat{\phi}$ denote the (empirical) density of scores. For the various scenarios, suppose that we want a specific distribution for the scores, denoted $g$ (uniform, Beta, etc.). A classical idea is to use ``rejection sampling" techniques to create a subsample of the dataset. Set 
$$
c = \sup_{s\in(0,1)} \frac{\widehat{\phi}(s)}{g(s)} \leq \infty.
$$
If $c$ is finite, and not too large, we can use the standard rejection technique, described in Algorithm \ref{alg:cap:1}. In a nutshell, point $i$ is kept with probability $(cg(s_i))^{-1}\widehat{\phi}(s_i)$.

\begin{algorithm}
\caption{Subsample a dataset so that the distribution of scores has density $g$ (Rejection, $c$ small)}\label{alg:cap:1}
\begin{algorithmic}
\Require $\mathcal{D}=\{(\boldsymbol{x}_i,y_i,{s}_i),i\in\{1,\cdots,n\}\}$ and $g$ (target density)
\State $\mathcal{I} \gets ,i\in\{1,\cdots,n\}$
\State $\widehat{\phi} \gets$ density of $\{({s}_i),i\in\mathcal{I}\}$, using \cite{chen1999beta}
\State $c = \displaystyle\sup_{s\in(0,1)} \frac{\widehat{\phi}(s)}{g(s)} \gets \max_{i=1,\cdots,n}\displaystyle\frac{\widehat{\phi}(s_i)}{g(s_i)} $
\For{$i\in\{1,\cdots,n\}$}
    \State $U \gets \mathcal{U}([0,1])$
    \If{$\displaystyle U > \frac{\widehat{\phi}(s_i)}{c\,g(s_i)}$}
        \State $\mathcal{I} \gets \mathcal{I}\backslash\{i\}$ , i.e., ``reject"
\EndIf 
\EndFor
\State $s\mathcal{D}=\{(\boldsymbol{x}_i,y_i,{s}_i),i\in\mathcal{I}\}$
\end{algorithmic}
\end{algorithm}  

If $c$ is too large, we use an iterative algorithm, described in Algorithm \ref{alg:cap:2}, inspired by \cite{rumbell2023novel} (alternative options could be the ``Empirical Supremum Rejection Sampling" introduced in \cite{caffo2002empirical}, for instance)

\begin{algorithm}
\caption{Subsample a dataset so that the distribution of scores has density $g$ (Iterative Rejection, $c$ large)}\label{alg:cap:2}
\begin{algorithmic}
\Require $\mathcal{D}=\{(\boldsymbol{x}_i,y_i,{s}_i),i\in\{1,\cdots,n\}\}$, $\epsilon>0$ and $g$ (target density)
\State $\mathcal{I} \gets \{1,\cdots,n\}$
\State $\widehat{\phi} \gets$ density of $\{({s}_i),i\in\mathcal{I}\}$, using \cite{chen1999beta}
\State $\widehat{\Phi} \gets$ c.d.f. of $\widehat{\phi}$, $\displaystyle{\widehat{\Phi}(t)=\int_0^t \widehat{\phi}(s)\mathrm{d}s}$
\State $d \gets \|\widehat{\Phi}-G\|_{\infty}$ (Kolmogorov-Smirnov distance)
\While{$d>\epsilon$}
\State $\mathcal{J} \gets \mathcal{I}$
\For{$i\in\mathcal{I}$}
    \State $U \gets \mathcal{U}([0,1])$
    \If{$\displaystyle U>\frac{\widehat{\phi}(s_i)}{g(s_i)}$}
        \State $\mathcal{J} \gets \mathcal{J}\backslash\{i\}$ , i.e., ``reject" observation $i$
\EndIf 
\EndFor
\State $\mathcal{I} \gets \mathcal{J}$
\State $\widehat{\phi} \gets$ density of $\{({s}_i),i\in\mathcal{I}\}$ and $\widehat{\Phi} \gets$ c.d.f. of $\widehat{\phi}$

\State $d \gets \|\widehat{\Phi}-G\|_{\infty}$ 
\EndWhile
\State $s\mathcal{D}=\{(\boldsymbol{x}_i,y_i,{s}_i),i\in\mathcal{I}\}$
\end{algorithmic}
\end{algorithm}  

\section{Simulated Data}

To simulate data, we consider different DGPs. The first three are from \citet{Ojeda_2023}. In the fourth, we add an interaction term between two predictors. Each scenario uses a logistic model to generate the outcome. Let \(Y_i\) be a binary variable following a Bernoulli distribution: \(Y_i \sim B(p_i)\), where \(p_i\) is the probability of observing \(Y_i = 1\). The probability \(p_i\) is defined by:
\begin{equation}
p_i = \mathbb{P}(Y = 1 \mid \mathbf{x}_i) = \big[{1+\exp(-\eta_i)}\big]^{-1}.\enspace\label{eq-true-propensity-appendix}
\end{equation}

For the second DGP, to introduce non-linearities, \(p^3\) is used as true probabilities instead of $p$. 

For all DGPs, \(\eta_i =\mathbf{x}_i^\top \boldsymbol{\beta}\), %
where \(\mathbf{x}_i\) is a vector of covariates and \(\boldsymbol{\beta}\) is a vector of arbitrary scalars. The covariate vector includes two continuous predictors for DGPs 1 and 2. For DGP 3, it includes five continuous and five categorical predictors. For DGP 4, it contains three continuous variables, the square of the first variable, and an interaction term between the second and third variables. Specifically, \(\eta_i = \beta_1 x_{1,i} + \beta_2 x_{2,i} + \beta_3 x_{3,i} + \beta_4 x_{1,i}^2 + \beta_5 x_{2,i} \times x_{3,i}\). Continuous predictors are drawn from \(\mathcal{N}(0,1)\). Categorical predictors consist of two variables with two categories, one with three categories, and one with five categories, all uniformly distributed. The values of coefficients \(\boldsymbol{\beta}\) are reported in Table~\ref{tab:dgp}.

For each DGP, we generate data considering four scenarios with varying numbers of noise variables: 0, 10, 50, or 100 variables drawn from \(\mathcal{N}(0,1)\). 

 For the fourth DGP, to achieve a similar probability distribution to DGP 1, we perform resampling using a rejection algorithm (details in Section~\ref{sec-appendix-resampling-algo}).

{
\begin{table}[H]
    \caption{Parameters of the different scenarios.}
    \label{tab:dgp}
    \centering\small
    \begin{tabular}{R{.6cm}R{.6cm}R{.6cm}R{2cm}R{5cm}R{2.5cm}}
        \toprule
        DGP & No. Cont. & No. Cat. & No. Noise & \(\boldsymbol{\beta}\) & Type \(\eta\) \\
        \midrule
        1 & 2 & 0 & \(\{0, 10, 50, 100\}\) & \((.5, 1)\) & Linear terms\\
        2 & \multicolumn{5}{c}{Same as DGP 2, but with probabilities \(p^3\)}\\
        3 & 5 & 5 & \(\{0, 10, 50, 100\}\) & \((.1, .2, .3, .4, .5, .01, .02, .03, .04, .05)\) & Linear terms\\
        4 & 3 & 0 & \(\{0, 10, 50, 100\}\) & \((.5, 1, .3)\) &  Non-linear terms\\
        \bottomrule
    \end{tabular}
    \begin{minipage}{\textwidth}
        \vspace{1ex}
        \scriptsize\underline{Notes:} No. Cont., No. Cat. and No. Noise correspond to the number of continuous, categorical and noise variables, respectively.
    \end{minipage}
\end{table}
}

The datasets are split into three parts: a training sample, a validation sample, and a test sample, each containing 10,000 observations.

\subsection{Trees}\label{sec:appendix-sim-trees}

\paragraph{Estimation} Regression trees are estimated using the \texttt{rpart()} function from the \texttt{R} package \{\texttt{rpart}\}. We vary a single hyperparameter: the minimum number of observations in the terminal leaves (argument \texttt{minbucket}). This hyperparameter varies according to the following sequence: \texttt{unique(round(2\^{}seq(1, 10, by = .1)))}. All variables (predictors and, if applicable, noise variables) are included in the model without transformation.

\paragraph{Metrics} Tables~\ref{tbl:trees-metrics-dgp-1}, \ref{tbl:trees-metrics-dgp-2}, \ref{tbl:trees-metrics-dgp-3}, and \ref{tbl:trees-metrics-dgp-4} supplement Table~\ref{tbl:trees-metrics}, presenting the results of 100 replications for each of the four DGPs, across the four configurations concerning noise variables introduced into the datasets for each DGP. Similar results to those shown in the main body of the paper are observed across all simulations. When the grid search optimization targets AUC (AUC* rows) rather than the KL divergence between estimated scores and true probabilities (KL* rows), there are significant losses in terms of KL divergence. At the same time, the variations in AUC are more modest. Additionally, a higher divergence is observed when optimizing according to either of the two calibration criteria (Brier* or ICI*) rather than optimizing according to KL divergence.

{
\setlength{\tabcolsep}{4pt}
\begin{table}[H]
\caption{Performance and calibration on test set over 100 replications (\textbf{DGP~1}).}
\label{tbl:trees-metrics-dgp-1}
    \centering\scriptsize
    \begin{tabular}{rcrrrrrrr}
\toprule
Noise & Selected Tree & No. Leaves & MSE & AUC & Brier & ICI & KL Div. & QR\\
\midrule
0 & Smallest & 5 (0.603) & 0.012 (0.002) & 0.720 (0.007) & 0.212 (0.002) & 0.012 (0.004) & 1.807 (0.164) & 0.930 (0.055)\\
 & Largest & 1 999 (44.796) & 0.134 (0.004) & 0.629 (0.006) & 0.333 (0.005) & 0.289 (0.008) & 4.190 (0.180) & 1.627 (0.010)\\
 & MSE* & 30 (4.800) & 0.003 (0.000) & 0.751 (0.005) & 0.203 (0.002) & 0.014 (0.004) & 0.300 (0.120) & 1.021 (0.048)\\
 & AUC* & 33 (6.639) & 0.004 (0.000) & 0.750 (0.005) & 0.203 (0.002) & 0.015 (0.004) & 0.283 (0.117) & 1.022 (0.047)\\
 & Brier* & 31 (6.170) & 0.004 (0.000) & 0.750 (0.005) & 0.203 (0.002) & 0.014 (0.004) & 0.284 (0.128) & 1.020 (0.046)\\
 & ICI* & 10 (6.782) & 0.008 (0.003) & 0.733 (0.012) & 0.208 (0.004) & 0.012 (0.004) & 1.262 (0.545) & 0.976 (0.078)\\
& KL* & 97 (38.855) & 0.008 (0.003) & 0.741 (0.007) & 0.207 (0.003) & 0.029 (0.010) & 0.086 (0.019) & 1.078 (0.038)\\
\cmidrule{1-9}
10 & Smallest & 5 (0.603) & 0.012 (0.002) & 0.720 (0.007) & 0.212 (0.002) & 0.012 (0.004) & 1.807 (0.164) & 0.930 (0.055)\\
 & Largest & 1 999 (44.796) & 0.134 (0.004) & 0.629 (0.006) & 0.333 (0.005) & 0.289 (0.008) & 4.190 (0.180) & 1.627 (0.010)\\
 & MSE* & 30 (4.721) & 0.003 (0.000) & 0.751 (0.005) & 0.203 (0.002) & 0.014 (0.004) & 0.299 (0.120) & 1.020 (0.048)\\
 & AUC* & 33 (6.836) & 0.004 (0.000) & 0.750 (0.005) & 0.203 (0.002) & 0.015 (0.004) & 0.273 (0.120) & 1.022 (0.047)\\
 & Brier* & 31 (6.460) & 0.004 (0.000) & 0.750 (0.005) & 0.203 (0.002) & 0.014 (0.004) & 0.285 (0.133) & 1.019 (0.047)\\
 & ICI* & 11 (7.039) & 0.008 (0.003) & 0.734 (0.013) & 0.208 (0.004) & 0.012 (0.004) & 1.219 (0.557) & 0.978 (0.081)\\
 & KL* & 97 (38.855) & 0.008 (0.003) & 0.741 (0.007) & 0.207 (0.003) & 0.029 (0.010) & 0.086 (0.019) & 1.078 (0.038)\\
\cmidrule{1-9}
50 & Smallest & 5 (0.603) & 0.012 (0.002) & 0.720 (0.007) & 0.212 (0.002) & 0.012 (0.004) & 1.807 (0.164) & 0.930 (0.055)\\
 & Largest & 1 999 (44.796) & 0.134 (0.004) & 0.629 (0.006) & 0.333 (0.005) & 0.289 (0.008) & 4.190 (0.180) & 1.627 (0.010)\\
 & MSE* & 30 (4.716) & 0.003 (0.000) & 0.751 (0.005) & 0.203 (0.002) & 0.014 (0.004) & 0.297 (0.120) & 1.020 (0.048)\\
 & AUC* & 33 (6.846) & 0.004 (0.000) & 0.750 (0.005) & 0.203 (0.002) & 0.015 (0.004) & 0.275 (0.117) & 1.022 (0.048)\\
 & Brier* & 31 (5.993) & 0.004 (0.000) & 0.750 (0.005) & 0.203 (0.002) & 0.014 (0.005) & 0.280 (0.130) & 1.020 (0.046)\\
 & ICI* & 10 (6.938) & 0.008 (0.003) & 0.733 (0.012) & 0.208 (0.004) & 0.012 (0.004) & 1.232 (0.552) & 0.976 (0.079)\\
 & KL* & 97 (38.855) & 0.008 (0.003) & 0.741 (0.007) & 0.207 (0.003) & 0.029 (0.010) & 0.086 (0.019) & 1.078 (0.038)\\
\cmidrule{1-9}
100 & Smallest & 5 (0.603) & 0.012 (0.002) & 0.720 (0.007) & 0.212 (0.002) & 0.012 (0.004) & 1.807 (0.164) & 0.930 (0.055)\\
 & Largest & 1 999 (44.796) & 0.134 (0.004) & 0.629 (0.006) & 0.333 (0.005) & 0.289 (0.008) & 4.190 (0.180) & 1.627 (0.010)\\
 & MSE* & 30 (4.746) & 0.003 (0.000) & 0.751 (0.005) & 0.203 (0.002) & 0.014 (0.004) & 0.298 (0.120) & 1.020 (0.048)\\
 & AUC* & 33 (7.079) & 0.004 (0.000) & 0.750 (0.005) & 0.203 (0.002) & 0.015 (0.005) & 0.278 (0.117) & 1.023 (0.046)\\
 & Brier* & 31 (6.313) & 0.004 (0.000) & 0.750 (0.005) & 0.203 (0.002) & 0.014 (0.004) & 0.288 (0.128) & 1.020 (0.047)\\
 & ICI* & 11 (6.597) & 0.008 (0.003) & 0.734 (0.012) & 0.208 (0.004) & 0.012 (0.004) & 1.219 (0.537) & 0.977 (0.078)\\
 & KL* & 97 (38.855) & 0.008 (0.003) & 0.741 (0.007) & 0.207 (0.003) & 0.029 (0.010) & 0.086 (0.019) & 1.078 (0.038)\\
 \bottomrule
\end{tabular}
\begin{minipage}{\textwidth}
\vspace{1ex}
\scriptsize\underline{Notes:} Noise: number of noise variables added, No. Leaves: average number of leaves over the 100 replicatiobns, MSE: Mean squared error, AUC: area under the receiver operating characteristic cuve, Brier: Brier's Score, ICI: integrated calibration index, KL Div.: Kullback-Leibler divergence, QR: ratio of the difference between the 90th and 10th percentiles of the scores to the difference between the 90th and 10th percentiles of the true probabilities. Values are reported as mean (standard deviation) over 100 replications. Multiple trees are estimated, with varying minimal number of observations in terminal leaf nodes. The metrics are reported for the tree with the smallest/largest number of leaves (Smallest/Largest), the tree that minimizes the MSE (MSE*), maximizes the AUC (AUC*), minimizes the Brier's score (Brier*), the ICI (ICI*), and the Kullback-Leibler divergence between the distribution of the underlying probabilities and the distribution of estimated scores (KL*).
\end{minipage}
\end{table}
}

%
%
{
\setlength{\tabcolsep}{4pt}
\begin{table}[H]
\caption{Performance and calibration on test set over 100 replications (\textbf{DGP~2}).}
\label{tbl:trees-metrics-dgp-2}
    \centering\scriptsize
    \begin{tabular}{rcrrrrrrr}
\toprule
Noise & Selected Tree & No. Leaves & MSE & AUC & Brier & ICI & KL Div. & QR\\
\midrule
0 & Smallest & 5 (0.559) & 0.014 (0.002) & 0.788 (0.008) & 0.131 (0.002) & 0.008 (0.003) & 1.268 (0.138) & 0.892 (0.104)\\
 & Largest & 1 289 (30.810) & 0.079 (0.003) & 0.676 (0.009) & 0.196 (0.004) & 0.181 (0.005) & 1.983 (0.220) & 1.923 (0.041)\\
 & MSE* & 44 (7.936) & 0.003 (0.000) & 0.833 (0.005) & 0.120 (0.002) & 0.011 (0.003) & 0.213 (0.075) & 1.028 (0.057)\\
 & AUC* & 38 (8.178) & 0.003 (0.000) & 0.833 (0.005) & 0.120 (0.002) & 0.010 (0.003) & 0.239 (0.099) & 1.030 (0.058)\\
 & Brier* & 42 (10.284) & 0.003 (0.000) & 0.833 (0.005) & 0.120 (0.002) & 0.011 (0.003) & 0.227 (0.097) & 1.031 (0.055)\\
 & ICI* & 13 (10.757) & 0.009 (0.004) & 0.810 (0.018) & 0.126 (0.005) & 0.008 (0.003) & 0.842 (0.387) & 0.973 (0.125)\\
 & KL* & 145 (49.047) & 0.008 (0.003) & 0.817 (0.009) & 0.125 (0.003) & 0.027 (0.009) & 0.064 (0.016) & 1.069 (0.041)\\
\cmidrule{1-9}
10 & Smallest & 5 (0.559) & 0.014 (0.002) & 0.788 (0.008) & 0.131 (0.003) & 0.008 (0.003) & 1.268 (0.138) & 0.892 (0.104)\\
 & Largest & 1 289 (30.810) & 0.079 (0.003) & 0.675 (0.009) & 0.196 (0.004) & 0.181 (0.005) & 1.983 (0.220) & 1.923 (0.041)\\
 & MSE* & 44 (7.937) & 0.003 (0.000) & 0.833 (0.005) & 0.120 (0.002) & 0.011 (0.003) & 0.213 (0.075) & 1.026 (0.056)\\
 & AUC* & 38 (7.755) & 0.003 (0.000) & 0.833 (0.005) & 0.121 (0.002) & 0.010 (0.003) & 0.241 (0.097) & 1.034 (0.058)\\
 & Brier* & 42 (10.641) & 0.003 (0.000) & 0.833 (0.005) & 0.120 (0.002) & 0.011 (0.003) & 0.225 (0.095) & 1.030 (0.054)\\
 & ICI* & 14 (11.569) & 0.008 (0.004) & 0.811 (0.019) & 0.126 (0.005) & 0.008 (0.003) & 0.816 (0.394) & 0.972 (0.128)\\
 & KL* & 145 (49.047) & 0.008 (0.003) & 0.817 (0.009) & 0.125 (0.003) & 0.027 (0.009) & 0.064 (0.016) & 1.069 (0.041)\\
\cmidrule{1-9}
50 & Smallest & 5 (0.559) & 0.014 (0.002) & 0.788 (0.008) & 0.131 (0.003) & 0.008 (0.003) & 1.268 (0.138) & 0.892 (0.104)\\
 & Largest & 1 289 (30.810) & 0.079 (0.003) & 0.675 (0.009) & 0.196 (0.004) & 0.181 (0.005) & 1.983 (0.220) & 1.923 (0.041)\\
 & MSE* & 44 (7.913) & 0.003 (0.000) & 0.833 (0.005) & 0.120 (0.002) & 0.011 (0.003) & 0.213 (0.074) & 1.026 (0.056)\\
 & AUC* & 38 (8.115) & 0.003 (0.000) & 0.833 (0.005) & 0.121 (0.002) & 0.010 (0.003) & 0.242 (0.095) & 1.034 (0.058)\\
 & Brier* & 42 (10.475) & 0.003 (0.000) & 0.833 (0.005) & 0.120 (0.002) & 0.011 (0.003) & 0.225 (0.096) & 1.027 (0.057)\\
 & ICI* & 14 (11.504) & 0.009 (0.004) & 0.811 (0.018) & 0.126 (0.005) & 0.008 (0.003) & 0.827 (0.392) & 0.970 (0.129)\\
 & KL* & 145 (49.047) & 0.008 (0.003) & 0.817 (0.009) & 0.125 (0.003) & 0.027 (0.009) & 0.064 (0.016) & 1.069 (0.041)\\
\cmidrule{1-9}
100 & Smallest & 5 (0.559) & 0.014 (0.002) & 0.788 (0.008) & 0.131 (0.003) & 0.008 (0.003) & 1.268 (0.138) & 0.892 (0.104)\\
 & Largest & 1 289 (30.810) & 0.079 (0.003) & 0.677 (0.009) & 0.196 (0.004) & 0.181 (0.005) & 1.983 (0.220) & 1.923 (0.041)\\
 & MSE* & 44 (7.441) & 0.003 (0.000) & 0.833 (0.005) & 0.120 (0.002) & 0.011 (0.003) & 0.212 (0.075) & 1.027 (0.057)\\
 & AUC* & 39 (8.365) & 0.003 (0.000) & 0.833 (0.005) & 0.120 (0.002) & 0.010 (0.003) & 0.236 (0.095) & 1.030 (0.053)\\
 & Brier* & 42 (10.840) & 0.003 (0.000) & 0.833 (0.005) & 0.120 (0.002) & 0.011 (0.003) & 0.224 (0.096) & 1.031 (0.054)\\
 & ICI* & 14 (11.434) & 0.009 (0.004) & 0.811 (0.018) & 0.126 (0.005) & 0.008 (0.003) & 0.817 (0.395) & 0.967 (0.127)\\
 & KL* & 145 (49.047) & 0.008 (0.003) & 0.817 (0.009) & 0.125 (0.003) & 0.027 (0.009) & 0.064 (0.016) & 1.069 (0.041)\\
\bottomrule
\end{tabular}
\begin{minipage}{\textwidth}
\vspace{1ex}
\scriptsize\underline{Notes:} Noise: number of noise variables added, No. Leaves: average number of leaves over the 100 replicatiobns, MSE: Mean squared error, AUC: area under the receiver operating characteristic cuve, Brier: Brier's Score, ICI: integrated calibration index, KL Div.: Kullback-Leibler divergence, QR: ratio of the difference between the 90th and 10th percentiles of the scores to the difference between the 90th and 10th percentiles of the true probabilities. Values are reported as mean (standard deviation) over 100 replications. Multiple trees are estimated, with varying minimal number of observations in terminal leaf nodes. The metrics are reported for the tree with the smallest/largest number of leaves (Smallest/Largest), the tree that minimizes the MSE (MSE*), maximizes the AUC (AUC*), minimizes the Brier's score (Brier*), the ICI (ICI*), and the Kullback-Leibler divergence between the distribution of the underlying probabilities and the distribution of estimated scores (KL*).
\end{minipage}
\end{table}
}

%
%
{
\setlength{\tabcolsep}{3.8pt}
\begin{table}[H]
\caption{Performance and calibration on test set over 100 replications (\textbf{DGP~3}).}
\label{tbl:trees-metrics-dgp-3}
    \centering\scriptsize
    \begin{tabular}{rcrrrrrrr}
\toprule
Noise & Selected Tree & No. Leaves & MSE & AUC & Brier & ICI & KL Div. & QR \\
\midrule
0 & Smallest & 5 (0.628) & 0.025 (0.000) & 0.546 (0.006) & 0.243 (0.001) & 0.013 (0.005) & 1.421 (0.224) & 0.345 (0.054)\\
 & Largest & 1 825 (190.935) & 0.157 (0.014) & 0.508 (0.006) & 0.375 (0.014) & 0.303 (0.027) & 6.075 (0.996) & 2.295 (0.017)\\
 & MSE* & 7 (1.515) & 0.025 (0.000) & 0.548 (0.007) & 0.243 (0.001) & 0.014 (0.005) & 1.303 (0.226) & 0.356 (0.054)\\
 & AUC* & 9 (3.518) & 0.025 (0.000) & 0.547 (0.006) & 0.244 (0.001) & 0.015 (0.006) & 1.231 (0.217) & 0.353 (0.056)\\
 & Brier* & 8 (2.536) & 0.025 (0.000) & 0.547 (0.006) & 0.243 (0.001) & 0.014 (0.005) & 1.285 (0.225) & 0.356 (0.056)\\
 & ICI* & 6 (1.788) & 0.025 (0.000) & 0.547 (0.007) & 0.243 (0.001) & 0.014 (0.005) & 1.335 (0.228) & 0.350 (0.055)\\
 & KL* & 253 (31.319) & 0.049 (0.003) & 0.516 (0.007) & 0.267 (0.003) & 0.129 (0.009) & 0.050 (0.015) & 1.004 (0.059)\\
\cmidrule{1-9}
10 & Smallest & 5 (0.628) & 0.025 (0.000) & 0.546 (0.006) & 0.243 (0.001) & 0.013 (0.005) & 1.421 (0.224) & 0.345 (0.054)\\
 & Largest & 1 825 (190.935) & 0.157 (0.014) & 0.508 (0.006) & 0.375 (0.014) & 0.303 (0.027) & 6.075 (0.996) & 2.295 (0.017)\\
 & MSE* & 7 (1.537) & 0.025 (0.000) & 0.548 (0.007) & 0.243 (0.001) & 0.014 (0.006) & 1.294 (0.224) & 0.358 (0.053)\\
 & AUC* & 9 (3.642) & 0.025 (0.000) & 0.547 (0.007) & 0.244 (0.001) & 0.016 (0.007) & 1.215 (0.215) & 0.360 (0.055)\\
 & Brier* & 8 (2.505) & 0.025 (0.000) & 0.547 (0.007) & 0.243 (0.001) & 0.014 (0.006) & 1.273 (0.217) & 0.354 (0.056)\\
 & ICI* & 6 (1.794) & 0.025 (0.000) & 0.547 (0.007) & 0.243 (0.001) & 0.014 (0.005) & 1.326 (0.230) & 0.352 (0.054)\\
 & KL* & 253 (31.319) & 0.049 (0.003) & 0.516 (0.006) & 0.267 (0.003) & 0.129 (0.009) & 0.050 (0.015) & 1.004 (0.059)\\
\cmidrule{1-9}
50 & Smallest & 5 (0.628) & 0.025 (0.000) & 0.546 (0.006) & 0.243 (0.001) & 0.013 (0.005) & 1.421 (0.224) & 0.345 (0.054)\\
 & Largest & 1 825 (190.935) & 0.157 (0.014) & 0.508 (0.006) & 0.375 (0.014) & 0.303 (0.027) & 6.075 (0.996) & 2.295 (0.017)\\
 & MSE* & 7 (1.524) & 0.025 (0.000) & 0.548 (0.007) & 0.243 (0.001) & 0.014 (0.006) & 1.286 (0.226) & 0.355 (0.055)\\
 & AUC* & 9 (3.909) & 0.025 (0.000) & 0.547 (0.007) & 0.244 (0.001) & 0.016 (0.007) & 1.215 (0.204) & 0.361 (0.055)\\
 & Brier* & 8 (2.521) & 0.025 (0.000) & 0.548 (0.007) & 0.243 (0.001) & 0.014 (0.005) & 1.273 (0.218) & 0.356 (0.058)\\
 & ICI* & 6 (1.808) & 0.025 (0.000) & 0.547 (0.007) & 0.243 (0.001) & 0.014 (0.005) & 1.339 (0.237) & 0.353 (0.055)\\
 & KL* & 253 (31.319) & 0.049 (0.003) & 0.515 (0.007) & 0.267 (0.003) & 0.129 (0.009) & 0.050 (0.015) & 1.004 (0.059)\\
\cmidrule{1-9}
100 & Smallest & 5 (0.628) & 0.025 (0.000) & 0.547 (0.006) & 0.243 (0.001) & 0.013 (0.005) & 1.421 (0.224) & 0.345 (0.054)\\
 & Largest & 1 825 (190.935) & 0.157 (0.014) & 0.508 (0.006) & 0.375 (0.014) & 0.303 (0.027) & 6.075 (0.996) & 2.295 (0.017)\\
 & MSE* & 7 (1.504) & 0.025 (0.000) & 0.548 (0.007) & 0.243 (0.001) & 0.014 (0.005) & 1.291 (0.224) & 0.357 (0.053)\\
 & AUC* & 9 (3.640) & 0.025 (0.000) & 0.547 (0.007) & 0.244 (0.001) & 0.016 (0.006) & 1.237 (0.232) & 0.358 (0.053)\\
 & Brier* & 8 (2.517) & 0.025 (0.000) & 0.548 (0.007) & 0.243 (0.001) & 0.014 (0.006) & 1.293 (0.219) & 0.354 (0.059)\\
 & ICI* & 6 (1.716) & 0.025 (0.000) & 0.547 (0.007) & 0.243 (0.001) & 0.014 (0.005) & 1.339 (0.228) & 0.353 (0.054)\\
 & KL* & 253 (31.319) & 0.049 (0.003) & 0.516 (0.006) & 0.267 (0.003) & 0.129 (0.009) & 0.050 (0.015) & 1.004 (0.059)\\
\bottomrule
\end{tabular}
\begin{minipage}{\textwidth}
\vspace{1ex}
\scriptsize\underline{Notes:} Noise: number of noise variables added, No. Leaves: average number of leaves over the 100 replicatiobns, MSE: Mean squared error, AUC: area under the receiver operating characteristic cuve, Brier: Brier's Score, ICI: integrated calibration index, KL Div.: Kullback-Leibler divergence, QR: ratio of the difference between the 90th and 10th percentiles of the scores to the difference between the 90th and 10th percentiles of the true probabilities. Values are reported as mean (standard deviation) over 100 replications. Multiple trees are estimated, with varying minimal number of observations in terminal leaf nodes. The metrics are reported for the tree with the smallest/largest number of leaves (Smallest/Largest), the tree that minimizes the MSE (MSE*), maximizes the AUC (AUC*), minimizes the Brier's score (Brier*), the ICI (ICI*), and the Kullback-Leibler divergence between the distribution of the underlying probabilities and the distribution of estimated scores (KL*).
\end{minipage}
\end{table}
}

%
%
{
\setlength{\tabcolsep}{3.9pt}
\begin{table}[H]
\caption{Performance and calibration on test set over 100 replications (\textbf{DGP~4}).}
\label{tbl:trees-metrics-dgp-4}
    \centering\scriptsize
    \begin{tabular}{rcrrrrrrr}
\toprule
Noise & Selected Tree & No. Leaves & MSE & AUC & Brier & ICI & KL Div. & QR \\
\midrule
0 & Smallest & 6 (0.522) & 0.035 (0.001) & 0.636 (0.007) & 0.235 (0.002) & 0.012 (0.005) & 1.515 (0.211) & 0.606 (0.054)\\
 & Largest & 2 052 (95.136) & 0.170 (0.007) & 0.563 (0.006) & 0.370 (0.008) & 0.314 (0.015) & 3.125 (0.207) & 1.647 (0.011)\\
 & MSE* & 32 (8.332) & 0.032 (0.001) & 0.654 (0.006) & 0.232 (0.002) & 0.021 (0.006) & 0.525 (0.135) & 0.681 (0.053)\\
 & AUC* & 36 (11.563) & 0.032 (0.001) & 0.654 (0.006) & 0.232 (0.002) & 0.023 (0.007) & 0.489 (0.153) & 0.688 (0.055)\\
 & Brier* & 32 (10.839) & 0.032 (0.001) & 0.654 (0.006) & 0.232 (0.002) & 0.021 (0.006) & 0.526 (0.157) & 0.685 (0.056)\\
 & ICI* & 9 (3.016) & 0.034 (0.001) & 0.642 (0.007) & 0.234 (0.002) & 0.012 (0.005) & 1.231 (0.295) & 0.633 (0.051)\\
 & KL* & 317 (56.534) & 0.055 (0.005) & 0.616 (0.008) & 0.255 (0.005) & 0.109 (0.014) & 0.046 (0.013) & 1.015 (0.047)\\
\cmidrule{1-9}
10 & Smallest & 6 (0.522) & 0.035 (0.001) & 0.636 (0.007) & 0.235 (0.002) & 0.012 (0.005) & 1.515 (0.211) & 0.606 (0.054)\\
 & Largest & 2 052 (95.136) & 0.170 (0.007) & 0.563 (0.006) & 0.370 (0.008) & 0.314 (0.015) & 3.125 (0.207) & 1.647 (0.011)\\
 & MSE* & 31 (7.952) & 0.032 (0.001) & 0.654 (0.006) & 0.232 (0.002) & 0.021 (0.006) & 0.527 (0.132) & 0.679 (0.051)\\
 & AUC* & 37 (12.267) & 0.032 (0.001) & 0.654 (0.006) & 0.232 (0.002) & 0.023 (0.007) & 0.484 (0.144) & 0.690 (0.056)\\
 & Brier* & 32 (10.939) & 0.032 (0.001) & 0.654 (0.006) & 0.232 (0.002) & 0.021 (0.006) & 0.530 (0.154) & 0.686 (0.057)\\
 & ICI* & 9 (3.023) & 0.034 (0.001) & 0.642 (0.008) & 0.234 (0.002) & 0.012 (0.005) & 1.230 (0.292) & 0.632 (0.053)\\
 & KL* & 317 (56.534) & 0.055 (0.005) & 0.616 (0.008) & 0.255 (0.005) & 0.109 (0.014) & 0.046 (0.013) & 1.015 (0.047)\\
\cmidrule{1-9}
50 & Smallest & 6 (0.522) & 0.035 (0.001) & 0.636 (0.007) & 0.235 (0.002) & 0.012 (0.005) & 1.515 (0.211) & 0.606 (0.054)\\
 & Largest & 2 052 (95.136) & 0.170 (0.007) & 0.563 (0.006) & 0.370 (0.008) & 0.314 (0.015) & 3.125 (0.207) & 1.647 (0.011)\\
 & MSE* & 31 (8.328) & 0.032 (0.001) & 0.654 (0.005) & 0.232 (0.002) & 0.021 (0.006) & 0.526 (0.135) & 0.680 (0.051)\\
 & AUC* & 36 (11.757) & 0.032 (0.001) & 0.654 (0.006) & 0.232 (0.002) & 0.023 (0.007) & 0.490 (0.152) & 0.691 (0.053)\\
 & Brier* & 32 (10.692) & 0.032 (0.001) & 0.654 (0.006) & 0.232 (0.002) & 0.021 (0.006) & 0.526 (0.156) & 0.684 (0.057)\\
 & ICI* & 9 (3.078) & 0.034 (0.001) & 0.642 (0.007) & 0.234 (0.002) & 0.012 (0.005) & 1.230 (0.306) & 0.633 (0.051)\\
 & KL* & 317 (56.534) & 0.055 (0.005) & 0.616 (0.008) & 0.255 (0.005) & 0.109 (0.014) & 0.046 (0.013) & 1.015 (0.047)\\
\cmidrule{1-9}
100 & Smallest & 6 (0.522) & 0.035 (0.001) & 0.636 (0.007) & 0.235 (0.002) & 0.012 (0.005) & 1.515 (0.211) & 0.606 (0.054)\\
 & Largest & 2 052 (95.136) & 0.170 (0.007) & 0.563 (0.006) & 0.370 (0.008) & 0.314 (0.015) & 3.125 (0.207) & 1.647 (0.011)\\
 & MSE* & 31 (8.266) & 0.032 (0.001) & 0.654 (0.006) & 0.232 (0.002) & 0.021 (0.006) & 0.524 (0.135) & 0.680 (0.054)\\
 & AUC* & 36 (11.359) & 0.032 (0.001) & 0.654 (0.006) & 0.232 (0.002) & 0.023 (0.007) & 0.482 (0.151) & 0.688 (0.055)\\
 & Brier* & 32 (11.051) & 0.032 (0.001) & 0.654 (0.006) & 0.232 (0.002) & 0.021 (0.006) & 0.528 (0.159) & 0.685 (0.056)\\
 & ICI* & 9 (3.095) & 0.034 (0.001) & 0.642 (0.008) & 0.234 (0.002) & 0.012 (0.005) & 1.229 (0.293) & 0.633 (0.050)\\
 & KL* & 317 (56.534) & 0.055 (0.005) & 0.616 (0.008) & 0.255 (0.005) & 0.109 (0.014) & 0.046 (0.013) & 1.015 (0.047)\\
\bottomrule
\end{tabular}
\begin{minipage}{\textwidth}
\vspace{1ex}
\scriptsize\underline{Notes:} Noise: number of noise variables added, No. Leaves: average number of leaves over the 100 replicatiobns, MSE: Mean squared error, AUC: area under the receiver operating characteristic cuve, Brier: Brier's Score, ICI: integrated calibration index, KL Div.: Kullback-Leibler divergence, QR: ratio of the difference between the 90th and 10th percentiles of the scores to the difference between the 90th and 10th percentiles of the true probabilities. Values are reported as mean (standard deviation) over 100 replications. Multiple trees are estimated, with varying minimal number of observations in terminal leaf nodes. The metrics are reported for the tree with the smallest/largest number of leaves (Smallest/Largest), the tree that minimizes the MSE (MSE*), maximizes the AUC (AUC*), minimizes the Brier's score (Brier*), the ICI (ICI*), and the Kullback-Leibler divergence between the distribution of the underlying probabilities and the distribution of estimated scores (KL*).
\end{minipage}
\end{table}
}

\paragraph{Distribution of Scores} Figs.~\ref{fig:bp_synthetic_trees_1}, \ref{fig:bp_synthetic_trees_2}, \ref{fig:bp_synthetic_trees_3}, and \ref{fig:bp_synthetic_trees_4} show, like Fig.~\ref{fig:tree-scenario1-hist-scores} in the main body, the distribution of scores obtained for each of the 4 DGPs, according to the number of noise variables introduced into the training set. The distributions depicted are those obtained during a single replication. Regardless of the DGP, the distribution of scores obtained by maximizing AUC (AUC*) bears little resemblance to that of the true probabilities. When the tree's hyperparameters are chosen to minimize calibration (Brier* or ICI*), the score distributions are also far from the true probabilities. In contrast, selecting hyperparameters based on the KL criterion results in scores whose distribution is much closer to the true probability.

\foreach \dgp in {1,2,3,4}
{
\begin{figure}[H]
    \centering
    \includegraphics[width=.95\textwidth]{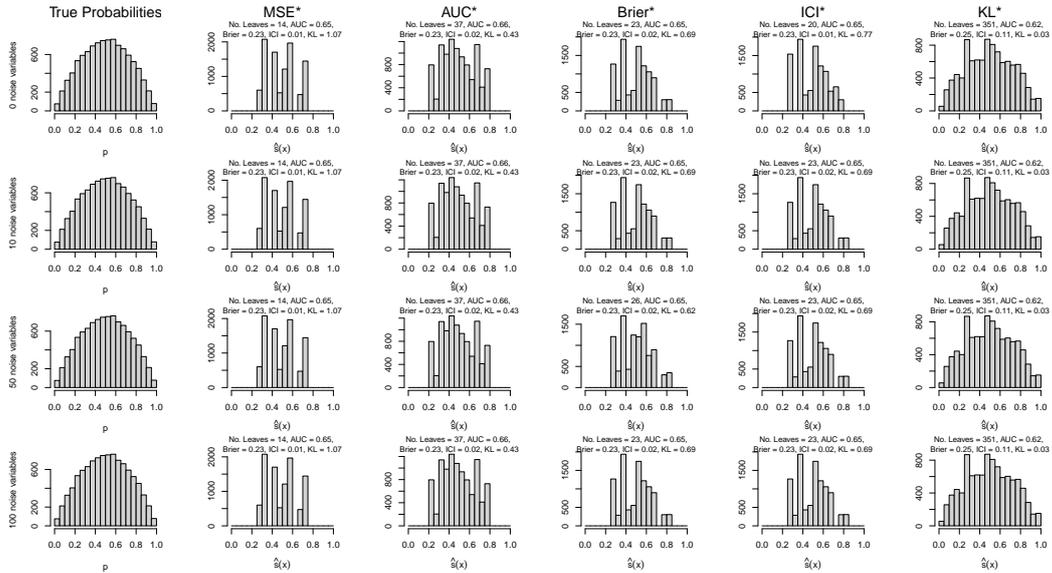}
    \caption{Distribution of true probabilities and estimated scores for \textbf{Trees} under \textbf{DGP \dgp}: single replication across various numbers of noise variables.}
    \label{fig:bp_synthetic_trees_\dgp}
    \begin{minipage}{\textwidth}
\vspace{1ex}
\scriptsize\underline{Notes:} MSE*, AUC*, Brier*, ICI*, and KL* refer to the models chosen based on the optimization, on the validation set, of the MSE, the AUC, the Brier score, the ICI, and the Kullback-Leibler divergence between the model's score distributions and the underlying probabilities, respectively. The metrics shown are computed on the test set.
\end{minipage}
\end{figure}
}

\paragraph{Relationship Between KL Divergence, Calibration, and Number of Leaves in the Tree}
Figs.~\ref{fig:trees-kl-calib-brier-leaves-all} and \ref{fig:trees-kl-calib-ici-leaves-all} complement Fig.~\ref{fig:trees-kl-calib-leaves-1} from the main body by showing the evolution of KL distance and calibration according to the number of leaves in the tree (indicated by the direction of the arrow). When the Brier score is used as the calibration measure (Fig.~\ref{fig:trees-kl-calib-brier-leaves-all}), an increase in the number of leaves initially leads to a degradation of calibration for DGPs 1 and 2. For all DGPs, as the tree becomes more complex, the Brier score improves in the same direction as the number of leaves. When calibration is measured using the ICI (Fig.~\ref{fig:trees-kl-calib-ici-leaves-all}), an increase in the number of leaves consistently leads to worse calibration. Concurrently, the KL distance shows a convex relationship with the number of tree leaves, suggesting the existence of an optimum.

\begin{figure}[H]
    \centering
    \includegraphics[width = .8\textwidth]{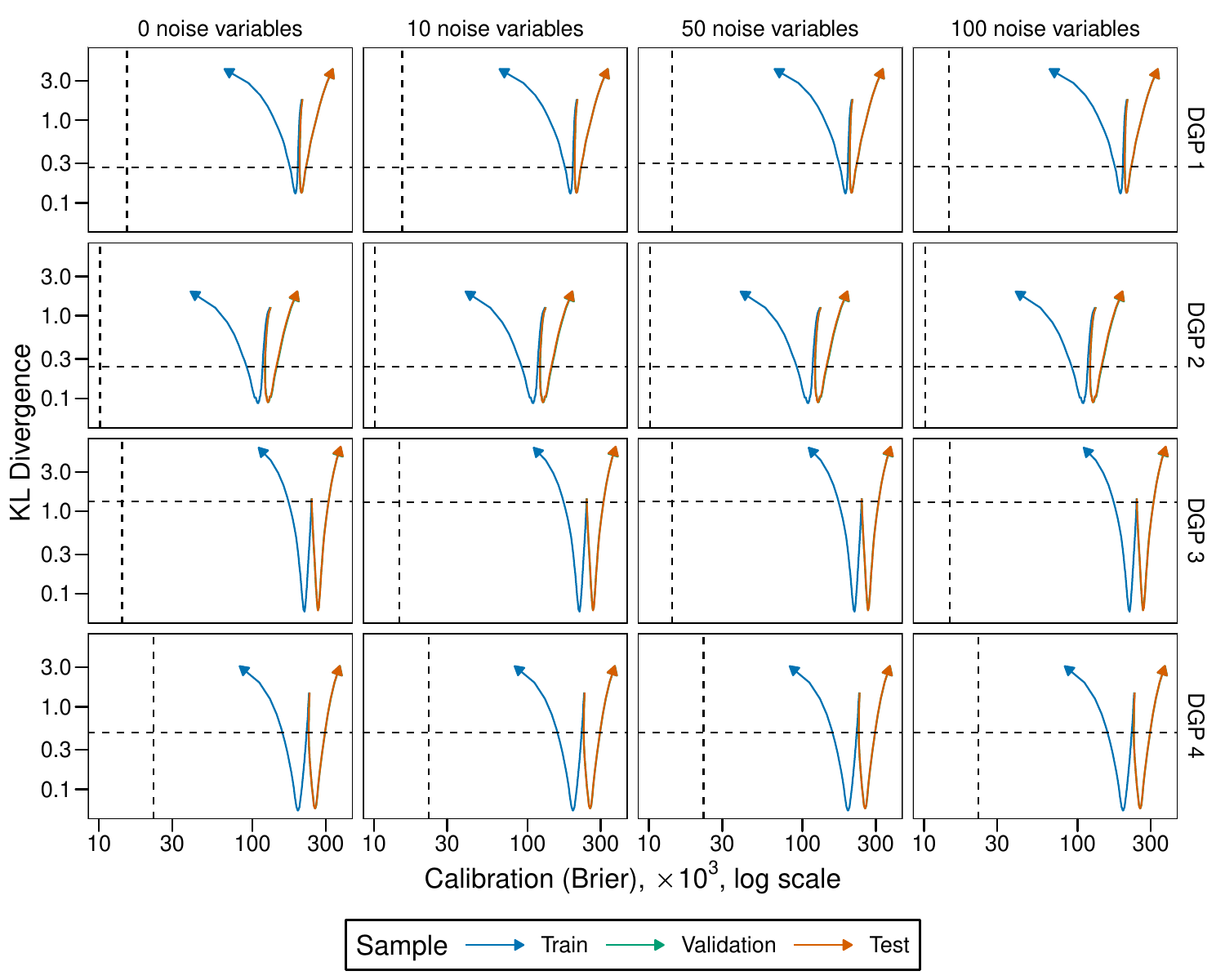}
    \caption{KL Div. and Calibration (\textbf{Brier score}) vs. Average Number of Tree Leaves.}
    \label{fig:trees-kl-calib-brier-leaves-all}
\end{figure}

\begin{figure}[H]
    \centering
    \includegraphics[width = .8\textwidth]{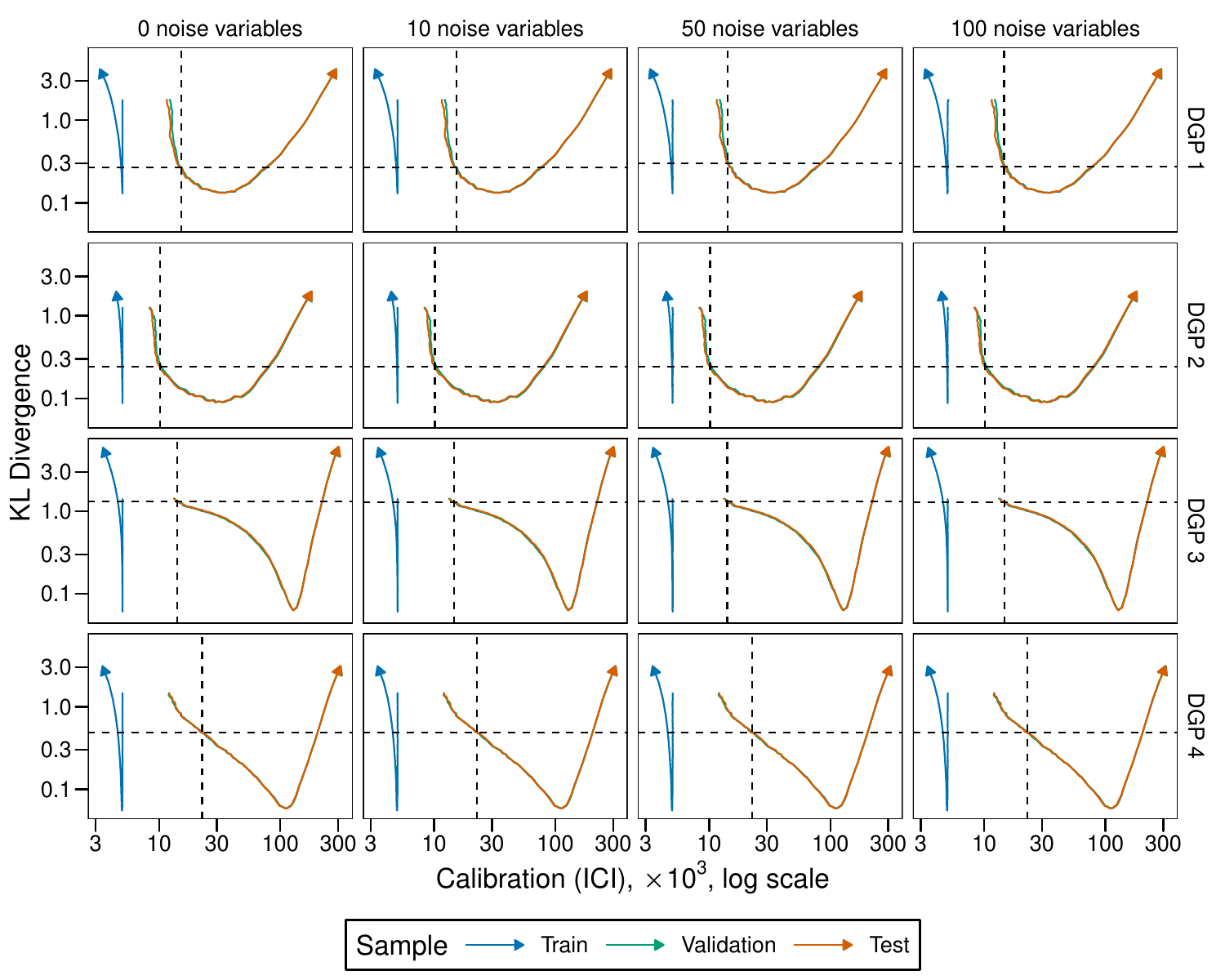}
    \caption{KL Div. and Calibration (\textbf{ICI}) vs. Average Number of Tree Leaves.}
    \label{fig:trees-kl-calib-ici-leaves-all}
\end{figure}

\subsection{Random Forests}\label{sec:appendix-rf}

\paragraph{Estimation} Random forests are estimated using the \texttt{ranger()} function from the \texttt{R} package \{\texttt{ranger}\}. The number of trees per forest is set to 250. The hyperparameters that vary are the number of candidate variables (argument \texttt{mtry}): 2, and when there are enough variables in the dataset, 4 and 6. The second hyperparameter we vary is the minimum number of observations per terminal leaf (argument \texttt{min.bucket}), which varies according to the following sequence: \texttt{unique(round(2\^{}seq(1, 14, by = .4)))}. All variables (predictors and, if applicable, noise variables) are included in the model without transformation.

\paragraph{Distribution of Scores} Figs.~\ref{fig:bp_synthetic_forest_1}, \ref{fig:bp_synthetic_forest_2}, \ref{fig:bp_synthetic_forest_3}, and \ref{fig:bp_synthetic_forest_4} show the distribution of the true probabilities (left column) and the distribution of scores estimated for forests whose hyperparameters were selected based on specific criteria: minimization of the true MSE (`MSE*'), maximization of the AUC (AUC*), minimization of the Brier score (Brier*), the ICI (ICI*), and the KL divergence between the estimated scores and the distribution of the true probabilities. Each of the 4 rows in the figures corresponds to a different number of noise variables introduced into the dataset (0, 10, 50, 100). Each figure refers to a different DGP (1, 2, 3, 4).

\foreach \dgp in {1,2,3,4}
{
    \begin{figure}[H]
        \centering
        \includegraphics[width=.95\textwidth]{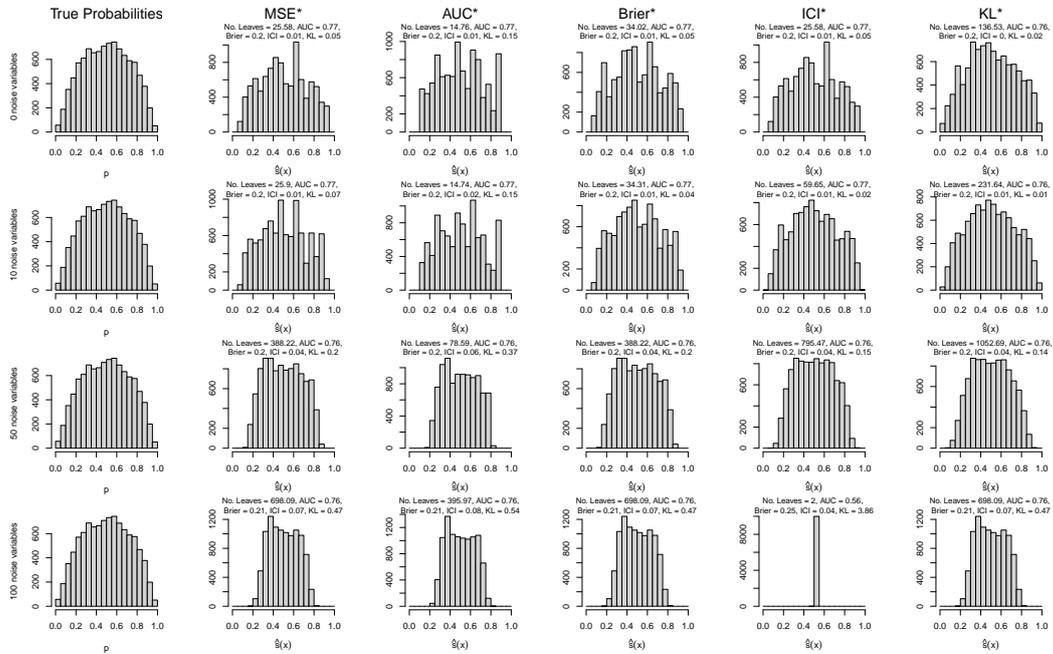}
        \caption{Distribution of true probabilities and estimated scores for \textbf{Random Forests} under \textbf{DGP \dgp}: single replication across various numbers of noise variables.}
        \label{fig:bp_synthetic_forest_\dgp}
        \begin{minipage}{\textwidth}
    \vspace{1ex}
    \scriptsize\underline{Notes:} Smallest and Largest refer to the models for which the average number of leaves in the trees of the forest is the smallest and largest, respectively. AUC* refers to the model chosen based on the maximization of AUC on the validation set. KL* denote the selection of the best model based on minimizing the Kullback-Leibler divergence between the model's score distributions and the underlying probabilities.
    \end{minipage}
    \end{figure}
}

\subsection{Extreme Gradient Boosting}\label{sec:appendix-xgb}

\paragraph{Estimation} XGBoost models are estimated using the \texttt{xgb.train} function from the \texttt{R} package \{\texttt{xgboost}\}. The learning rate is set to 0.3. The tree depth (argument \texttt{max\_depth}) varies according to the following values: 2, 4, 6. The number of boosting iterations (argument \texttt{nrounds}) ranges from 1 to 400. All variables (predictors and, if applicable, noise variables) are included in the model without transformation.

\paragraph{Distribution of Scores} Figs.~\ref{fig:bp_synthetic_xgb_1}, \ref{fig:bp_synthetic_xgb_2}, \ref{fig:bp_synthetic_xgb_3}, and \ref{fig:bp_synthetic_xgb_4} show the distribution of the true probabilities (left column) and the distribution of scores estimated for XGBoost whose hyperparameters were selected based on specific criteria: minimization of the true MSE (MSE*), maximization of the AUC (AUC*), minimization of the Brier score (Brier*), the ICI (ICI*), and the KL divergence between the estimated scores and the distribution of the true probabilities. The optimization is made on the validation set. The metrics of interest are then computed on the test set and reported on the figures. Each of the 4 rows in the figures corresponds to a different number of noise variables introduced into the dataset (0, 10, 50, 100). Each figure refers to a different DGP (1, 2, 3, 4).

\foreach \dgp in {1,2,3,4}
{
        \begin{figure}[H]
        \centering
        \includegraphics[width=.95\textwidth]{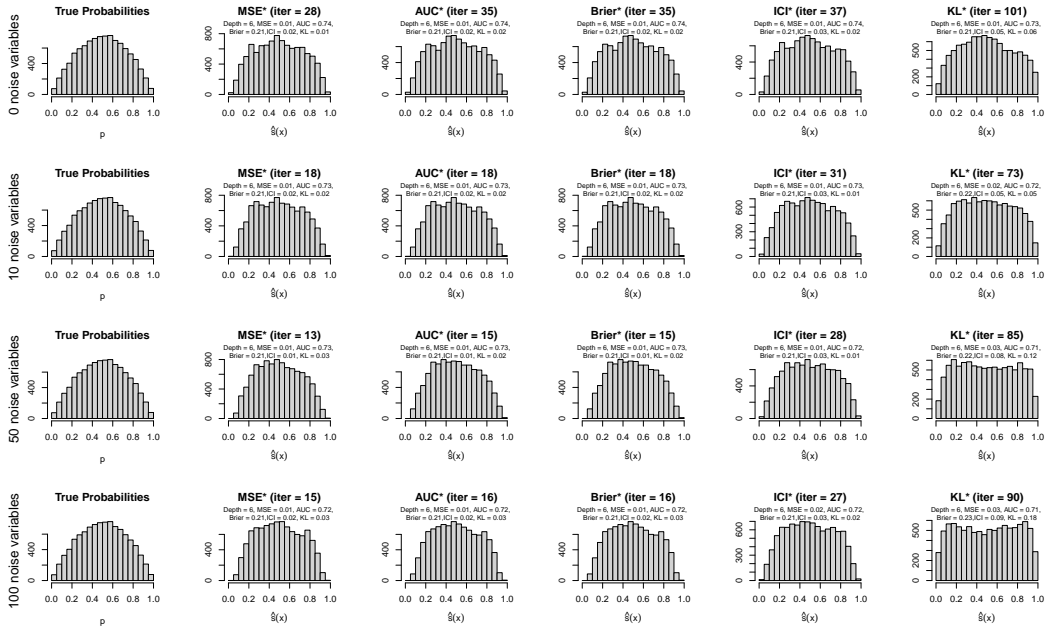}
        \caption{Distribution of true probabilities and estimated scores for \textbf{Extreme Gradient Boosting} under \textbf{DGP \dgp}: single replication.}
        \label{fig:bp_synthetic_xgb_\dgp}
        \begin{minipage}{\textwidth}
        \vspace{1ex}
        \scriptsize\underline{Notes:} MSE* refers to the model iteration selected by minimizing MSE on the validation set. AUC* refers to the model chosen by maximizing AUC. Brier*, ICI*, and KL* denote models selected based on minimizing the Brier score, ICI, and Kullback-Leibler divergence, respectively.
        \end{minipage}
        \end{figure}
}

\subsection{Generalized Linear Models}

\paragraph{Estimation} GLM models are estimated using the \texttt{glm()} function from the \texttt{R} package \{\texttt{stats}\}, using a logistic link function. All variables (predictors and, if applicable, noise variables) are included in the model without transformation.

\paragraph{Distribution of Scores} Fig.~\ref{fig:bp_synthetic_glm} shows the distribution of the true probabilities (left column) and the distribution of scores estimated with GLMs on the test set. The other 4 columns in the figure correspond to a different number of noise variables introduced into the dataset (0, 10, 50, 100). Each row corresponds to a DGP (1, 2, 3, 4).

\begin{figure}[H]
    \centering
    \includegraphics[width = \textwidth]{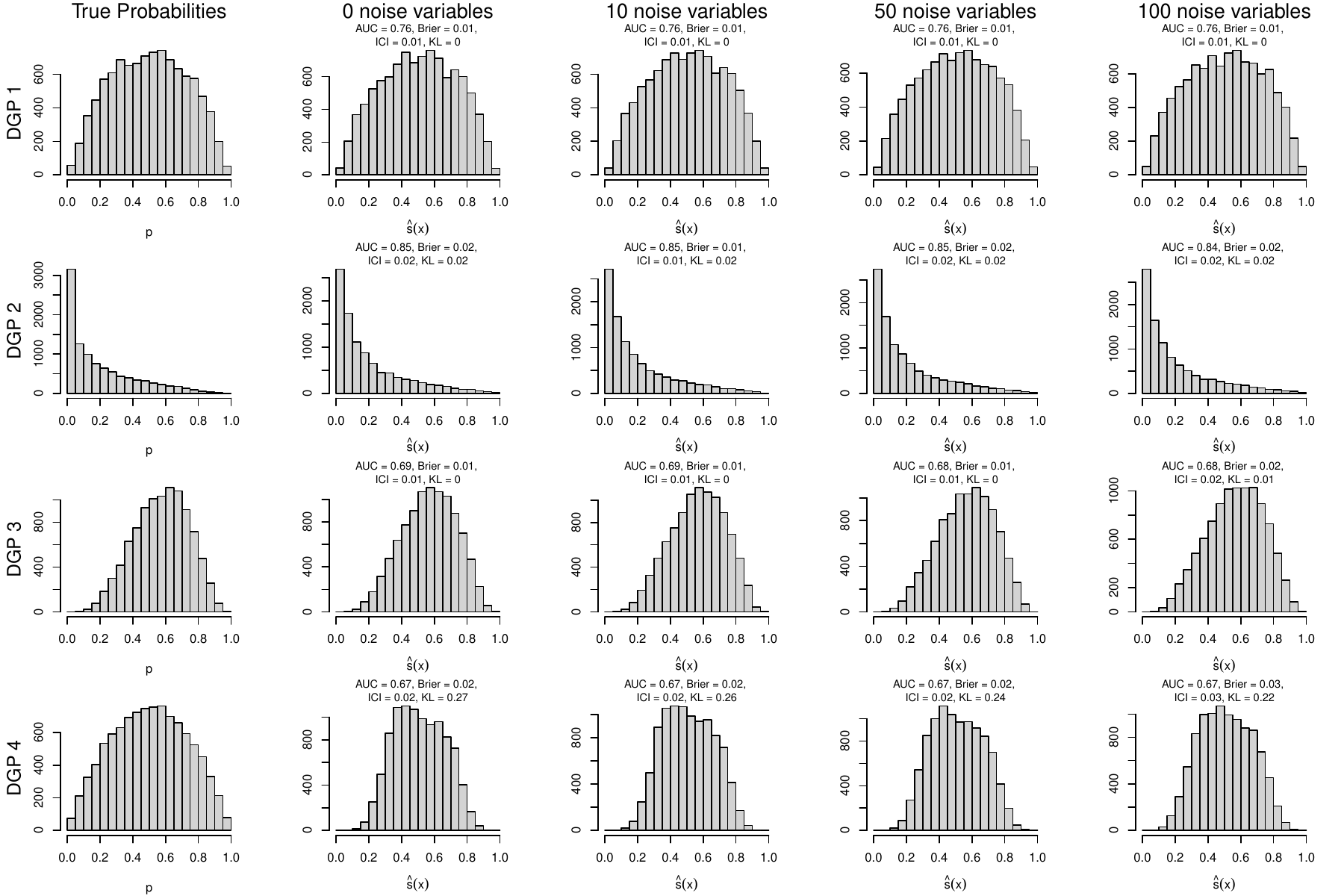}
    \caption{Distribution of true probabilities and estimated scores for \textbf{Generalized Linear Models} under the four DGPs (in rows): single replication across different number of noise variables (in columns).}
    \label{fig:bp_synthetic_glm}
\end{figure}

\subsection{Generalized Additive Models}

\paragraph{Estimation} GAM models are estimated using the \texttt{gam()} function from the \texttt{R} package \{\texttt{gam}\}. All variables (predictors and, if applicable, noise variables) are included in the model. Categorical variables are one-hot encoded. For numerical variables, the smoothing parameter is set to 6.

\paragraph{Distribution of Scores} Fig.~\ref{fig:bp_synthetic_gam} shows the distribution of the true probabilities (left column) and the distribution of scores estimated with GAMs, on the test set. The other 4 columns in the figure correspond to a different number of noise variables introduced into the dataset (0, 10, 50, 100). Each row corresponds to a DGP (1, 2, 3, 4).

\begin{figure}[H]
    \centering
    \includegraphics[width = \textwidth]{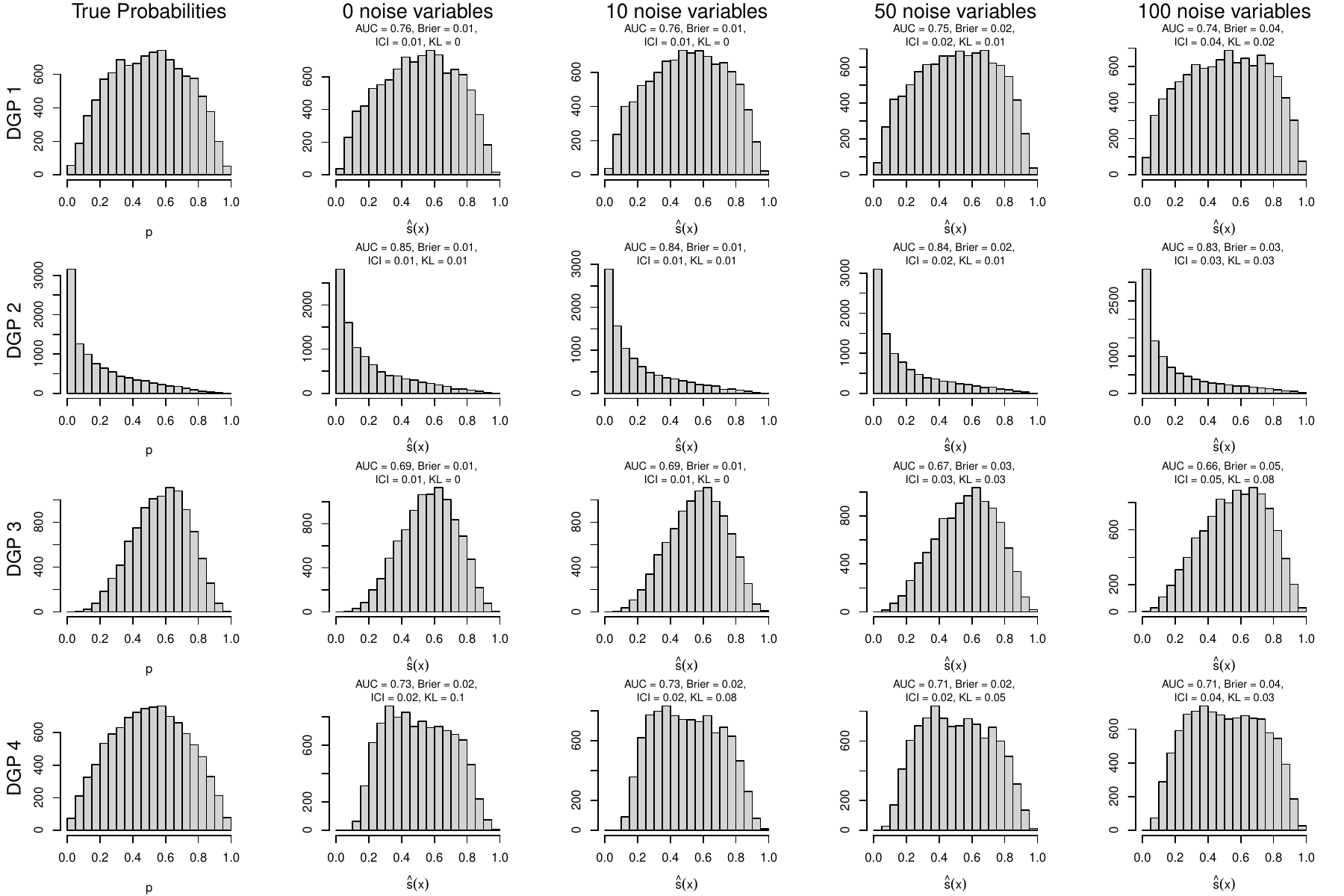}
    \caption{Distribution of true probabilities and estimated scores for \textbf{Generalized Additive Models} under the four DGPs (in rows): single replication across different number of noise variables (in columns).}
    \label{fig:bp_synthetic_gam}
\end{figure}

\subsection{Generalized Additive Model Selection}

\paragraph{Estimation} GAM models are estimated using the \texttt{gamsel()} function from the \texttt{R} package \{\texttt{gamsel}\}. All variables (predictors and, if applicable, noise variables) are included in the model. Categorical variables are one-hot encoded. For each numerical variables, the maximum degrees of freedom is set to 6. However, we use the \texttt{gamsel.cv()} function to determine, using cross-validation, whether each term in the gam should be nonzero, linear, or a non-linear spline.

\paragraph{Distribution of Scores} Fig.~\ref{fig:bp_synthetic_gamsel} shows the distribution of the true probabilities (left column) and the distribution of scores estimated with GAMSELs. The other 4 columns in the figure correspond to a different number of noise variables introduced into the dataset (0, 10, 50, 100). Each row corresponds to a DGP (1, 2, 3, 4).

\begin{figure}[H]
    \centering
    \includegraphics[width = \textwidth]{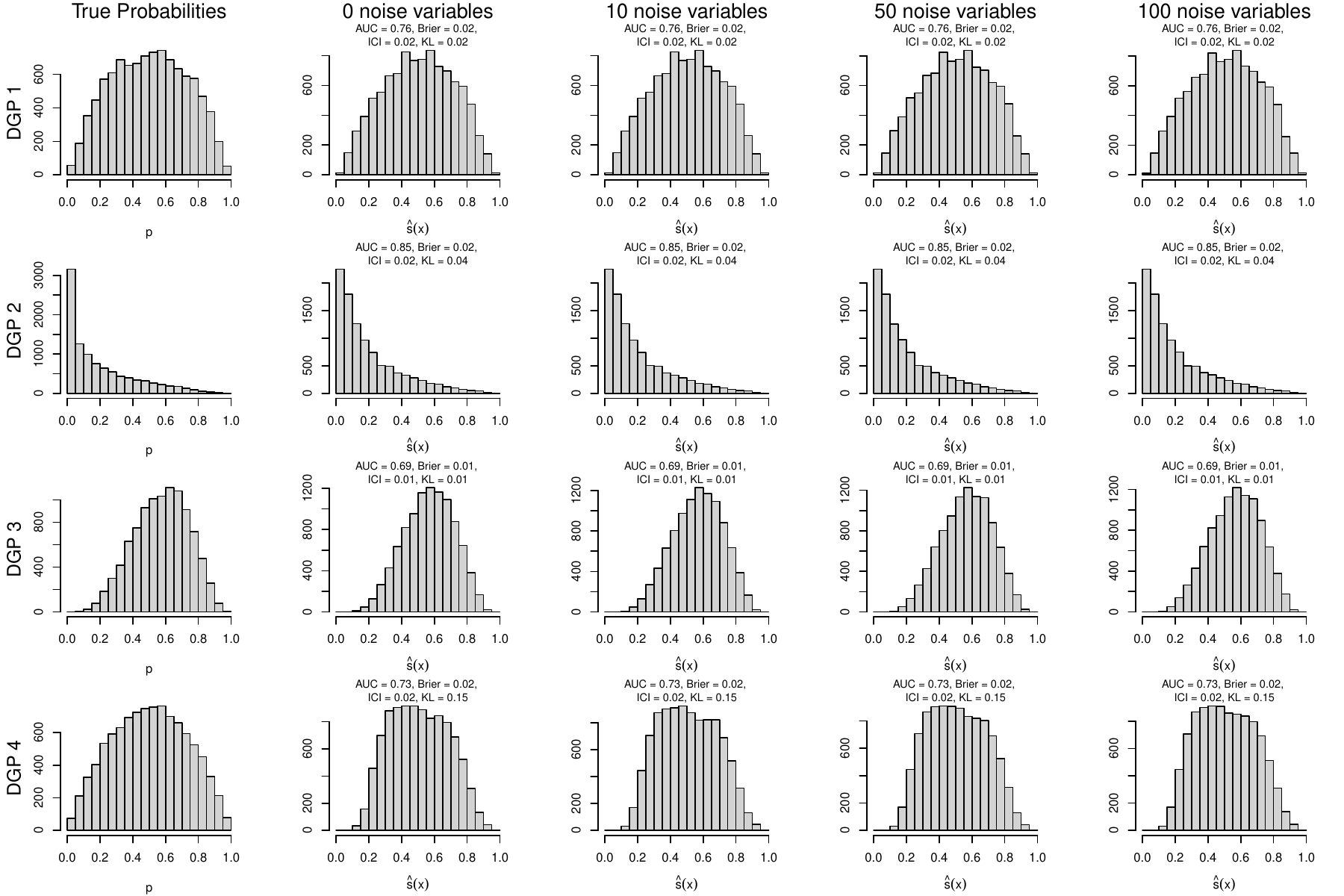}
    \caption{Distribution of true probabilities and estimated scores for \textbf{Generalized Additive Models with Selection} under the four DGPs (in rows): single replication across different number of noise variables (in columns).}
    \label{fig:bp_synthetic_gamsel}
\end{figure}

\subsection{Comparison of Models}

Tables~\ref{tab:synthetic-results-dgp1}, \ref{tab:synthetic-results-dgp2}, \ref{tab:synthetic-results-dgp3}, and \ref{tab:synthetic-results-dgp4} complement Table~\ref{tab:synthetic-results-dgp1} by showing the metrics obtained from 100 replications of simulated probability distributions for DGP 1 to 4, respectively. In addition to the metrics computed for the Random Forest and XGBoost, we included metrics for the GLM, GAM, and GAMSEL models. Since these models do not require hyperparameter selection, the metrics are computed once on the test set. We report the metrics for these models in the section of the table corresponding to hyperparameters selected by optimizing KL divergence (KL*).


{
\setlength{\tabcolsep}{2pt}
\begin{table}[ht!]
    \caption{Comparison of metrics computed on the test set for models selected based on AUC, or KL divergence across 100 replications of \textbf{DGP 1}. Standard errors are provided in parentheses.}
    \label{tab:synthetic-results-dgp1}
    \centering\tiny
    \begin{tabular}{clccccccccccccccc}
\toprule
\multicolumn{2}{c}{ } & \multicolumn{5}{c}{AUC*} & \multicolumn{10}{c}{KL*} \\
\cmidrule(l{3pt}r{3pt}){3-7} \cmidrule(l{3pt}r{3pt}){8-17}
Noise & Model & AUC & Brier & ICI & KL & QR & AUC & Brier & ICI & KL & QR & \(\Delta\)AUC & \(\Delta\)Brier & \(\Delta\)ICI & \(\Delta\)KL & \(\Delta\)QR\\
\midrule
 0 & RF & 0.76 (0.00) & 0.20 (0.00) & 0.01 (0.00) & 0.06 (0.02) & 1.00 (0.03) & 0.75 (0.01) & 0.20 (0.00) & 0.02 (0.01) & 0.03 (0.01) & 1.04 (0.02) & \cellcolor[HTML]{FFCCCC}{\textcolor[HTML]{333333}{-0.01}} & \cellcolor[HTML]{FFD6D6}{\textcolor[HTML]{333333}{0.00}} & \cellcolor[HTML]{FFD6D6}{\textcolor[HTML]{333333}{0.01}} & \cellcolor[HTML]{E9F6E9}{\textcolor[HTML]{1A1A1A}{-0.03}} & 0.04\\
 & XGB & 0.76 (0.00) & 0.20 (0.00) & 0.01 (0.00) & 0.05 (0.03) & 0.96 (0.03) & 0.75 (0.01) & 0.20 (0.00) & 0.01 (0.00) & 0.02 (0.01) & 1.01 (0.02) & \cellcolor[HTML]{FFD6D6}{\textcolor[HTML]{333333}{0.00}} & \cellcolor[HTML]{FFD6D6}{\textcolor[HTML]{333333}{0.00}} & \cellcolor[HTML]{FFD6D6}{\textcolor[HTML]{333333}{0.00}} & \cellcolor[HTML]{E9F6E9}{\textcolor[HTML]{1A1A1A}{-0.03}} & 0.04\\
 & GLM &  &  &  &  &  & 0.76 (0.01) & 0.20 (0.00) & 0.01 (0.00) & 0.00 (0.00) & 1.00 (0.02) & \cellcolor{white}{\textcolor{white}{}} & \cellcolor{white}{\textcolor{white}{}} & \cellcolor{white}{\textcolor{white}{}} & \cellcolor{white}{\textcolor{white}{}} & \\
 & GAM &  &  &  &  &  & 0.76 (0.01) & 0.20 (0.00) & 0.01 (0.00) & 0.00 (0.00) & 1.00 (0.02) & \cellcolor{white}{\textcolor{white}{}} & \cellcolor{white}{\textcolor{white}{}} & \cellcolor{white}{\textcolor{white}{}} & \cellcolor{white}{\textcolor{white}{}} & \\
 & GAMSEL &  &  &  &  &  & 0.76 (0.00) & 0.20 (0.00) & 0.01 (0.00) & 0.01 (0.01) & 0.93 (0.02) & \cellcolor{white}{\textcolor{white}{}} & \cellcolor{white}{\textcolor{white}{}} & \cellcolor{white}{\textcolor{white}{}} & \cellcolor{white}{\textcolor{white}{}} & \\
10 & RF & 0.76 (0.00) & 0.20 (0.00) & 0.01 (0.00) & 0.06 (0.03) & 0.98 (0.03) & 0.75 (0.01) & 0.20 (0.00) & 0.01 (0.00) & 0.01 (0.00) & 1.00 (0.02) & \cellcolor[HTML]{FFC2C2}{\textcolor[HTML]{2B2B2B}{-0.01}} & \cellcolor[HTML]{FFCCCC}{\textcolor[HTML]{333333}{0.00}} & \cellcolor[HTML]{FFD6D6}{\textcolor[HTML]{333333}{0.00}} & \cellcolor[HTML]{E9F6E9}{\textcolor[HTML]{1A1A1A}{-0.05}} & 0.02\\
\cmidrule{2-17}
 & XGB & 0.76 (0.00) & 0.20 (0.00) & 0.01 (0.01) & 0.07 (0.05) & 0.93 (0.04) & 0.75 (0.01) & 0.20 (0.00) & 0.02 (0.00) & 0.01 (0.00) & 1.02 (0.02) & \cellcolor[HTML]{FFCCCC}{\textcolor[HTML]{333333}{-0.01}} & \cellcolor[HTML]{FFD6D6}{\textcolor[HTML]{333333}{0.00}} & \cellcolor[HTML]{FFD6D6}{\textcolor[HTML]{333333}{0.00}} & \cellcolor[HTML]{E9F6E9}{\textcolor[HTML]{1A1A1A}{-0.06}} & 0.08\\
 & GLM &  &  &  &  &  & 0.76 (0.01) & 0.20 (0.00) & 0.01 (0.00) & 0.00 (0.00) & 1.00 (0.02) & \cellcolor{white}{\textcolor{white}{}} & \cellcolor{white}{\textcolor{white}{}} & \cellcolor{white}{\textcolor{white}{}} & \cellcolor{white}{\textcolor{white}{}} & \\
 & GAM &  &  &  &  &  & 0.76 (0.01) & 0.20 (0.00) & 0.01 (0.00) & 0.00 (0.00) & 1.01 (0.02) & \cellcolor{white}{\textcolor{white}{}} & \cellcolor{white}{\textcolor{white}{}} & \cellcolor{white}{\textcolor{white}{}} & \cellcolor{white}{\textcolor{white}{}} & \\
 & GAMSEL &  &  &  &  &  & 0.76 (0.01) & 0.20 (0.00) & 0.01 (0.00) & 0.01 (0.01) & 0.93 (0.02) & \cellcolor{white}{\textcolor{white}{}} & \cellcolor{white}{\textcolor{white}{}} & \cellcolor{white}{\textcolor{white}{}} & \cellcolor{white}{\textcolor{white}{}} & \\
\cmidrule{1-17}
50 & RF & 0.75 (0.00) & 0.20 (0.00) & 0.04 (0.01) & 0.27 (0.05) & 0.76 (0.03) & 0.75 (0.01) & 0.21 (0.00) & 0.03 (0.01) & 0.14 (0.02) & 0.81 (0.02) & \cellcolor[HTML]{FFD6D6}{\textcolor[HTML]{333333}{0.00}} & \cellcolor[HTML]{FFD6D6}{\textcolor[HTML]{333333}{0.00}} & \cellcolor[HTML]{E9F6E9}{\textcolor[HTML]{1A1A1A}{-0.01}} & \cellcolor[HTML]{E9F6E9}{\textcolor[HTML]{1A1A1A}{-0.12}} & 0.04\\
 & XGB & 0.76 (0.00) & 0.20 (0.00) & 0.02 (0.01) & 0.09 (0.05) & 0.91 (0.04) & 0.74 (0.01) & 0.21 (0.00) & 0.02 (0.01) & 0.01 (0.00) & 1.02 (0.02) & \cellcolor[HTML]{FFC2C2}{\textcolor[HTML]{2B2B2B}{-0.01}} & \cellcolor[HTML]{FFCCCC}{\textcolor[HTML]{333333}{0.00}} & \cellcolor[HTML]{FFD6D6}{\textcolor[HTML]{333333}{0.00}} & \cellcolor[HTML]{E9F6E9}{\textcolor[HTML]{1A1A1A}{-0.08}} & 0.10\\
 & GLM &  &  &  &  &  & 0.76 (0.01) & 0.20 (0.00) & 0.01 (0.00) & 0.00 (0.00) & 1.01 (0.02) & \cellcolor{white}{\textcolor{white}{}} & \cellcolor{white}{\textcolor{white}{}} & \cellcolor{white}{\textcolor{white}{}} & \cellcolor{white}{\textcolor{white}{}} & \\
 & GAM &  &  &  &  &  & 0.75 (0.01) & 0.20 (0.00) & 0.02 (0.01) & 0.01 (0.01) & 1.05 (0.02) & \cellcolor{white}{\textcolor{white}{}} & \cellcolor{white}{\textcolor{white}{}} & \cellcolor{white}{\textcolor{white}{}} & \cellcolor{white}{\textcolor{white}{}} & \\
 & GAMSEL &  &  &  &  &  & 0.76 (0.01) & 0.20 (0.00) & 0.01 (0.00) & 0.01 (0.01) & 0.93 (0.02) & \cellcolor{white}{\textcolor{white}{}} & \cellcolor{white}{\textcolor{white}{}} & \cellcolor{white}{\textcolor{white}{}} & \cellcolor{white}{\textcolor{white}{}} & \\
\cmidrule{1-17}
 100 & RF & 0.75 (0.00) & 0.21 (0.00) & 0.07 (0.01) & 0.59 (0.08) & 0.57 (0.03) & 0.74 (0.01) & 0.21 (0.00) & 0.06 (0.01) & 0.45 (0.03) & 0.61 (0.02) & \cellcolor[HTML]{FFD6D6}{\textcolor[HTML]{333333}{0.00}} & \cellcolor[HTML]{E9F6E9}{\textcolor[HTML]{1A1A1A}{0.00}} & \cellcolor[HTML]{E9F6E9}{\textcolor[HTML]{1A1A1A}{-0.01}} & \cellcolor[HTML]{D4F2D4}{\textcolor[HTML]{1A1A1A}{-0.13}} & 0.04\\
 & XGB & 0.76 (0.00) & 0.20 (0.00) & 0.02 (0.01) & 0.09 (0.04) & 0.91 (0.03) & 0.74 (0.01) & 0.21 (0.00) & 0.02 (0.01) & 0.01 (0.00) & 1.02 (0.01) & \cellcolor[HTML]{FFB8B8}{\textcolor[HTML]{2B2B2B}{-0.01}} & \cellcolor[HTML]{FFCCCC}{\textcolor[HTML]{333333}{0.00}} & \cellcolor[HTML]{FFD6D6}{\textcolor[HTML]{333333}{0.00}} & \cellcolor[HTML]{E9F6E9}{\textcolor[HTML]{1A1A1A}{-0.08}} & 0.11\\
 & GLM &  &  &  &  &  & 0.75 (0.01) & 0.20 (0.00) & 0.01 (0.00) & 0.00 (0.00) & 1.02 (0.02) & \cellcolor{white}{\textcolor{white}{}} & \cellcolor{white}{\textcolor{white}{}} & \cellcolor{white}{\textcolor{white}{}} & \cellcolor{white}{\textcolor{white}{}} & \\
 & GAM &  &  &  &  &  & 0.74 (0.01) & 0.21 (0.00) & 0.04 (0.00) & 0.03 (0.01) & 1.10 (0.02) & \cellcolor{white}{\textcolor{white}{}} & \cellcolor{white}{\textcolor{white}{}} & \cellcolor{white}{\textcolor{white}{}} & \cellcolor{white}{\textcolor{white}{}} & \\
 & GAMSEL &  &  &  &  &  & 0.76 (0.01) & 0.20 (0.00) & 0.01 (0.00) & 0.01 (0.01) & 0.93 (0.02) & \cellcolor{white}{\textcolor{white}{}} & \cellcolor{white}{\textcolor{white}{}} & \cellcolor{white}{\textcolor{white}{}} & \cellcolor{white}{\textcolor{white}{}} & \\
\bottomrule
\end{tabular}
\begin{minipage}{\textwidth}
\vspace{1ex}
\scriptsize\underline{Notes:} RF, XGB, GLM, GAM and GAMSEL denote Random Forest, Extreme Gradient Boosting, Generalized Linear Model, Generalized Additive Models, Generalized Additive Models with model selection, respectively. AUC* refers to models chosen based on the maximization or minimization of AUC on the validation set. KL* denote the selection of the best model based on minimizing the Kullback-Leibler divergence between the model's score distributions and the underlying probabilities. Columns $\Delta AUC$, $\Delta \text{Brier}$, $\Delta \text{ICI}$, $\Delta \text{KL}$, and  $\Delta \text{QR}$ display the differences between the AUC/Brier/ICI/KL/QR values from models selected by minimizing the Kullback-Leibler divergence and those maximizing AUC. Negative values indicate a reduction in AUC/Brier/ICI/KL/QR when prioritizing minimization of the Kullback-Leibler divergence over maximization of AUC. QR refers to the ratio of the difference between the 90th and 10th percentiles of the scores to the difference between the 90th and 10th percentiles of the true probabilities.
\end{minipage}
\end{table}
}

{
%
%
\setlength{\tabcolsep}{1.6pt}
\begin{table}[H]
    \caption{Comparison of metrics computed on the test set for models selected based on AUC, or KL divergence across 100 replications of \textbf{DGP 2}. Standard errors are provided in parentheses.}
    \label{tab:synthetic-results-dgp2}
    \centering\tiny
    \begin{tabular}{clccccccccccccccc}
\toprule
\multicolumn{2}{c}{ } & \multicolumn{5}{c}{AUC*} & \multicolumn{10}{c}{KL*} \\
\cmidrule(l{3pt}r{3pt}){3-7} \cmidrule(l{3pt}r{3pt}){8-17}
Noise & Model & AUC & Brier & ICI & KL & QR & AUC & Brier & ICI & KL & QR & \(\Delta\)AUC & \(\Delta\)Brier & \(\Delta\)ICI & \(\Delta\)KL & \(\Delta\)QR\\
\midrule
 0 & Random Forests & 0.84 (0.00) & 0.12 (0.00) & 0.01 (0.00) & 0.04 (0.02) & 1.01 (0.03) & 0.83 (0.01) & 0.12 (0.00) & 0.01 (0.00) & 0.02 (0.01) & 1.03 (0.03) & \cellcolor[HTML]{FFC2C2}{\textcolor[HTML]{2B2B2B}{-0.01}} & \cellcolor[HTML]{FFD6D6}{\textcolor[HTML]{333333}{0.00}} & \cellcolor[HTML]{FFD6D6}{\textcolor[HTML]{333333}{0.00}} & \cellcolor[HTML]{E9F6E9}{\textcolor[HTML]{1A1A1A}{-0.03}} & 0.02\\
 & XGB & 0.84 (0.00) & 0.12 (0.00) & 0.01 (0.00) & 0.03 (0.01) & 0.99 (0.03) & 0.84 (0.01) & 0.12 (0.00) & 0.01 (0.00) & 0.01 (0.01) & 1.01 (0.03) & \cellcolor[HTML]{FFD6D6}{\textcolor[HTML]{333333}{0.00}} & \cellcolor[HTML]{FFD6D6}{\textcolor[HTML]{333333}{0.00}} & \cellcolor[HTML]{E9F6E9}{\textcolor[HTML]{1A1A1A}{0.00}} & \cellcolor[HTML]{E9F6E9}{\textcolor[HTML]{1A1A1A}{-0.01}} & 0.02\\
 & GLM &  &  &  &  &  & 0.84 (0.00) & 0.12 (0.00) & 0.01 (0.00) & 0.02 (0.00) & 0.99 (0.02) & \cellcolor{white}{\textcolor{white}{}} & \cellcolor{white}{\textcolor{white}{}} & \cellcolor{white}{\textcolor{white}{}} & \cellcolor{white}{\textcolor{white}{}} & \\
 & GAM &  &  &  &  &  & 0.84 (0.00) & 0.12 (0.00) & 0.01 (0.00) & 0.01 (0.00) & 1.00 (0.02) & \cellcolor{white}{\textcolor{white}{}} & \cellcolor{white}{\textcolor{white}{}} & \cellcolor{white}{\textcolor{white}{}} & \cellcolor{white}{\textcolor{white}{}} & \\
 & GAMSEL &  &  &  &  &  & 0.84 (0.00) & 0.12 (0.00) & 0.02 (0.00) & 0.04 (0.01) & 0.92 (0.02) & \cellcolor{white}{\textcolor{white}{}} & \cellcolor{white}{\textcolor{white}{}} & \cellcolor{white}{\textcolor{white}{}} & \cellcolor{white}{\textcolor{white}{}} & \\
\cmidrule{1-17}
10 & RF & 0.84 (0.00) & 0.12 (0.00) & 0.01 (0.00) & 0.04 (0.01) & 0.99 (0.03) & 0.83 (0.01) & 0.12 (0.00) & 0.01 (0.00) & 0.01 (0.00) & 1.00 (0.03) & \cellcolor[HTML]{FFC2C2}{\textcolor[HTML]{2B2B2B}{-0.01}} & \cellcolor[HTML]{FFD6D6}{\textcolor[HTML]{333333}{0.00}} & \cellcolor[HTML]{FFD6D6}{\textcolor[HTML]{333333}{0.00}} & \cellcolor[HTML]{E9F6E9}{\textcolor[HTML]{1A1A1A}{-0.03}} & 0.01\\
 & XGB & 0.84 (0.00) & 0.12 (0.00) & 0.01 (0.00) & 0.03 (0.01) & 0.96 (0.03) & 0.83 (0.01) & 0.12 (0.00) & 0.01 (0.00) & 0.01 (0.00) & 1.00 (0.02) & \cellcolor[HTML]{FFCCCC}{\textcolor[HTML]{333333}{-0.01}} & \cellcolor[HTML]{FFD6D6}{\textcolor[HTML]{333333}{0.00}} & \cellcolor[HTML]{E9F6E9}{\textcolor[HTML]{1A1A1A}{0.00}} & \cellcolor[HTML]{E9F6E9}{\textcolor[HTML]{1A1A1A}{-0.02}} & 0.04\\
 & GLM &  &  &  &  &  & 0.84 (0.01) & 0.12 (0.00) & 0.01 (0.00) & 0.02 (0.00) & 0.99 (0.02) & \cellcolor{white}{\textcolor{white}{}} & \cellcolor{white}{\textcolor{white}{}} & \cellcolor{white}{\textcolor{white}{}} & \cellcolor{white}{\textcolor{white}{}} & \\
 & GAM &  &  &  &  &  & 0.84 (0.01) & 0.12 (0.00) & 0.01 (0.00) & 0.01 (0.00) & 1.00 (0.02) & \cellcolor{white}{\textcolor{white}{}} & \cellcolor{white}{\textcolor{white}{}} & \cellcolor{white}{\textcolor{white}{}} & \cellcolor{white}{\textcolor{white}{}} & \\
 & GAMSEL &  &  &  &  &  & 0.84 (0.00) & 0.12 (0.00) & 0.02 (0.00) & 0.04 (0.01) & 0.92 (0.02) & \cellcolor{white}{\textcolor{white}{}} & \cellcolor{white}{\textcolor{white}{}} & \cellcolor{white}{\textcolor{white}{}} & \cellcolor{white}{\textcolor{white}{}} & \\
\cmidrule{1-17}
 50 & RF & 0.83 (0.00) & 0.12 (0.00) & 0.04 (0.01) & 0.37 (0.11) & 0.78 (0.04) & 0.83 (0.01) & 0.12 (0.00) & 0.03 (0.00) & 0.12 (0.02) & 0.88 (0.02) & \cellcolor[HTML]{FFCCCC}{\textcolor[HTML]{333333}{-0.01}} & \cellcolor[HTML]{E9F6E9}{\textcolor[HTML]{1A1A1A}{0.00}} & \cellcolor[HTML]{E9F6E9}{\textcolor[HTML]{1A1A1A}{-0.01}} & \cellcolor[HTML]{BFEFBF}{\textcolor[HTML]{1A1A1A}{-0.26}} & 0.10\\
 & XGB & 0.84 (0.00) & 0.12 (0.00) & 0.01 (0.00) & 0.04 (0.06) & 0.95 (0.03) & 0.83 (0.01) & 0.12 (0.00) & 0.01 (0.00) & 0.01 (0.00) & 1.00 (0.02) & \cellcolor[HTML]{FFCCCC}{\textcolor[HTML]{333333}{-0.01}} & \cellcolor[HTML]{FFD6D6}{\textcolor[HTML]{333333}{0.00}} & \cellcolor[HTML]{E9F6E9}{\textcolor[HTML]{1A1A1A}{0.00}} & \cellcolor[HTML]{E9F6E9}{\textcolor[HTML]{1A1A1A}{-0.04}} & 0.05\\
 & GLM &  &  &  &  &  & 0.84 (0.00) & 0.12 (0.00) & 0.01 (0.00) & 0.03 (0.00) & 1.00 (0.02) & \cellcolor{white}{\textcolor{white}{}} & \cellcolor{white}{\textcolor{white}{}} & \cellcolor{white}{\textcolor{white}{}} & \cellcolor{white}{\textcolor{white}{}} & \\
 & GAM &  &  &  &  &  & 0.83 (0.00) & 0.12 (0.00) & 0.02 (0.00) & 0.02 (0.01) & 1.04 (0.02) & \cellcolor{white}{\textcolor{white}{}} & \cellcolor{white}{\textcolor{white}{}} & \cellcolor{white}{\textcolor{white}{}} & \cellcolor{white}{\textcolor{white}{}} & \\
 & GAMSEL &  &  &  &  &  & 0.84 (0.00) & 0.12 (0.00) & 0.02 (0.00) & 0.04 (0.01) & 0.92 (0.02) & \cellcolor{white}{\textcolor{white}{}} & \cellcolor{white}{\textcolor{white}{}} & \cellcolor{white}{\textcolor{white}{}} & \cellcolor{white}{\textcolor{white}{}} & \\
\cmidrule{1-17}
 100 & RF & 0.82 (0.01) & 0.13 (0.00) & 0.07 (0.01) & 0.84 (0.11) & 0.60 (0.05) & 0.82 (0.01) & 0.13 (0.00) & 0.05 (0.00) & 0.49 (0.04) & 0.72 (0.02) & \cellcolor[HTML]{FFCCCC}{\textcolor[HTML]{333333}{0.00}} & \cellcolor[HTML]{D4F2D4}{\textcolor[HTML]{1A1A1A}{0.00}} & \cellcolor[HTML]{D4F2D4}{\textcolor[HTML]{1A1A1A}{-0.02}} & \cellcolor[HTML]{BFEFBF}{\textcolor[HTML]{1A1A1A}{-0.35}} & 0.12\\
 & XGB & 0.84 (0.00) & 0.12 (0.00) & 0.01 (0.00) & 0.04 (0.02) & 0.94 (0.03) & 0.83 (0.01) & 0.12 (0.00) & 0.01 (0.00) & 0.01 (0.00) & 0.99 (0.02) & \cellcolor[HTML]{FFC2C2}{\textcolor[HTML]{2B2B2B}{-0.01}} & \cellcolor[HTML]{FFCCCC}{\textcolor[HTML]{333333}{0.00}} & \cellcolor[HTML]{E9F6E9}{\textcolor[HTML]{1A1A1A}{0.00}} & \cellcolor[HTML]{E9F6E9}{\textcolor[HTML]{1A1A1A}{-0.03}} & 0.06\\
 & GLM &  &  &  &  &  & 0.84 (0.01) & 0.12 (0.00) & 0.01 (0.00) & 0.03 (0.00) & 1.01 (0.02) & \cellcolor{white}{\textcolor{white}{}} & \cellcolor{white}{\textcolor{white}{}} & \cellcolor{white}{\textcolor{white}{}} & \cellcolor{white}{\textcolor{white}{}} & \\
 & GAM &  &  &  &  &  & 0.82 (0.01) & 0.13 (0.00) & 0.03 (0.00) & 0.04 (0.01) & 1.08 (0.02) & \cellcolor{white}{\textcolor{white}{}} & \cellcolor{white}{\textcolor{white}{}} & \cellcolor{white}{\textcolor{white}{}} & \cellcolor{white}{\textcolor{white}{}} & \\
 & GAMSEL &  &  &  &  &  & 0.84 (0.00) & 0.12 (0.00) & 0.02 (0.00) & 0.04 (0.01) & 0.92 (0.02) & \cellcolor{white}{\textcolor{white}{}} & \cellcolor{white}{\textcolor{white}{}} & \cellcolor{white}{\textcolor{white}{}} & \cellcolor{white}{\textcolor{white}{}} & \\
\bottomrule
\end{tabular}
\begin{minipage}{\textwidth}
\vspace{1ex}
\scriptsize\underline{Notes:} RF, XGB, GLM, GAM and GAMSEL denote Random Forest, Extreme Gradient Boosting, Generalized Linear Model, Generalized Additive Models, Generalized Additive Models with model selection, respectively. AUC* refers to models chosen based on the maximization or minimization of AUC on the validation set. KL* denote the selection of the best model based on minimizing the Kullback-Leibler divergence between the model's score distributions and the underlying probabilities. Columns $\Delta AUC$, $\Delta \text{Brier}$, $\Delta \text{ICI}$, $\Delta \text{KL}$, and  $\Delta \text{QR}$ display the differences between the AUC/Brier/ICI/KL/QR values from models selected by minimizing the Kullback-Leibler divergence and those maximizing AUC. Negative values indicate a reduction in AUC/Brier/ICI/KL/QR when prioritizing minimization of the Kullback-Leibler divergence over maximization of AUC. QR refers to the ratio of the difference between the 90th and 10th percentiles of the scores to the difference between the 90th and 10th percentiles of the true probabilities.
\end{minipage}
\end{table}
}

{
%
%
\setlength{\tabcolsep}{2pt}
\begin{table}[H]
    \caption{Comparison of metrics computed on the test set for models selected based on AUC, or KL divergence across 100 replications of \textbf{DGP 3}. Standard errors are provided in parentheses.}
    \label{tab:synthetic-results-dgp3}
    \centering\tiny
    \begin{tabular}{clccccccccccccccc}
\toprule
\multicolumn{2}{c}{ } & \multicolumn{5}{c}{AUC*} & \multicolumn{10}{c}{KL*} \\
\cmidrule(l{3pt}r{3pt}){3-7} \cmidrule(l{3pt}r{3pt}){8-17}
Noise & Model & AUC & Brier & ICI & KL & QR & AUC & Brier & ICI & KL & QR & \(\Delta\)AUC & \(\Delta\)Brier & \(\Delta\)ICI & \(\Delta\)KL & \(\Delta\)QR\\
\midrule
0 & RF & 0.68 (0.01) & 0.22 (0.00) & 0.02 (0.01) & 0.12 (0.08) & 0.81 (0.10) & 0.67 (0.01) & 0.22 (0.00) & 0.02 (0.01) & 0.01 (0.00) & 1.02 (0.02) & \cellcolor[HTML]{FFC2C2}{\textcolor[HTML]{2B2B2B}{-0.01}} & \cellcolor[HTML]{FFD6D6}{\textcolor[HTML]{333333}{0.00}} & \cellcolor[HTML]{E9F6E9}{\textcolor[HTML]{1A1A1A}{-0.01}} & \cellcolor[HTML]{E9F6E9}{\textcolor[HTML]{1A1A1A}{-0.11}} & 0.21\\
 & XGB & 0.68 (0.01) & 0.22 (0.00) & 0.01 (0.00) & 0.01 (0.01) & 0.96 (0.04) & 0.68 (0.01) & 0.22 (0.00) & 0.01 (0.00) & 0.01 (0.00) & 1.00 (0.02) & \cellcolor[HTML]{FFD6D6}{\textcolor[HTML]{333333}{0.00}} & \cellcolor[HTML]{FFD6D6}{\textcolor[HTML]{333333}{0.00}} & \cellcolor[HTML]{FFD6D6}{\textcolor[HTML]{333333}{0.00}} & \cellcolor[HTML]{E9F6E9}{\textcolor[HTML]{1A1A1A}{-0.01}} & 0.04\\
 & GLM &  &  &  &  &  & 0.69 (0.00) & 0.22 (0.00) & 0.01 (0.00) & 0.00 (0.00) & 1.01 (0.03) & \cellcolor{white}{\textcolor{white}{}} & \cellcolor{white}{\textcolor{white}{}} & \cellcolor{white}{\textcolor{white}{}} & \cellcolor{white}{\textcolor{white}{}} & \\
 & GAM &  &  &  &  &  & 0.69 (0.01) & 0.22 (0.00) & 0.01 (0.00) & 0.01 (0.00) & 1.01 (0.03) & \cellcolor{white}{\textcolor{white}{}} & \cellcolor{white}{\textcolor{white}{}} & \cellcolor{white}{\textcolor{white}{}} & \cellcolor{white}{\textcolor{white}{}} & \\
 & GAMSEL &  &  &  &  &  & 0.69 (0.00) & 0.22 (0.00) & 0.01 (0.00) & 0.01 (0.01) & 0.93 (0.03) & \cellcolor{white}{\textcolor{white}{}} & \cellcolor{white}{\textcolor{white}{}} & \cellcolor{white}{\textcolor{white}{}} & \cellcolor{white}{\textcolor{white}{}} & \\
\cmidrule{1-17}
 10 & RF & 0.68 (0.01) & 0.22 (0.00) & 0.03 (0.01) & 0.16 (0.09) & 0.77 (0.10) & 0.67 (0.01) & 0.22 (0.00) & 0.01 (0.00) & 0.01 (0.01) & 0.96 (0.03) & \cellcolor[HTML]{FFC2C2}{\textcolor[HTML]{2B2B2B}{-0.01}} & \cellcolor[HTML]{FFD6D6}{\textcolor[HTML]{333333}{0.00}} & \cellcolor[HTML]{D4F2D4}{\textcolor[HTML]{1A1A1A}{-0.02}} & \cellcolor[HTML]{D4F2D4}{\textcolor[HTML]{1A1A1A}{-0.15}} & 0.19\\
 & XGB & 0.68 (0.01) & 0.22 (0.00) & 0.01 (0.00) & 0.03 (0.02) & 0.91 (0.05) & 0.68 (0.01) & 0.22 (0.00) & 0.01 (0.01) & 0.00 (0.00) & 1.00 (0.02) & \cellcolor[HTML]{FFCCCC}{\textcolor[HTML]{333333}{0.00}} & \cellcolor[HTML]{FFD6D6}{\textcolor[HTML]{333333}{0.00}} & \cellcolor[HTML]{FFD6D6}{\textcolor[HTML]{333333}{0.00}} & \cellcolor[HTML]{E9F6E9}{\textcolor[HTML]{1A1A1A}{-0.02}} & 0.09\\
 & GLM &  &  &  &  &  & 0.69 (0.00) & 0.22 (0.00) & 0.01 (0.00) & 0.00 (0.00) & 1.01 (0.03) & \cellcolor{white}{\textcolor{white}{}} & \cellcolor{white}{\textcolor{white}{}} & \cellcolor{white}{\textcolor{white}{}} & \cellcolor{white}{\textcolor{white}{}} & \\
 & GAM &  &  &  &  &  & 0.69 (0.01) & 0.22 (0.00) & 0.01 (0.01) & 0.01 (0.01) & 1.03 (0.03) & \cellcolor{white}{\textcolor{white}{}} & \cellcolor{white}{\textcolor{white}{}} & \cellcolor{white}{\textcolor{white}{}} & \cellcolor{white}{\textcolor{white}{}} & \\
 & GAMSEL &  &  &  &  &  & 0.69 (0.00) & 0.22 (0.00) & 0.01 (0.00) & 0.01 (0.01) & 0.93 (0.03) & \cellcolor{white}{\textcolor{white}{}} & \cellcolor{white}{\textcolor{white}{}} & \cellcolor{white}{\textcolor{white}{}} & \cellcolor{white}{\textcolor{white}{}} & \\
\cmidrule{1-17}
 50 & RF & 0.68 (0.01) & 0.22 (0.00) & 0.04 (0.01) & 0.33 (0.10) & 0.63 (0.06) & 0.67 (0.01) & 0.22 (0.00) & 0.03 (0.00) & 0.19 (0.03) & 0.71 (0.02) & \cellcolor[HTML]{FFCCCC}{\textcolor[HTML]{333333}{0.00}} & \cellcolor[HTML]{E9F6E9}{\textcolor[HTML]{1A1A1A}{0.00}} & \cellcolor[HTML]{D4F2D4}{\textcolor[HTML]{1A1A1A}{-0.01}} & \cellcolor[HTML]{D4F2D4}{\textcolor[HTML]{1A1A1A}{-0.14}} & 0.07\\
 & XGB & 0.68 (0.01) & 0.22 (0.00) & 0.01 (0.01) & 0.05 (0.03) & 0.87 (0.05) & 0.67 (0.01) & 0.22 (0.00) & 0.02 (0.00) & 0.00 (0.00) & 1.01 (0.02) & \cellcolor[HTML]{FFC2C2}{\textcolor[HTML]{2B2B2B}{-0.01}} & \cellcolor[HTML]{FFD6D6}{\textcolor[HTML]{333333}{0.00}} & \cellcolor[HTML]{FFD6D6}{\textcolor[HTML]{333333}{0.00}} & \cellcolor[HTML]{E9F6E9}{\textcolor[HTML]{1A1A1A}{-0.05}} & 0.14\\
 & GLM &  &  &  &  &  & 0.69 (0.01) & 0.22 (0.00) & 0.01 (0.01) & 0.01 (0.00) & 1.03 (0.03) & \cellcolor{white}{\textcolor{white}{}} & \cellcolor{white}{\textcolor{white}{}} & \cellcolor{white}{\textcolor{white}{}} & \cellcolor{white}{\textcolor{white}{}} & \\
 & GAM &  &  &  &  &  & 0.67 (0.01) & 0.22 (0.00) & 0.03 (0.01) & 0.03 (0.01) & 1.11 (0.03) & \cellcolor{white}{\textcolor{white}{}} & \cellcolor{white}{\textcolor{white}{}} & \cellcolor{white}{\textcolor{white}{}} & \cellcolor{white}{\textcolor{white}{}} & \\
 & GAMSEL &  &  &  &  &  & 0.69 (0.00) & 0.22 (0.00) & 0.01 (0.00) & 0.01 (0.01) & 0.93 (0.03) & \cellcolor{white}{\textcolor{white}{}} & \cellcolor{white}{\textcolor{white}{}} & \cellcolor{white}{\textcolor{white}{}} & \cellcolor{white}{\textcolor{white}{}} & \\
\cmidrule{1-17}
 100 & RF & 0.67 (0.01) & 0.23 (0.00) & 0.06 (0.01) & 0.68 (0.12) & 0.46 (0.04) & 0.67 (0.01) & 0.23 (0.00) & 0.05 (0.00) & 0.46 (0.04) & 0.54 (0.02) & \cellcolor[HTML]{FFCCCC}{\textcolor[HTML]{333333}{-0.01}} & \cellcolor[HTML]{E9F6E9}{\textcolor[HTML]{1A1A1A}{0.00}} & \cellcolor[HTML]{D4F2D4}{\textcolor[HTML]{1A1A1A}{-0.02}} & \cellcolor[HTML]{D4F2D4}{\textcolor[HTML]{1A1A1A}{-0.22}} & 0.08\\
 & XGB & 0.68 (0.01) & 0.22 (0.00) & 0.01 (0.01) & 0.06 (0.03) & 0.85 (0.04) & 0.67 (0.01) & 0.22 (0.00) & 0.02 (0.01) & 0.00 (0.00) & 1.01 (0.01) & \cellcolor[HTML]{FFC2C2}{\textcolor[HTML]{2B2B2B}{-0.01}} & \cellcolor[HTML]{FFCCCC}{\textcolor[HTML]{333333}{0.00}} & \cellcolor[HTML]{FFD6D6}{\textcolor[HTML]{333333}{0.00}} & \cellcolor[HTML]{E9F6E9}{\textcolor[HTML]{1A1A1A}{-0.06}} & 0.16\\
 & GLM &  &  &  &  &  & 0.68 (0.01) & 0.22 (0.00) & 0.02 (0.01) & 0.01 (0.01) & 1.05 (0.03) & \cellcolor{white}{\textcolor{white}{}} & \cellcolor{white}{\textcolor{white}{}} & \cellcolor{white}{\textcolor{white}{}} & \cellcolor{white}{\textcolor{white}{}} & \\
 & GAM &  &  &  &  &  & 0.66 (0.01) & 0.23 (0.00) & 0.05 (0.01) & 0.09 (0.02) & 1.20 (0.03) & \cellcolor{white}{\textcolor{white}{}} & \cellcolor{white}{\textcolor{white}{}} & \cellcolor{white}{\textcolor{white}{}} & \cellcolor{white}{\textcolor{white}{}} & \\
 & GAMSEL &  &  &  &  &  & 0.69 (0.00) & 0.22 (0.00) & 0.01 (0.00) & 0.01 (0.01) & 0.93 (0.03) & \cellcolor{white}{\textcolor{white}{}} & \cellcolor{white}{\textcolor{white}{}} & \cellcolor{white}{\textcolor{white}{}} & \cellcolor{white}{\textcolor{white}{}} & \\
\bottomrule
\end{tabular}
\begin{minipage}{\textwidth}
\vspace{1ex}
\scriptsize\underline{Notes:} RF, XGB, GLM, GAM and GAMSEL denote Random Forest, Extreme Gradient Boosting, Generalized Linear Model, Generalized Additive Models, Generalized Additive Models with model selection, respectively. AUC* refers to models chosen based on the maximization or minimization of AUC on the validation set. KL* denote the selection of the best model based on minimizing the Kullback-Leibler divergence between the model's score distributions and the underlying probabilities. Columns $\Delta AUC$, $\Delta \text{Brier}$, $\Delta \text{ICI}$, $\Delta \text{KL}$, and  $\Delta \text{QR}$ display the differences between the AUC/Brier/ICI/KL/QR values from models selected by minimizing the Kullback-Leibler divergence and those maximizing AUC. Negative values indicate a reduction in AUC/Brier/ICI/KL/QR when prioritizing minimization of the Kullback-Leibler divergence over maximization of AUC. QR refers to the ratio of the difference between the 90th and 10th percentiles of the scores to the difference between the 90th and 10th percentiles of the true probabilities.
\end{minipage}
\end{table}
}

{
%
%
\setlength{\tabcolsep}{2pt}
\begin{table}[H]
    \caption{Comparison of metrics computed on the test set for models selected based on AUC, or KL divergence across 100 replications of \textbf{DGP 4}. Standard errors are provided in parentheses.}
    \label{tab:synthetic-results-dgp4}
    \centering\tiny
    \begin{tabular}{clccccccccccccccc}
\toprule
\multicolumn{2}{c}{ } & \multicolumn{5}{c}{AUC*} & \multicolumn{10}{c}{KL*} \\
\cmidrule(l{3pt}r{3pt}){3-7} \cmidrule(l{3pt}r{3pt}){8-17}
Noise & Model & AUC & Brier & ICI & KL & QR & AUC & Brier & ICI & KL & QR & \(\Delta\)AUC & \(\Delta\)Brier & \(\Delta\)ICI & \(\Delta\)KL & \(\Delta\)QR\\
\midrule
0 & RF & 0.74 (0.01) & 0.21 (0.00) & 0.01 (0.00) & 0.04 (0.01) & 0.95 (0.02) & 0.74 (0.01) & 0.21 (0.00) & 0.02 (0.01) & 0.01 (0.00) & 1.03 (0.02) & \cellcolor[HTML]{FFCCCC}{\textcolor[HTML]{333333}{-0.01}} & \cellcolor[HTML]{FFD6D6}{\textcolor[HTML]{333333}{0.00}} & \cellcolor[HTML]{FFD6D6}{\textcolor[HTML]{333333}{0.01}} & \cellcolor[HTML]{E9F6E9}{\textcolor[HTML]{1A1A1A}{-0.03}} & 0.08\\
 & XGB & 0.75 (0.01) & 0.20 (0.00) & 0.01 (0.00) & 0.04 (0.01) & 0.96 (0.03) & 0.74 (0.01) & 0.21 (0.00) & 0.02 (0.01) & 0.01 (0.00) & 1.03 (0.02) & \cellcolor[HTML]{FFD6D6}{\textcolor[HTML]{333333}{0.00}} & \cellcolor[HTML]{FFD6D6}{\textcolor[HTML]{333333}{0.00}} & \cellcolor[HTML]{FFD6D6}{\textcolor[HTML]{333333}{0.01}} & \cellcolor[HTML]{E9F6E9}{\textcolor[HTML]{1A1A1A}{-0.03}} & 0.07\\
 & GLM &  &  &  &  &  & 0.68 (0.01) & 0.23 (0.00) & 0.01 (0.00) & 0.27 (0.03) & 0.67 (0.02) & \cellcolor{white}{\textcolor{white}{}} & \cellcolor{white}{\textcolor{white}{}} & \cellcolor{white}{\textcolor{white}{}} & \cellcolor{white}{\textcolor{white}{}} & \\
 & GAM &  &  &  &  &  & 0.73 (0.01) & 0.21 (0.00) & 0.01 (0.00) & 0.08 (0.01) & 0.88 (0.02) & \cellcolor{white}{\textcolor{white}{}} & \cellcolor{white}{\textcolor{white}{}} & \cellcolor{white}{\textcolor{white}{}} & \cellcolor{white}{\textcolor{white}{}} & \\
 & GAMSEL &  &  &  &  &  & 0.73 (0.01) & 0.21 (0.00) & 0.02 (0.01) & 0.14 (0.02) & 0.79 (0.02) & \cellcolor{white}{\textcolor{white}{}} & \cellcolor{white}{\textcolor{white}{}} & \cellcolor{white}{\textcolor{white}{}} & \cellcolor{white}{\textcolor{white}{}} & \\
\cmidrule{1-17}
10 & RF & 0.74 (0.01) & 0.21 (0.00) & 0.01 (0.00) & 0.06 (0.02) & 0.93 (0.03) & 0.73 (0.01) & 0.21 (0.00) & 0.02 (0.00) & 0.01 (0.01) & 0.96 (0.02) & \cellcolor[HTML]{FFC2C2}{\textcolor[HTML]{2B2B2B}{-0.01}} & \cellcolor[HTML]{FFCCCC}{\textcolor[HTML]{333333}{0.00}} & \cellcolor[HTML]{FFD6D6}{\textcolor[HTML]{333333}{0.00}} & \cellcolor[HTML]{E9F6E9}{\textcolor[HTML]{1A1A1A}{-0.04}} & 0.03\\
 & XGB & 0.74 (0.01) & 0.21 (0.00) & 0.01 (0.00) & 0.08 (0.02) & 0.89 (0.03) & 0.73 (0.01) & 0.21 (0.00) & 0.03 (0.01) & 0.02 (0.00) & 1.04 (0.02) & \cellcolor[HTML]{FFC2C2}{\textcolor[HTML]{2B2B2B}{-0.01}} & \cellcolor[HTML]{FFCCCC}{\textcolor[HTML]{333333}{0.00}} & \cellcolor[HTML]{FFCCCC}{\textcolor[HTML]{333333}{0.02}} & \cellcolor[HTML]{E9F6E9}{\textcolor[HTML]{1A1A1A}{-0.07}} & 0.15\\
 & GLM &  &  &  &  &  & 0.68 (0.01) & 0.23 (0.00) & 0.01 (0.00) & 0.26 (0.03) & 0.67 (0.02) & \cellcolor{white}{\textcolor{white}{}} & \cellcolor{white}{\textcolor{white}{}} & \cellcolor{white}{\textcolor{white}{}} & \cellcolor{white}{\textcolor{white}{}} & \\
 & GAM &  &  &  &  &  & 0.73 (0.01) & 0.21 (0.00) & 0.01 (0.00) & 0.07 (0.01) & 0.90 (0.02) & \cellcolor{white}{\textcolor{white}{}} & \cellcolor{white}{\textcolor{white}{}} & \cellcolor{white}{\textcolor{white}{}} & \cellcolor{white}{\textcolor{white}{}} & \\
 & GAMSEL &  &  &  &  &  & 0.73 (0.01) & 0.21 (0.00) & 0.02 (0.01) & 0.14 (0.02) & 0.79 (0.02) & \cellcolor{white}{\textcolor{white}{}} & \cellcolor{white}{\textcolor{white}{}} & \cellcolor{white}{\textcolor{white}{}} & \cellcolor{white}{\textcolor{white}{}} & \\
\cmidrule{1-17}
 50 & RF & 0.73 (0.01) & 0.21 (0.00) & 0.04 (0.01) & 0.34 (0.04) & 0.66 (0.02) & 0.73 (0.01) & 0.21 (0.00) & 0.04 (0.01) & 0.30 (0.03) & 0.67 (0.02) & \cellcolor[HTML]{FFD6D6}{\textcolor[HTML]{333333}{0.00}} & \cellcolor[HTML]{FFD6D6}{\textcolor[HTML]{333333}{0.00}} & \cellcolor[HTML]{E9F6E9}{\textcolor[HTML]{1A1A1A}{0.00}} & \cellcolor[HTML]{E9F6E9}{\textcolor[HTML]{1A1A1A}{-0.04}} & 0.01\\
 & XGB & 0.74 (0.01) & 0.21 (0.00) & 0.02 (0.01) & 0.12 (0.03) & 0.84 (0.03) & 0.72 (0.01) & 0.21 (0.00) & 0.03 (0.01) & 0.02 (0.00) & 1.03 (0.02) & \cellcolor[HTML]{FFADAD}{\textcolor[HTML]{232323}{-0.02}} & \cellcolor[HTML]{FFC2C2}{\textcolor[HTML]{2B2B2B}{0.01}} & \cellcolor[HTML]{FFCCCC}{\textcolor[HTML]{333333}{0.02}} & \cellcolor[HTML]{E9F6E9}{\textcolor[HTML]{1A1A1A}{-0.11}} & 0.19\\
 & GLM &  &  &  &  &  & 0.67 (0.01) & 0.23 (0.00) & 0.01 (0.00) & 0.24 (0.03) & 0.69 (0.02) & \cellcolor{white}{\textcolor{white}{}} & \cellcolor{white}{\textcolor{white}{}} & \cellcolor{white}{\textcolor{white}{}} & \cellcolor{white}{\textcolor{white}{}} & \\
 & GAM &  &  &  &  &  & 0.72 (0.01) & 0.22 (0.00) & 0.02 (0.00) & 0.04 (0.01) & 0.94 (0.02) & \cellcolor{white}{\textcolor{white}{}} & \cellcolor{white}{\textcolor{white}{}} & \cellcolor{white}{\textcolor{white}{}} & \cellcolor{white}{\textcolor{white}{}} & \\
 & GAMSEL &  &  &  &  &  & 0.73 (0.01) & 0.21 (0.00) & 0.02 (0.01) & 0.14 (0.02) & 0.79 (0.02) & \cellcolor{white}{\textcolor{white}{}} & \cellcolor{white}{\textcolor{white}{}} & \cellcolor{white}{\textcolor{white}{}} & \cellcolor{white}{\textcolor{white}{}} & \\
\cmidrule{1-17}
100 & RF & 0.72 (0.01) & 0.22 (0.00) & 0.07 (0.01) & 0.72 (0.05) & 0.48 (0.02) & 0.72 (0.01) & 0.22 (0.00) & 0.07 (0.01) & 0.65 (0.04) & 0.49 (0.02) & \cellcolor[HTML]{FFD6D6}{\textcolor[HTML]{333333}{0.00}} & \cellcolor[HTML]{E9F6E9}{\textcolor[HTML]{1A1A1A}{0.00}} & \cellcolor[HTML]{E9F6E9}{\textcolor[HTML]{1A1A1A}{0.00}} & \cellcolor[HTML]{E9F6E9}{\textcolor[HTML]{1A1A1A}{-0.06}} & 0.01\\
 & XGB & 0.74 (0.01) & 0.21 (0.00) & 0.02 (0.01) & 0.14 (0.04) & 0.82 (0.04) & 0.72 (0.01) & 0.22 (0.00) & 0.04 (0.01) & 0.02 (0.00) & 1.03 (0.01) & \cellcolor[HTML]{FFA3A3}{\textcolor[HTML]{1F1F1F}{-0.02}} & \cellcolor[HTML]{FFB8B8}{\textcolor[HTML]{2B2B2B}{0.01}} & \cellcolor[HTML]{FFCCCC}{\textcolor[HTML]{333333}{0.02}} & \cellcolor[HTML]{E9F6E9}{\textcolor[HTML]{1A1A1A}{-0.12}} & 0.21\\
 & GLM &  &  &  &  &  & 0.67 (0.01) & 0.23 (0.00) & 0.01 (0.00) & 0.22 (0.03) & 0.71 (0.02) & \cellcolor{white}{\textcolor{white}{}} & \cellcolor{white}{\textcolor{white}{}} & \cellcolor{white}{\textcolor{white}{}} & \cellcolor{white}{\textcolor{white}{}} & \\
 & GAM &  &  &  &  &  & 0.71 (0.01) & 0.22 (0.00) & 0.04 (0.01) & 0.02 (0.00) & 1.00 (0.02) & \cellcolor{white}{\textcolor{white}{}} & \cellcolor{white}{\textcolor{white}{}} & \cellcolor{white}{\textcolor{white}{}} & \cellcolor{white}{\textcolor{white}{}} & \\
 & GAMSEL &  &  &  &  &  & 0.73 (0.01) & 0.21 (0.00) & 0.02 (0.01) & 0.14 (0.02) & 0.79 (0.02) & \cellcolor{white}{\textcolor{white}{}} & \cellcolor{white}{\textcolor{white}{}} & \cellcolor{white}{\textcolor{white}{}} & \cellcolor{white}{\textcolor{white}{}} & \\
\bottomrule
\end{tabular}
\begin{minipage}{\textwidth}
\vspace{1ex}
\scriptsize\underline{Notes:} RF, \scriptsize\underline{Notes:} RF, XGB, GLM, GAM and GAMSEL denote Random Forest, Extreme Gradient Boosting, Generalized Linear Model, Generalized Additive Models, Generalized Additive Models with model selection, respectively. AUC* refers to models chosen based on the maximization or minimization of AUC on the validation set. KL* denote the selection of the best model based on minimizing the Kullback-Leibler divergence between the model's score distributions and the underlying probabilities. Columns $\Delta AUC$, $\Delta \text{Brier}$, $\Delta \text{ICI}$, $\Delta \text{KL}$, and  $\Delta \text{QR}$ display the differences between the AUC/Brier/ICI/KL/QR values from models selected by minimizing the Kullback-Leibler divergence and those maximizing AUC. Negative values indicate a reduction in AUC/Brier/ICI/KL/QR when prioritizing minimization of the Kullback-Leibler divergence over maximization of AUC. QR refers to the ratio of the difference between the 90th and 10th percentiles of the scores to the difference between the 90th and 10th percentiles of the true probabilities.
\end{minipage}
\end{table}
}


\section{Real-World Data}

Ten datasets from the UCI Machine Learning Repository are used in this study. The datasets are split into three parts: a training sample (64\% of the observations), a validation sample (16\%), and a test sample (the remaining 20\% of observations).

\subsection{Datasets}\label{sec:appendix-datasets}

The main characteristics of the datasets are summarized in Table~\ref{tab:real-data-datasets}.

{
\setlength{\tabcolsep}{1.6pt}
\begin{table}[H]
    \centering\small
    \caption{Key characteristics of the datasets}\label{tab:real-data-datasets}
    \begin{tabular}{lcC{1.6cm}C{1.6cm}C{1.6cm}cc}
\toprule
Dataset & \(n\) & No. predictors & No. num. predictors & Prop. target = 1 & Reference & License\\
\midrule
 \texttt{abalone} & 4,177 & 8 & 8 & 0.37 & \citet{misc_abalone_1} & CC BY 4.0\\
 \texttt{adult} & 32,561 & 14 & 6 & 0.24 & \citet{misc_adult_2} & CC BY 4.0\\
 \texttt{bank} & 45,211 & 16 & 7 & 0.12 & \citet{misc_bank_marketing_222} & CC BY 4.0\\
 \texttt{default} & 30,000 & 23 & 14 & 0.22 & \citet{misc_default_of_credit_card_clients_350} & CC BY 4.0\\
 \texttt{drybean} & 13,611 & 16 & 16 & 0.26 & \citet{misc_dry_bean_602} & CC BY 4.0 \\
 \texttt{coupon} & 12,079 & 22 & 0 & 0.57 & \citet{misc_vehicle_coupon_recommendation_603} & CC BY 4.0\\
 \texttt{mushroom} & 8,124 & 21 & 0 & 0.52 & \citet{misc_mushroom_73}& CC BY 4.0\\
 \texttt{occupancy} & 20,560 & 5 & 5 & 0.23 & \citet{misc_occupancy_detection_357} & CC BY 4.0\\
 \texttt{winequality} & 6,495 & 12 & 11 & 0.63 & \citet{misc_wine_quality_186} & CC BY 4.0\\
 \texttt{spambase} & 4,601 & 57 & 57 & 0.39 & \citet{misc_spambase_94} & CC BY 4.0\\
\bottomrule
\end{tabular}
\begin{minipage}{\textwidth}
\vspace{1ex}
\scriptsize\underline{Notes:} $n$ represents the number of observations, 'No. predictors' the total number of predictors, 'No. num. predictors' the number of numeric predictors, and 'Prop. target = 1' the proportion of positive observed events.
\end{minipage}
\end{table}
}

Most of the datasets used are associated with classification tasks. If not, they contain a binary variable suitable for classification or a variable that can be converted into a binary variable. The target variables for each dataset are as follows:
\begin{itemize}
	\item \texttt{abalone}: gender of abalones (1 for male, 0 for female); originally used to predict the size of abalones.
	\item \texttt{adult}: high income (1 if income $\geq$ 50k per year).
	\item \texttt{bank}: subscription to a term deposit (1 if yes, 0 otherwise).
	\item \texttt{default}: default payment (1 if default, 0 otherwise).
	\item \texttt{drybean}: type of dry bean (1 if dermason, 0 otherwise); originally a multi-class variable.
	\item \texttt{coupon}: acceptance of a recommended coupon in different driving scenarios (1 if accepted, 0 otherwise).
	\item \texttt{mushroom}: mushroom classification (1 if edible, 0 otherwise).
	\item \texttt{occupancy}: prediction of room occupancy (1 if occupied, 0 otherwise); originally aimed at predicting the age of occupancy from physical measurements.
	\item \texttt{winequality}: quality of wine (1 if quality $\geq$ 6, 0 otherwise); originally a scale from 0 to 10, with 0 being bad quality and 10 being good quality.
	\item \texttt{spambase}: email classification (1 if spam, 0 otherwise).
\end{itemize}

\subsection{Estimation Results}

\paragraph{Metrics} Table~\ref{tab:results-real-glm} presents the metrics computed on the test set for each dataset where the models used to predict the binary event are either a GLM, a GAM, or a GAMSEL. Table~\ref{tab:real-data-results-glm} and Table~\ref{tab:real-data-results-gam} complement Table~\ref{tab:real-data-results-gamsel} by presenting metrics obtained when the assumed underlying probabilities follow a Beta distribution, with shape parameters estimated using scores from GLM and GAM models, respectively.

{
\setlength{\tabcolsep}{.6pt}
\begin{table}[H]
    \caption{Metrics on real-world dataset for statistical learning models (Test set).}
    \label{tab:results-real-glm}
    \centering\scriptsize
    \begin{tabular}{llccccccllcccccc}
\toprule
Dataset & Model & AUC & Brier & ICI & \(\text{KL}^{\text{GLM}}\) & \(\text{KL}^{\text{GAM}}\) & \(\text{KL}^{\text{GAMSEL}}\) & Dataset & Model & AUC & Brier & ICI & \(\text{KL}^{\text{GLM}}\) & \(\text{KL}^{\text{GAM}}\) & \(\text{KL}^{\text{GAMSEL}}\)\\
\midrule
 & GLM & 0.70 & 0.20 & 0.06 & 0.05 & 0.11 & 0.05 &  & GLM & 0.75 & 0.20 & 0.01 & 0.02 & 0.02 & 0.06\\
\cmidrule{10-16}
 & GAM & 0.71 & 0.20 & 0.02 & 0.32 & 0.13 & 0.30 &  & GAM & 0.75 & 0.20 & 0.01 & 0.02 & 0.02 & 0.06\\
\cmidrule{10-16}
\multirow[t]{-3}{*}{\raggedright\arraybackslash \texttt{abalone}} & GAMSEL & 0.71 & 0.20 & 0.04 & 0.11 & 0.17 & 0.10 & \multirow[t]{-3}{*}{\raggedright\arraybackslash \texttt{coupon}} & GAMSEL & 0.75 & 0.20 & 0.02 & 0.05 & 0.05 & 0.03\\
\cmidrule{1-16}
 & GLM & 0.90 & 0.10 & 0.00 & 0.02 & 0.02 & 0.16 &  & GLM & 1.00 & 0.00 & 0.00 & 0.29 & 0.29 & 1.40\\
\cmidrule{10-16}
 & GAM & 0.91 & 0.10 & 0.01 & 0.05 & 0.03 & 0.23 &  & GAM & 1.00 & 0.00 & 0.00 & 0.29 & 0.29 & 1.40\\
\cmidrule{10-16}
\multirow[t]{-3}{*}{\raggedright\arraybackslash \texttt{adult}} & GAMSEL & 0.90 & 0.11 & 0.03 & 0.06 & 0.10 & 0.05 & \multirow[t]{-3}{*}{\raggedright\arraybackslash \texttt{mushroom}} & GAMSEL & 1.00 & 0.01 & 0.04 & 0.61 & 0.61 & 0.90\\
\cmidrule{1-16}
 & GLM & 0.91 & 0.07 & 0.03 & 0.20 & 0.15 & 0.32 &  & GLM & 1.00 & 0.01 & 0.01 & 0.46 & 0.35 & 0.85\\
\cmidrule{10-16}
 & GAM & 0.92 & 0.07 & 0.02 & 0.20 & 0.14 & 0.34 &  & GAM & 1.00 & 0.01 & 0.01 & 0.50 & 0.38 & 0.92\\
\cmidrule{10-16}
\multirow[t]{-3}{*}{\raggedright\arraybackslash bank} & GAMSEL & 0.91 & 0.07 & 0.03 & 0.13 & 0.10 & 0.20 & \multirow[t]{-3}{*}{\raggedright\arraybackslash \texttt{occupancy}} & GAMSEL & 0.99 & 0.02 & 0.04 & 0.41 & 0.34 & 0.63\\
\cmidrule{1-16}
 & GLM & 0.77 & 0.14 & 0.02 & 0.45 & 0.45 & 0.49 &  & GLM & 0.74 & 0.20 & 0.04 & 0.06 & 0.22 & 0.06\\
\cmidrule{10-16}
 & GAM & 0.78 & 0.13 & 0.01 & 0.29 & 0.28 & 0.35 &  & GAM & 0.79 & 0.18 & 0.04 & 0.14 & 0.02 & 0.21\\
\cmidrule{10-16}
\multirow[t]{-3}{*}{\raggedright\arraybackslash \texttt{default}} & GAMSEL & 0.76 & 0.14 & 0.03 & 0.73 & 0.75 & 0.70 & \multirow[t]{-3}{*}{\raggedright\arraybackslash \texttt{winequality}} & GAMSEL & 0.75 & 0.20 & 0.04 & 0.04 & 0.24 & 0.03\\
\cmidrule{1-16}
 & GLM & 0.99 & 0.03 & 0.01 & 0.06 & 0.05 & 0.57 &  & GLM & 0.97 & 0.06 & 0.02 & 0.02 & 0.15 & 0.37\\
\cmidrule{10-16}
 & GAM & 0.99 & 0.03 & 0.01 & 0.08 & 0.07 & 0.64 &  & GAM & 0.91 & 0.08 & 0.07 & 0.67 & 0.29 & 1.69\\
\cmidrule{10-16}
\multirow[t]{-3}{*}{\raggedright\arraybackslash \texttt{drybean}} & GAMSEL & 0.99 & 0.04 & 0.06 & 0.19 & 0.21 & 0.06 & \multirow[t]{-3}{*}{\raggedright\arraybackslash \texttt{spambase}} & GAMSEL & 0.96 & 0.07 & 0.06 & 0.44 & 0.98 & 0.18\\
\bottomrule
\end{tabular}
\begin{minipage}{\textwidth}
\vspace{1ex}
\scriptsize\underline{Notes:} \(\text{KL}^{\text{GLM}}\), \(\text{KL}^{\text{GAM}}\), and \(\text{KL}^{\text{GAMSEL}}\) refer to the Kullback-Leibler divergence computed using the prior Beta distribution fitted using a GLM, a GAM, or a GAMSEL, respectively.
\end{minipage}
\end{table}
}


{
\setlength{\tabcolsep}{1.7pt}
\begin{table}[H]
    \caption{AUC, Brier score, AIC, or KL divergence with Beta distributed priors estimated with a GLM.}
    \label{tab:real-data-results-glm}
    \centering\tiny
    \begin{tabular}{llcccccccccccccccccccc>{}c>{}c>{}c>{}c}
\toprule
\multicolumn{2}{c}{ } & \multicolumn{5}{c}{AUC*} & \multicolumn{5}{c}{Brier*} & \multicolumn{5}{c}{ICI*} & \multicolumn{9}{c}{KL*} \\
\cmidrule(l{3pt}r{3pt}){3-7} \cmidrule(l{3pt}r{3pt}){8-12} \cmidrule(l{3pt}r{3pt}){13-17} \cmidrule(l{3pt}r{3pt}){18-26}
Dataset & Model & AUC & Brier & ICI & KL & QR & AUC & Brier & ICI & KL & QR & AUC & Brier & ICI & KL & QR & AUC & Brier & ICI & KL & QR & \(\Delta\)AUC & \(\Delta\)Brier & \(\Delta\)ICI & \(\Delta\)KL\\
\midrule
 \texttt{abalone} & RF & 0.71 & 0.20 & 0.03 & 0.34 & 1.21 & 0.71 & 0.20 & 0.03 & 0.34 & 1.24 & 0.51 & 0.23 & 0.02 & 2.73 & 0.00 & 0.71 & 0.20 & 0.03 & 0.33 & 1.28 & \cellcolor[HTML]{FFD6D6}{\textcolor[HTML]{333333}{0.00}} & \cellcolor[HTML]{E9F6E9}{\textcolor[HTML]{1A1A1A}{0.00}} & \cellcolor[HTML]{FFD6D6}{\textcolor[HTML]{333333}{0.00}} & \cellcolor[HTML]{E9F6E9}{\textcolor[HTML]{1A1A1A}{-0.01}}\\
 & XGB & 0.69 & 0.20 & 0.03 & 0.42 & 1.45 & 0.69 & 0.20 & 0.04 & 0.56 & 1.06 & 0.70 & 0.20 & 0.04 & 0.80 & 1.03 & 0.69 & 0.21 & 0.05 & 0.24 & 1.23 & \cellcolor[HTML]{FFD6D6}{\textcolor[HTML]{333333}{0.00}} & \cellcolor[HTML]{FFD6D6}{\textcolor[HTML]{333333}{0.00}} & \cellcolor[HTML]{FFC2C2}{\textcolor[HTML]{2B2B2B}{0.02}} & \cellcolor[HTML]{E9F6E9}{\textcolor[HTML]{1A1A1A}{-0.18}}\\
\cmidrule{1-26}
 \texttt{adult} & RF & 0.92 & 0.10 & 0.03 & 0.03 & 0.88 & 0.92 & 0.10 & 0.03 & 0.02 & 0.89 & 0.51 & 0.18 & 0.00 & 4.46 & 0.00 & 0.92 & 0.10 & 0.03 & 0.02 & 0.89 & \cellcolor[HTML]{FFD6D6}{\textcolor[HTML]{333333}{0.00}} & \cellcolor[HTML]{E9F6E9}{\textcolor[HTML]{1A1A1A}{0.00}} & \cellcolor[HTML]{E9F6E9}{\textcolor[HTML]{1A1A1A}{0.00}} & \cellcolor[HTML]{E9F6E9}{\textcolor[HTML]{1A1A1A}{-0.01}}\\
 & XGB & 0.93 & 0.09 & 0.01 & 0.09 & 1.00 & 0.93 & 0.09 & 0.01 & 0.09 & 1.00 & 0.93 & 0.09 & 0.01 & 0.09 & 0.97 & 0.91 & 0.10 & 0.02 & 0.04 & 0.90 & \cellcolor[HTML]{FFCCCC}{\textcolor[HTML]{333333}{-0.01}} & \cellcolor[HTML]{FFC2C2}{\textcolor[HTML]{2B2B2B}{0.01}} & \cellcolor[HTML]{FFCCCC}{\textcolor[HTML]{333333}{0.01}} & \cellcolor[HTML]{E9F6E9}{\textcolor[HTML]{1A1A1A}{-0.05}}\\
\cmidrule{1-26}
 \texttt{bank} & RF & 0.94 & 0.06 & 0.02 & 0.19 & 1.08 & 0.94 & 0.06 & 0.02 & 0.21 & 1.10 & 0.94 & 0.06 & 0.02 & 0.21 & 1.12 & 0.92 & 0.07 & 0.04 & 0.07 & 0.82 & \cellcolor[HTML]{FFC2C2}{\textcolor[HTML]{2B2B2B}{-0.02}} & \cellcolor[HTML]{FFC2C2}{\textcolor[HTML]{2B2B2B}{0.01}} & \cellcolor[HTML]{FFC2C2}{\textcolor[HTML]{2B2B2B}{0.02}} & \cellcolor[HTML]{E9F6E9}{\textcolor[HTML]{1A1A1A}{-0.12}}\\
 & XGB & 0.93 & 0.06 & 0.02 & 0.36 & 1.17 & 0.93 & 0.06 & 0.02 & 0.28 & 1.12 & 0.93 & 0.06 & 0.02 & 0.34 & 1.15 & 0.91 & 0.07 & 0.03 & 0.07 & 0.93 & \cellcolor[HTML]{FFB8B8}{\textcolor[HTML]{2B2B2B}{-0.02}} & \cellcolor[HTML]{FFCCCC}{\textcolor[HTML]{333333}{0.00}} & \cellcolor[HTML]{FFCCCC}{\textcolor[HTML]{333333}{0.01}} & \cellcolor[HTML]{E9F6E9}{\textcolor[HTML]{1A1A1A}{-0.29}}\\
\cmidrule{1-26}
 \texttt{default} & RF & 0.78 & 0.13 & 0.02 & 0.20 & 1.10 & 0.78 & 0.13 & 0.01 & 0.18 & 1.12 & 0.78 & 0.13 & 0.01 & 0.16 & 1.15 & 0.77 & 0.14 & 0.02 & 0.13 & 1.17 & \cellcolor[HTML]{FFC2C2}{\textcolor[HTML]{2B2B2B}{-0.02}} & \cellcolor[HTML]{FFD6D6}{\textcolor[HTML]{333333}{0.00}} & \cellcolor[HTML]{FFD6D6}{\textcolor[HTML]{333333}{0.00}} & \cellcolor[HTML]{E9F6E9}{\textcolor[HTML]{1A1A1A}{-0.07}}\\
 & XGB & 0.78 & 0.13 & 0.01 & 0.23 & 1.17 & 0.78 & 0.13 & 0.01 & 0.29 & 1.15 & 0.78 & 0.13 & 0.01 & 0.22 & 1.19 & 0.77 & 0.13 & 0.01 & 0.19 & 1.17 & \cellcolor[HTML]{FFCCCC}{\textcolor[HTML]{333333}{-0.01}} & \cellcolor[HTML]{FFD6D6}{\textcolor[HTML]{333333}{0.00}} & \cellcolor[HTML]{FFD6D6}{\textcolor[HTML]{333333}{0.00}} & \cellcolor[HTML]{E9F6E9}{\textcolor[HTML]{1A1A1A}{-0.04}}\\
\cmidrule{1-26}
\texttt{drybean} & RF & 0.99 & 0.03 & 0.01 & 0.06 & 1.00 & 0.99 & 0.03 & 0.01 & 0.07 & 1.00 & 0.99 & 0.03 & 0.01 & 0.06 & 1.00 & 0.99 & 0.03 & 0.02 & 0.02 & 0.98 & \cellcolor[HTML]{FFD6D6}{\textcolor[HTML]{333333}{0.00}} & \cellcolor[HTML]{FFCCCC}{\textcolor[HTML]{333333}{0.00}} & \cellcolor[HTML]{FFCCCC}{\textcolor[HTML]{333333}{0.01}} & \cellcolor[HTML]{E9F6E9}{\textcolor[HTML]{1A1A1A}{-0.04}}\\
 & XGB & 0.99 & 0.03 & 0.01 & 0.08 & 1.00 & 0.99 & 0.03 & 0.01 & 0.09 & 1.00 & 0.99 & 0.03 & 0.01 & 0.09 & 1.00 & 0.99 & 0.03 & 0.04 & 0.07 & 0.92 & \cellcolor[HTML]{FFD6D6}{\textcolor[HTML]{333333}{0.00}} & \cellcolor[HTML]{FFD6D6}{\textcolor[HTML]{333333}{0.00}} & \cellcolor[HTML]{FFB8B8}{\textcolor[HTML]{2B2B2B}{0.03}} & \cellcolor[HTML]{E9F6E9}{\textcolor[HTML]{1A1A1A}{-0.02}}\\
\cmidrule{1-26}
 \texttt{coupon} & RF & 0.83 & 0.17 & 0.07 & 0.04 & 0.98 & 0.83 & 0.17 & 0.07 & 0.04 & 0.98 & 0.51 & 0.24 & 0.00 & 3.60 & 0.00 & 0.83 & 0.17 & 0.07 & 0.04 & 0.98 & \cellcolor[HTML]{FFD6D6}{\textcolor[HTML]{333333}{0.00}} & \cellcolor[HTML]{E9F6E9}{\textcolor[HTML]{1A1A1A}{0.00}} & \cellcolor[HTML]{E9F6E9}{\textcolor[HTML]{1A1A1A}{0.00}} & \cellcolor[HTML]{E9F6E9}{\textcolor[HTML]{1A1A1A}{0.00}}\\
 & XGB & 0.84 & 0.17 & 0.10 & 2.27 & 1.74 & 0.84 & 0.16 & 0.03 & 0.81 & 1.53 & 0.83 & 0.16 & 0.02 & 0.37 & 1.39 & 0.78 & 0.19 & 0.03 & 0.04 & 1.03 & \cellcolor[HTML]{FF8F8F}{\textcolor[HTML]{141414}{-0.06}} & \cellcolor[HTML]{FFADAD}{\textcolor[HTML]{232323}{0.02}} & \cellcolor[HTML]{426F42}{\textcolor[HTML]{E6E6E6}{-0.07}} & \cellcolor[HTML]{81B781}{\textcolor[HTML]{E6E6E6}{-2.23}}\\
\cmidrule{1-26}
 \texttt{mushroom} & RF & 1.00 & 0.01 & 0.05 & 0.23 & 0.96 & 1.00 & 0.00 & 0.01 & 0.22 & 1.00 & 1.00 & 0.00 & 0.01 & 0.22 & 1.00 & 1.00 & 0.01 & 0.04 & 0.11 & 0.99 & \cellcolor[HTML]{FFD6D6}{\textcolor[HTML]{333333}{0.00}} & \cellcolor[HTML]{FFD6D6}{\textcolor[HTML]{333333}{0.00}} & \cellcolor[HTML]{D4F2D4}{\textcolor[HTML]{1A1A1A}{-0.02}} & \cellcolor[HTML]{E9F6E9}{\textcolor[HTML]{1A1A1A}{-0.12}}\\
 & XGB & 1.00 & 0.00 & 0.00 & 0.28 & 1.00 & 1.00 & 0.00 & 0.00 & 0.29 & 1.00 & 1.00 & 0.00 & 0.00 & 0.28 & 1.00 & 1.00 & 0.01 & 0.04 & 0.13 & 0.97 & \cellcolor[HTML]{FFD6D6}{\textcolor[HTML]{333333}{0.00}} & \cellcolor[HTML]{FFC2C2}{\textcolor[HTML]{2B2B2B}{0.01}} & \cellcolor[HTML]{FFB8B8}{\textcolor[HTML]{2B2B2B}{0.03}} & \cellcolor[HTML]{E9F6E9}{\textcolor[HTML]{1A1A1A}{-0.15}}\\
\cmidrule{1-26}
 \texttt{occupancy} & RF & 1.00 & 0.01 & 0.00 & 0.56 & 1.04 & 1.00 & 0.01 & 0.00 & 0.57 & 1.04 & 1.00 & 0.01 & 0.00 & 0.57 & 1.04 & 1.00 & 0.01 & 0.04 & 0.31 & 0.97 & \cellcolor[HTML]{FFD6D6}{\textcolor[HTML]{333333}{0.00}} & \cellcolor[HTML]{FFC2C2}{\textcolor[HTML]{2B2B2B}{0.01}} & \cellcolor[HTML]{FFB8B8}{\textcolor[HTML]{2B2B2B}{0.03}} & \cellcolor[HTML]{E9F6E9}{\textcolor[HTML]{1A1A1A}{-0.25}}\\
 & XGB & 1.00 & 0.01 & 0.01 & 0.60 & 1.04 & 1.00 & 0.01 & 0.01 & 0.66 & 1.04 & 1.00 & 0.01 & 0.00 & 0.54 & 1.03 & 1.00 & 0.01 & 0.04 & 0.47 & 0.95 & \cellcolor[HTML]{FFD6D6}{\textcolor[HTML]{333333}{0.00}} & \cellcolor[HTML]{FFD6D6}{\textcolor[HTML]{333333}{0.00}} & \cellcolor[HTML]{FFADAD}{\textcolor[HTML]{232323}{0.04}} & \cellcolor[HTML]{E9F6E9}{\textcolor[HTML]{1A1A1A}{-0.13}}\\
\cmidrule{1-26}
 \texttt{winequality} & RF & 0.89 & 0.14 & 0.07 & 0.32 & 1.42 & 0.89 & 0.13 & 0.03 & 0.69 & 1.58 & 0.51 & 0.24 & 0.03 & 3.43 & 0.00 & 0.84 & 0.17 & 0.08 & 0.05 & 1.07 & \cellcolor[HTML]{FF9999}{\textcolor[HTML]{1A1A1A}{-0.05}} & \cellcolor[HTML]{FF7A7A}{\textcolor[HTML]{0A0A0A}{0.03}} & \cellcolor[HTML]{FFCCCC}{\textcolor[HTML]{333333}{0.01}} & \cellcolor[HTML]{E9F6E9}{\textcolor[HTML]{1A1A1A}{-0.27}}\\
 & XGB & 0.87 & 0.15 & 0.12 & 4.06 & 1.97 & 0.86 & 0.14 & 0.04 & 1.63 & 1.75 & 0.83 & 0.17 & 0.03 & 0.35 & 1.39 & 0.80 & 0.18 & 0.04 & 0.11 & 1.12 & \cellcolor[HTML]{FF7A7A}{\textcolor[HTML]{0A0A0A}{-0.07}} & \cellcolor[HTML]{FF8F8F}{\textcolor[HTML]{141414}{0.03}} & \cellcolor[HTML]{2F5D2F}{\textcolor[HTML]{E6E6E6}{-0.08}} & \cellcolor[HTML]{2F5D2F}{\textcolor[HTML]{E6E6E6}{-3.96}}\\
\cmidrule{1-26}
 \texttt{spambase} & RF & 0.99 & 0.05 & 0.06 & 0.21 & 0.96 & 0.99 & 0.04 & 0.04 & 0.10 & 0.98 & 0.51 & 0.24 & 0.01 & 6.08 & 0.00 & 0.99 & 0.04 & 0.04 & 0.10 & 0.98 & \cellcolor[HTML]{FFD6D6}{\textcolor[HTML]{333333}{0.00}} & \cellcolor[HTML]{E9F6E9}{\textcolor[HTML]{1A1A1A}{0.00}} & \cellcolor[HTML]{BFEFBF}{\textcolor[HTML]{1A1A1A}{-0.02}} & \cellcolor[HTML]{E9F6E9}{\textcolor[HTML]{1A1A1A}{-0.11}}\\
 & XGB & 0.98 & 0.04 & 0.01 & 0.23 & 1.00 & 0.99 & 0.04 & 0.01 & 0.17 & 1.00 & 0.99 & 0.04 & 0.01 & 0.17 & 1.00 & 0.98 & 0.04 & 0.01 & 0.10 & 0.99 & \cellcolor[HTML]{FFD6D6}{\textcolor[HTML]{333333}{0.00}} & \cellcolor[HTML]{FFC2C2}{\textcolor[HTML]{2B2B2B}{0.01}} & \cellcolor[HTML]{E9F6E9}{\textcolor[HTML]{1A1A1A}{0.00}} & \cellcolor[HTML]{E9F6E9}{\textcolor[HTML]{1A1A1A}{-0.13}}\\
\bottomrule
\end{tabular}
\begin{minipage}{\textwidth}
\vspace{1ex}
\scriptsize\underline{Notes:} RF and XGB denote Random Forest, and Extreme Gradient Boosting, respectively. AUC*/Brier*/ICI* refer to models chosen based on the maximization or minimization of AUC/Brier/ICI on the test set. KL* denote the selection of the best model based on minimizing the Kullback-Leibler divergence between the model's score distributions and the assumed priors of the underlying probabilities, estimated using a GLM model. Columns $\Delta AUC$, $\Delta Brier$, $\Delta ICI$ and $\Delta KL$ display the differences between the AUC/Brier/ICI/KL values from models selected by minimizing the Kullback-Leibler divergence and those maximizing AUC. Negative values indicate a reduction in AUC/Brier/ICI/KL when prioritizing minimization of the Kullback-Leibler divergence over optimization of AUC/Brier/ICI.
\end{minipage}
\end{table}
}

{
\setlength{\tabcolsep}{1.7pt}
\begin{table}[H]
    \caption{AUC, Brier score, AIC, or KL divergence with Beta distributed priors estimated with a GAM.}
    \label{tab:real-data-results-gam}
    \centering\tiny
\begin{tabular}{llcccccccccccccccccccccccc}
\toprule
\multicolumn{2}{c}{ } & \multicolumn{5}{c}{AUC*} & \multicolumn{5}{c}{Brier*} & \multicolumn{5}{c}{ICI*} & \multicolumn{9}{c}{KL*} \\
\cmidrule(l{3pt}r{3pt}){3-7} \cmidrule(l{3pt}r{3pt}){8-12} \cmidrule(l{3pt}r{3pt}){13-17} \cmidrule(l{3pt}r{3pt}){18-26}
Dataset & Model & AUC & Brier & ICI & KL & QR & AUC & Brier & ICI & KL & QR & AUC & Brier & ICI & KL & QR & AUC & Brier & ICI & KL & QR & \(\Delta\)AUC & \(\Delta\)Brier & \(\Delta\)ICI & \(\Delta\)KL\\
\midrule
 \texttt{abalone} & RF & 0.71 & 0.20 & 0.03 & 0.24 & 0.92 & 0.71 & 0.20 & 0.03 & 0.94 & 0.21 & 0.51 & 0.23 & 0.02 & 3.18 & 0.00 & 0.70 & 0.20 & 0.04 & 0.13 & 1.07 & \cellcolor[HTML]{FFCCCC}{\textcolor[HTML]{333333}{-0.01}} & \cellcolor[HTML]{FFCCCC}{\textcolor[HTML]{333333}{0.00}} & \cellcolor[HTML]{FFCCCC}{\textcolor[HTML]{333333}{0.01}} & \cellcolor[HTML]{E9F6E9}{\textcolor[HTML]{1A1A1A}{-0.11}}\\
 & XGB & 0.69 & 0.20 & 0.03 & 0.12 & 1.11 & 0.69 & 0.20 & 0.04 & 0.80 & 0.58 & 0.70 & 0.20 & 0.04 & 0.92 & 0.78 & 0.69 & 0.21 & 0.04 & 0.14 & 1.26 & \cellcolor[HTML]{FFD6D6}{\textcolor[HTML]{333333}{0.00}} & \cellcolor[HTML]{FFCCCC}{\textcolor[HTML]{333333}{0.00}} & \cellcolor[HTML]{FFCCCC}{\textcolor[HTML]{333333}{0.01}} & \cellcolor[HTML]{FFD6D6}{\textcolor[HTML]{333333}{0.02}}\\
\cmidrule{1-26}
 \texttt{adult} & RF & 0.92 & 0.10 & 0.03 & 0.06 & 0.83 & 0.92 & 0.10 & 0.03 & 0.83 & 0.04 & 0.51 & 0.18 & 0.00 & 4.63 & 0.00 & 0.92 & 0.10 & 0.03 & 0.04 & 0.83 & \cellcolor[HTML]{FFD6D6}{\textcolor[HTML]{333333}{0.00}} & \cellcolor[HTML]{E9F6E9}{\textcolor[HTML]{1A1A1A}{0.00}} & \cellcolor[HTML]{E9F6E9}{\textcolor[HTML]{1A1A1A}{0.00}} & \cellcolor[HTML]{E9F6E9}{\textcolor[HTML]{1A1A1A}{-0.01}}\\
 & XGB & 0.93 & 0.09 & 0.01 & 0.06 & 0.93 & 0.93 & 0.09 & 0.01 & 0.93 & 0.06 & 0.93 & 0.09 & 0.01 & 0.06 & 0.91 & 0.92 & 0.09 & 0.01 & 0.05 & 0.88 & \cellcolor[HTML]{FFCCCC}{\textcolor[HTML]{333333}{-0.01}} & \cellcolor[HTML]{FFC2C2}{\textcolor[HTML]{2B2B2B}{0.00}} & \cellcolor[HTML]{FFD6D6}{\textcolor[HTML]{333333}{0.01}} & \cellcolor[HTML]{E9F6E9}{\textcolor[HTML]{1A1A1A}{-0.02}}\\
\cmidrule{1-26}
 \texttt{bank} & RF & 0.94 & 0.06 & 0.02 & 0.14 & 1.09 & 0.94 & 0.06 & 0.02 & 1.11 & 0.15 & 0.94 & 0.06 & 0.02 & 0.15 & 1.13 & 0.92 & 0.07 & 0.04 & 0.05 & 0.87 & \cellcolor[HTML]{FFC2C2}{\textcolor[HTML]{2B2B2B}{-0.01}} & \cellcolor[HTML]{FFB8B8}{\textcolor[HTML]{2B2B2B}{0.01}} & \cellcolor[HTML]{FFC2C2}{\textcolor[HTML]{2B2B2B}{0.02}} & \cellcolor[HTML]{E9F6E9}{\textcolor[HTML]{1A1A1A}{-0.08}}\\
 & XGB & 0.93 & 0.06 & 0.02 & 0.28 & 1.18 & 0.93 & 0.06 & 0.02 & 1.13 & 0.21 & 0.93 & 0.06 & 0.02 & 0.26 & 1.16 & 0.91 & 0.07 & 0.03 & 0.05 & 0.96 & \cellcolor[HTML]{FFB8B8}{\textcolor[HTML]{2B2B2B}{-0.02}} & \cellcolor[HTML]{FFC2C2}{\textcolor[HTML]{2B2B2B}{0.00}} & \cellcolor[HTML]{FFCCCC}{\textcolor[HTML]{333333}{0.01}} & \cellcolor[HTML]{E9F6E9}{\textcolor[HTML]{1A1A1A}{-0.23}}\\
\cmidrule{1-26}
 \texttt{default} & RF & 0.78 & 0.13 & 0.02 & 0.20 & 1.07 & 0.78 & 0.13 & 0.01 & 1.09 & 0.17 & 0.78 & 0.13 & 0.01 & 0.15 & 1.12 & 0.77 & 0.14 & 0.02 & 0.12 & 1.14 & \cellcolor[HTML]{FFC2C2}{\textcolor[HTML]{2B2B2B}{-0.02}} & \cellcolor[HTML]{FFCCCC}{\textcolor[HTML]{333333}{0.00}} & \cellcolor[HTML]{FFD6D6}{\textcolor[HTML]{333333}{0.00}} & \cellcolor[HTML]{E9F6E9}{\textcolor[HTML]{1A1A1A}{-0.08}}\\
 & XGB & 0.78 & 0.13 & 0.01 & 0.22 & 1.14 & 0.78 & 0.13 & 0.01 & 1.12 & 0.29 & 0.78 & 0.13 & 0.01 & 0.20 & 1.16 & 0.77 & 0.13 & 0.01 & 0.17 & 1.14 & \cellcolor[HTML]{FFCCCC}{\textcolor[HTML]{333333}{-0.01}} & \cellcolor[HTML]{FFD6D6}{\textcolor[HTML]{333333}{0.00}} & \cellcolor[HTML]{FFD6D6}{\textcolor[HTML]{333333}{0.00}} & \cellcolor[HTML]{E9F6E9}{\textcolor[HTML]{1A1A1A}{-0.05}}\\
\cmidrule{1-26}
 \texttt{drybean} & RF & 0.99 & 0.03 & 0.01 & 0.05 & 1.00 & 0.99 & 0.03 & 0.01 & 1.00 & 0.06 & 0.99 & 0.03 & 0.01 & 0.05 & 1.00 & 0.99 & 0.03 & 0.02 & 0.02 & 0.98 & \cellcolor[HTML]{FFD6D6}{\textcolor[HTML]{333333}{0.00}} & \cellcolor[HTML]{FFCCCC}{\textcolor[HTML]{333333}{0.00}} & \cellcolor[HTML]{FFCCCC}{\textcolor[HTML]{333333}{0.01}} & \cellcolor[HTML]{E9F6E9}{\textcolor[HTML]{1A1A1A}{-0.03}}\\
 & XGB & 0.99 & 0.03 & 0.01 & 0.07 & 1.00 & 0.99 & 0.03 & 0.01 & 1.00 & 0.08 & 0.99 & 0.03 & 0.01 & 0.08 & 1.00 & 0.99 & 0.03 & 0.02 & 0.05 & 0.97 & \cellcolor[HTML]{FFD6D6}{\textcolor[HTML]{333333}{0.00}} & \cellcolor[HTML]{FFCCCC}{\textcolor[HTML]{333333}{0.00}} & \cellcolor[HTML]{FFD6D6}{\textcolor[HTML]{333333}{0.01}} & \cellcolor[HTML]{E9F6E9}{\textcolor[HTML]{1A1A1A}{-0.02}}\\
\cmidrule{1-26}
 \texttt{coupon} & RF & 0.83 & 0.17 & 0.07 & 0.04 & 0.98 & 0.83 & 0.17 & 0.07 & 0.98 & 0.04 & 0.51 & 0.24 & 0.00 & 3.60 & 0.00 & 0.83 & 0.17 & 0.07 & 0.04 & 0.98 & \cellcolor[HTML]{FFD6D6}{\textcolor[HTML]{333333}{0.00}} & \cellcolor[HTML]{E9F6E9}{\textcolor[HTML]{1A1A1A}{0.00}} & \cellcolor[HTML]{E9F6E9}{\textcolor[HTML]{1A1A1A}{0.00}} & \cellcolor[HTML]{E9F6E9}{\textcolor[HTML]{1A1A1A}{0.00}}\\
 & XGB & 0.84 & 0.17 & 0.10 & 2.27 & 1.74 & 0.84 & 0.16 & 0.03 & 1.53 & 0.81 & 0.83 & 0.16 & 0.02 & 0.37 & 1.39 & 0.78 & 0.19 & 0.03 & 0.04 & 1.03 & \cellcolor[HTML]{FF7A7A}{\textcolor[HTML]{0A0A0A}{-0.06}} & \cellcolor[HTML]{FF8F8F}{\textcolor[HTML]{141414}{0.02}} & \cellcolor[HTML]{426F42}{\textcolor[HTML]{E6E6E6}{-0.07}} & \cellcolor[HTML]{2F5D2F}{\textcolor[HTML]{E6E6E6}{-2.23}}\\
\cmidrule{1-26}
 \texttt{mushroom} & RF & 1.00 & 0.01 & 0.05 & 0.23 & 0.96 & 1.00 & 0.00 & 0.01 & 1.00 & 0.22 & 1.00 & 0.00 & 0.01 & 0.22 & 1.00 & 1.00 & 0.01 & 0.04 & 0.11 & 0.99 & \cellcolor[HTML]{FFD6D6}{\textcolor[HTML]{333333}{0.00}} & \cellcolor[HTML]{FFCCCC}{\textcolor[HTML]{333333}{0.00}} & \cellcolor[HTML]{D4F2D4}{\textcolor[HTML]{1A1A1A}{-0.02}} & \cellcolor[HTML]{E9F6E9}{\textcolor[HTML]{1A1A1A}{-0.12}}\\
 & XGB & 1.00 & 0.00 & 0.00 & 0.28 & 1.00 & 1.00 & 0.00 & 0.00 & 1.00 & 0.29 & 1.00 & 0.00 & 0.00 & 0.28 & 1.00 & 1.00 & 0.01 & 0.04 & 0.13 & 0.97 & \cellcolor[HTML]{FFD6D6}{\textcolor[HTML]{333333}{0.00}} & \cellcolor[HTML]{FFB8B8}{\textcolor[HTML]{2B2B2B}{0.01}} & \cellcolor[HTML]{FFB8B8}{\textcolor[HTML]{2B2B2B}{0.03}} & \cellcolor[HTML]{E9F6E9}{\textcolor[HTML]{1A1A1A}{-0.15}}\\
\cmidrule{1-26}
 \texttt{occupancy} & RF & 1.00 & 0.01 & 0.00 & 0.43 & 1.03 & 1.00 & 0.01 & 0.00 & 1.03 & 0.44 & 1.00 & 0.01 & 0.00 & 0.44 & 1.03 & 1.00 & 0.01 & 0.02 & 0.26 & 0.99 & \cellcolor[HTML]{FFD6D6}{\textcolor[HTML]{333333}{0.00}} & \cellcolor[HTML]{FFC2C2}{\textcolor[HTML]{2B2B2B}{0.00}} & \cellcolor[HTML]{FFC2C2}{\textcolor[HTML]{2B2B2B}{0.02}} & \cellcolor[HTML]{E9F6E9}{\textcolor[HTML]{1A1A1A}{-0.17}}\\
 & XGB & 1.00 & 0.01 & 0.01 & 0.47 & 1.03 & 1.00 & 0.01 & 0.01 & 1.03 & 0.52 & 1.00 & 0.01 & 0.00 & 0.41 & 1.02 & 1.00 & 0.01 & 0.04 & 0.37 & 0.94 & \cellcolor[HTML]{FFD6D6}{\textcolor[HTML]{333333}{0.00}} & \cellcolor[HTML]{FFCCCC}{\textcolor[HTML]{333333}{0.00}} & \cellcolor[HTML]{FFADAD}{\textcolor[HTML]{232323}{0.04}} & \cellcolor[HTML]{E9F6E9}{\textcolor[HTML]{1A1A1A}{-0.10}}\\
\cmidrule{1-26}
 \texttt{winequality} & RF & 0.89 & 0.14 & 0.07 & 0.04 & 1.10 & 0.89 & 0.13 & 0.03 & 1.23 & 0.12 & 0.51 & 0.24 & 0.03 & 3.95 & 0.00 & 0.86 & 0.15 & 0.06 & 0.03 & 1.02 & \cellcolor[HTML]{FFA3A3}{\textcolor[HTML]{1F1F1F}{-0.03}} & \cellcolor[HTML]{FF8585}{\textcolor[HTML]{101010}{0.02}} & \cellcolor[HTML]{D4F2D4}{\textcolor[HTML]{1A1A1A}{-0.01}} & \cellcolor[HTML]{E9F6E9}{\textcolor[HTML]{1A1A1A}{-0.01}}\\
 & XGB & 0.87 & 0.15 & 0.12 & 1.91 & 1.53 & 0.86 & 0.14 & 0.04 & 1.36 & 0.53 & 0.83 & 0.17 & 0.03 & 0.04 & 1.08 & 0.82 & 0.17 & 0.03 & 0.04 & 1.00 & \cellcolor[HTML]{FF7A7A}{\textcolor[HTML]{0A0A0A}{-0.06}} & \cellcolor[HTML]{FF7A7A}{\textcolor[HTML]{0A0A0A}{0.02}} & \cellcolor[HTML]{2F5D2F}{\textcolor[HTML]{E6E6E6}{-0.08}} & \cellcolor[HTML]{578252}{\textcolor[HTML]{E6E6E6}{-1.87}}\\
\cmidrule{1-26}
 \texttt{spambase} & RF & 0.99 & 0.05 & 0.06 & 0.63 & 0.96 & 0.99 & 0.04 & 0.04 & 0.98 & 0.37 & 0.51 & 0.24 & 0.01 & 7.17 & 0.00 & 0.99 & 0.04 & 0.04 & 0.37 & 0.98 & \cellcolor[HTML]{FFD6D6}{\textcolor[HTML]{333333}{0.00}} & \cellcolor[HTML]{D4F2D4}{\textcolor[HTML]{1A1A1A}{0.00}} & \cellcolor[HTML]{BFEFBF}{\textcolor[HTML]{1A1A1A}{-0.02}} & \cellcolor[HTML]{D4F2D4}{\textcolor[HTML]{1A1A1A}{-0.26}}\\
 & XGB & 0.98 & 0.04 & 0.01 & 0.03 & 1.00 & 0.99 & 0.04 & 0.01 & 1.00 & 0.03 & 0.99 & 0.04 & 0.01 & 0.03 & 1.00 & 0.98 & 0.04 & 0.02 & 0.04 & 1.00 & \cellcolor[HTML]{FFD6D6}{\textcolor[HTML]{333333}{0.00}} & \cellcolor[HTML]{FFCCCC}{\textcolor[HTML]{333333}{0.00}} & \cellcolor[HTML]{FFD6D6}{\textcolor[HTML]{333333}{0.00}} & \cellcolor[HTML]{FFD6D6}{\textcolor[HTML]{333333}{0.00}}\\
\bottomrule
\end{tabular}
\begin{minipage}{\textwidth}
\vspace{1ex}
\scriptsize\underline{Notes:} RF and XGB denote Random Forest, and Extreme Gradient Boosting, respectively. AUC*/Brier*/ICI* refer to models chosen based on the maximization or minimization of AUC/Brier/ICI on the test set. KL* denote the selection of the best model based on minimizing the Kullback-Leibler divergence between the model's score distributions and the assumed priors of the underlying probabilities, estimated using a GAM model. Columns $\Delta AUC$, $\Delta Brier$, $\Delta ICI$ and $\Delta KL$ display the differences between the AUC/Brier/ICI/KL values from models selected by minimizing the Kullback-Leibler divergence and those maximizing AUC. Negative values indicate a reduction in AUC/Brier/ICI/KL when prioritizing minimization of the Kullback-Leibler divergence over optimization of AUC/Brier/ICI.
\end{minipage}
\end{table}
}

{
\setlength{\tabcolsep}{1.7pt}
\begin{table}[ht!]
    \caption{AUC, Brier score, AIC, or KL divergence with Beta distributed priors estimated with a GAMSEL.}
\label{tab:real-data-results-gamsel}
    \centering\tiny
\begin{tabular}{llcccccccccccccccccccccccc}
\toprule
\multicolumn{2}{c}{ } & \multicolumn{5}{c}{AUC*} & \multicolumn{5}{c}{Brier*} & \multicolumn{5}{c}{ICI*} & \multicolumn{9}{c}{KL*} \\
\cmidrule(l{3pt}r{3pt}){3-7} \cmidrule(l{3pt}r{3pt}){8-12} \cmidrule(l{3pt}r{3pt}){13-17} \cmidrule(l{3pt}r{3pt}){18-26}
Dataset & Model & AUC & Brier & ICI & KL & QR & AUC & Brier & ICI & KL & QR & AUC & Brier & ICI & KL & QR & AUC & Brier & ICI & KL & QR & \(\Delta\)AUC & \(\Delta\)Brier & \(\Delta\)ICI & \(\Delta\)KL\\
\midrule
 \texttt{abalone} & RF & 0.71 & 0.20 & 0.03 & 0.33 & 1.21 & 0.71 & 0.20 & 0.03 & 0.33 & 1.23 & 0.51 & 0.23 & 0.02 & 2.74 & 0.00 & 0.71 & 0.20 & 0.03 & 0.32 & 1.28 & \cellcolor[HTML]{FFD6D6}{\textcolor[HTML]{333333}{0.00}} & \cellcolor[HTML]{E9F6E9}{\textcolor[HTML]{1A1A1A}{0.00}} & \cellcolor[HTML]{FFD6D6}{\textcolor[HTML]{333333}{0.00}} & \cellcolor[HTML]{E9F6E9}{\textcolor[HTML]{1A1A1A}{-0.01}}\\
 & XGB & 0.69 & 0.20 & 0.03 & 0.40 & 1.45 & 0.69 & 0.20 & 0.04 & 0.56 & 1.06 & 0.70 & 0.20 & 0.04 & 0.80 & 1.02 & 0.69 & 0.21 & 0.05 & 0.24 & 1.23 & \cellcolor[HTML]{FFD6D6}{\textcolor[HTML]{333333}{0.00}} & \cellcolor[HTML]{FFD6D6}{\textcolor[HTML]{333333}{0.00}} & \cellcolor[HTML]{FFC2C2}{\textcolor[HTML]{2B2B2B}{0.02}} & \cellcolor[HTML]{E9F6E9}{\textcolor[HTML]{1A1A1A}{-0.16}}\\
\cmidrule{1-26}
  \texttt{adult} & RF & 0.92 & 0.10 & 0.03 & 0.08 & 1.09 & 0.92 & 0.10 & 0.03 & 0.09 & 1.10 & 0.51 & 0.18 & 0.00 & 3.96 & 0.00 & 0.91 & 0.10 & 0.05 & 0.04 & 0.98 & \cellcolor[HTML]{FFD6D6}{\textcolor[HTML]{333333}{-0.01}} & \cellcolor[HTML]{FFCCCC}{\textcolor[HTML]{333333}{0.01}} & \cellcolor[HTML]{FFC2C2}{\textcolor[HTML]{2B2B2B}{0.02}} & \cellcolor[HTML]{E9F6E9}{\textcolor[HTML]{1A1A1A}{-0.04}}\\
 & XGB & 0.93 & 0.09 & 0.01 & 0.35 & 1.23 & 0.93 & 0.09 & 0.01 & 0.35 & 1.23 & 0.93 & 0.09 & 0.01 & 0.34 & 1.20 & 0.91 & 0.10 & 0.03 & 0.09 & 1.04 & \cellcolor[HTML]{FFC2C2}{\textcolor[HTML]{2B2B2B}{-0.02}} & \cellcolor[HTML]{FFB8B8}{\textcolor[HTML]{2B2B2B}{0.01}} & \cellcolor[HTML]{FFB8B8}{\textcolor[HTML]{2B2B2B}{0.02}} & \cellcolor[HTML]{E9F6E9}{\textcolor[HTML]{1A1A1A}{-0.26}}\\
\cmidrule{1-26}
  \texttt{bank} & RF & 0.94 & 0.06 & 0.02 & 0.31 & 1.30 & 0.94 & 0.06 & 0.02 & 0.34 & 1.32 & 0.94 & 0.06 & 0.02 & 0.35 & 1.34 & 0.91 & 0.07 & 0.05 & 0.06 & 0.87 & \cellcolor[HTML]{FFB8B8}{\textcolor[HTML]{2B2B2B}{-0.03}} & \cellcolor[HTML]{FFC2C2}{\textcolor[HTML]{2B2B2B}{0.01}} & \cellcolor[HTML]{FFB8B8}{\textcolor[HTML]{2B2B2B}{0.03}} & \cellcolor[HTML]{E9F6E9}{\textcolor[HTML]{1A1A1A}{-0.25}}\\
 & XGB & 0.93 & 0.06 & 0.02 & 0.57 & 1.40 & 0.93 & 0.06 & 0.02 & 0.46 & 1.34 & 0.93 & 0.06 & 0.02 & 0.54 & 1.38 & 0.89 & 0.07 & 0.04 & 0.07 & 0.80 & \cellcolor[HTML]{FFA3A3}{\textcolor[HTML]{1F1F1F}{-0.04}} & \cellcolor[HTML]{FFC2C2}{\textcolor[HTML]{2B2B2B}{0.01}} & \cellcolor[HTML]{FFC2C2}{\textcolor[HTML]{2B2B2B}{0.02}} & \cellcolor[HTML]{D4F2D4}{\textcolor[HTML]{1A1A1A}{-0.50}}\\
\cmidrule{1-26}
  \texttt{default} & RF & 0.78 & 0.13 & 0.02 & 0.22 & 1.24 & 0.78 & 0.13 & 0.01 & 0.23 & 1.27 & 0.78 & 0.13 & 0.01 & 0.24 & 1.30 & 0.78 & 0.14 & 0.03 & 0.20 & 1.07 & \cellcolor[HTML]{FFD6D6}{\textcolor[HTML]{333333}{0.00}} & \cellcolor[HTML]{FFD6D6}{\textcolor[HTML]{333333}{0.00}} & \cellcolor[HTML]{FFCCCC}{\textcolor[HTML]{333333}{0.01}} & \cellcolor[HTML]{E9F6E9}{\textcolor[HTML]{1A1A1A}{-0.02}}\\
 & XGB & 0.78 & 0.13 & 0.01 & 0.34 & 1.32 & 0.78 & 0.13 & 0.01 & 0.36 & 1.30 & 0.78 & 0.13 & 0.01 & 0.35 & 1.35 & 0.79 & 0.13 & 0.01 & 0.34 & 1.28 & \cellcolor[HTML]{2F5D2F}{\textcolor[HTML]{E6E6E6}{0.00}} & \cellcolor[HTML]{E9F6E9}{\textcolor[HTML]{1A1A1A}{0.00}} & \cellcolor[HTML]{FFD6D6}{\textcolor[HTML]{333333}{0.00}} & \cellcolor[HTML]{E9F6E9}{\textcolor[HTML]{1A1A1A}{0.00}}\\
\cmidrule{1-26}
  \texttt{drybean} & RF & 0.99 & 0.03 & 0.01 & 0.56 & 1.17 & 0.99 & 0.03 & 0.01 & 0.59 & 1.17 & 0.99 & 0.03 & 0.01 & 0.57 & 1.17 & 0.99 & 0.04 & 0.05 & 0.12 & 1.08 & \cellcolor[HTML]{FFD6D6}{\textcolor[HTML]{333333}{0.00}} & \cellcolor[HTML]{FFB8B8}{\textcolor[HTML]{2B2B2B}{0.01}} & \cellcolor[HTML]{FFA3A3}{\textcolor[HTML]{1F1F1F}{0.04}} & \cellcolor[HTML]{E9F6E9}{\textcolor[HTML]{1A1A1A}{-0.44}}\\
 & XGB & 0.99 & 0.03 & 0.01 & 0.62 & 1.16 & 0.99 & 0.03 & 0.01 & 0.63 & 1.17 & 0.99 & 0.03 & 0.01 & 0.65 & 1.17 & 0.99 & 0.03 & 0.03 & 0.37 & 1.09 & \cellcolor[HTML]{FFD6D6}{\textcolor[HTML]{333333}{0.00}} & \cellcolor[HTML]{FFCCCC}{\textcolor[HTML]{333333}{0.00}} & \cellcolor[HTML]{FFC2C2}{\textcolor[HTML]{2B2B2B}{0.02}} & \cellcolor[HTML]{E9F6E9}{\textcolor[HTML]{1A1A1A}{-0.25}}\\
\cmidrule{1-26}
 \texttt{coupon} & RF & 0.83 & 0.17 & 0.07 & 0.04 & 1.10 & 0.83 & 0.17 & 0.07 & 0.04 & 1.10 & 0.51 & 0.24 & 0.00 & 3.40 & 0.00 & 0.82 & 0.18 & 0.07 & 0.03 & 1.02 & \cellcolor[HTML]{FFCCCC}{\textcolor[HTML]{333333}{-0.01}} & \cellcolor[HTML]{FFCCCC}{\textcolor[HTML]{333333}{0.01}} & \cellcolor[HTML]{FFD6D6}{\textcolor[HTML]{333333}{0.00}} & \cellcolor[HTML]{E9F6E9}{\textcolor[HTML]{1A1A1A}{-0.01}}\\
 & XGB & 0.84 & 0.17 & 0.10 & 3.15 & 1.94 & 0.84 & 0.16 & 0.03 & 1.27 & 1.72 & 0.83 & 0.16 & 0.02 & 0.66 & 1.55 & 0.77 & 0.19 & 0.03 & 0.05 & 1.07 & \cellcolor[HTML]{FF8585}{\textcolor[HTML]{101010}{-0.07}} & \cellcolor[HTML]{FFA3A3}{\textcolor[HTML]{1F1F1F}{0.02}} & \cellcolor[HTML]{426F42}{\textcolor[HTML]{E6E6E6}{-0.07}} & \cellcolor[HTML]{6CA56C}{\textcolor[HTML]{E6E6E6}{-3.10}}\\
\cmidrule{1-26}
  \texttt{mushroom} & RF & 1.00 & 0.01 & 0.05 & 0.65 & 0.98 & 1.00 & 0.00 & 0.01 & 1.28 & 1.02 & 1.00 & 0.00 & 0.01 & 1.28 & 1.02 & 1.00 & 0.03 & 0.07 & 0.41 & 0.97 & \cellcolor[HTML]{FFD6D6}{\textcolor[HTML]{333333}{0.00}} & \cellcolor[HTML]{FF9999}{\textcolor[HTML]{1A1A1A}{0.02}} & \cellcolor[HTML]{FFC2C2}{\textcolor[HTML]{2B2B2B}{0.02}} & \cellcolor[HTML]{E9F6E9}{\textcolor[HTML]{1A1A1A}{-0.24}}\\
 & XGB & 1.00 & 0.00 & 0.00 & 1.40 & 1.02 & 1.00 & 0.00 & 0.00 & 1.40 & 1.02 & 1.00 & 0.00 & 0.00 & 1.40 & 1.02 & 1.00 & 0.02 & 0.05 & 0.57 & 0.95 & \cellcolor[HTML]{FFD6D6}{\textcolor[HTML]{333333}{0.00}} & \cellcolor[HTML]{FFADAD}{\textcolor[HTML]{232323}{0.02}} & \cellcolor[HTML]{FF9999}{\textcolor[HTML]{1A1A1A}{0.05}} & \cellcolor[HTML]{D4F2D4}{\textcolor[HTML]{1A1A1A}{-0.83}}\\
\cmidrule{1-26}
  \texttt{occupancy} & RF & 1.00 & 0.01 & 0.00 & 1.02 & 1.15 & 1.00 & 0.01 & 0.00 & 1.02 & 1.15 & 1.00 & 0.01 & 0.00 & 1.02 & 1.15 & 1.00 & 0.02 & 0.07 & 0.37 & 1.02 & \cellcolor[HTML]{FFD6D6}{\textcolor[HTML]{333333}{0.00}} & \cellcolor[HTML]{FFB8B8}{\textcolor[HTML]{2B2B2B}{0.02}} & \cellcolor[HTML]{FF8585}{\textcolor[HTML]{101010}{0.07}} & \cellcolor[HTML]{D4F2D4}{\textcolor[HTML]{1A1A1A}{-0.65}}\\
 & XGB & 1.00 & 0.01 & 0.01 & 1.07 & 1.15 & 1.00 & 0.01 & 0.01 & 1.14 & 1.15 & 1.00 & 0.01 & 0.00 & 0.97 & 1.14 & 1.00 & 0.01 & 0.04 & 0.82 & 1.05 & \cellcolor[HTML]{FFD6D6}{\textcolor[HTML]{333333}{0.00}} & \cellcolor[HTML]{FFD6D6}{\textcolor[HTML]{333333}{0.00}} & \cellcolor[HTML]{FFADAD}{\textcolor[HTML]{232323}{0.04}} & \cellcolor[HTML]{E9F6E9}{\textcolor[HTML]{1A1A1A}{-0.24}}\\
\cmidrule{1-26}
 \texttt{winequality} & RF & 0.89 & 0.14 & 0.07 & 0.44 & 1.51 & 0.89 & 0.13 & 0.03 & 0.88 & 1.68 & 0.51 & 0.24 & 0.03 & 3.35 & 0.00 & 0.83 & 0.17 & 0.08 & 0.05 & 1.05 & \cellcolor[HTML]{FF8F8F}{\textcolor[HTML]{141414}{-0.06}} & \cellcolor[HTML]{FF7A7A}{\textcolor[HTML]{0A0A0A}{0.04}} & \cellcolor[HTML]{FFCCCC}{\textcolor[HTML]{333333}{0.01}} & \cellcolor[HTML]{E9F6E9}{\textcolor[HTML]{1A1A1A}{-0.38}}\\
 & XGB & 0.87 & 0.15 & 0.12 & 4.66 & 2.10 & 0.86 & 0.14 & 0.04 & 1.95 & 1.86 & 0.83 & 0.17 & 0.03 & 0.47 & 1.48 & 0.80 & 0.18 & 0.04 & 0.13 & 1.12 & \cellcolor[HTML]{FF7A7A}{\textcolor[HTML]{0A0A0A}{-0.08}} & \cellcolor[HTML]{FF8F8F}{\textcolor[HTML]{141414}{0.03}} & \cellcolor[HTML]{2F5D2F}{\textcolor[HTML]{E6E6E6}{-0.07}} & \cellcolor[HTML]{2F5D2F}{\textcolor[HTML]{E6E6E6}{-4.53}}\\
\cmidrule{1-26}
  \texttt{spambase} & RF & 0.99 & 0.05 & 0.06 & 0.17 & 1.00 & 0.99 & 0.04 & 0.04 & 0.26 & 1.03 & 0.51 & 0.24 & 0.01 & 4.96 & 0.00 & 0.97 & 0.07 & 0.08 & 0.07 & 0.95 & \cellcolor[HTML]{FFCCCC}{\textcolor[HTML]{333333}{-0.01}} & \cellcolor[HTML]{FF9999}{\textcolor[HTML]{1A1A1A}{0.02}} & \cellcolor[HTML]{FFC2C2}{\textcolor[HTML]{2B2B2B}{0.02}} & \cellcolor[HTML]{E9F6E9}{\textcolor[HTML]{1A1A1A}{-0.10}}\\
 & XGB & 0.98 & 0.04 & 0.01 & 1.01 & 1.05 & 0.99 & 0.04 & 0.01 & 0.88 & 1.05 & 0.99 & 0.04 & 0.01 & 0.87 & 1.05 & 0.98 & 0.05 & 0.04 & 0.27 & 1.00 & \cellcolor[HTML]{FFD6D6}{\textcolor[HTML]{333333}{-0.01}} & \cellcolor[HTML]{FFADAD}{\textcolor[HTML]{232323}{0.02}} & \cellcolor[HTML]{FFC2C2}{\textcolor[HTML]{2B2B2B}{0.02}} & \cellcolor[HTML]{D4F2D4}{\textcolor[HTML]{1A1A1A}{-0.75}}\\
\bottomrule
\end{tabular}
\begin{minipage}{\textwidth}
\vspace{1ex}
\scriptsize\underline{Notes:} RF and XGB denote Random Forest, and Extreme Gradient Boosting, respectively. AUC*/Brier*/ICI* refer to models chosen based on the maximization or minimization of AUC/Brier/ICI on the test set. KL* denote the selection of the best model based on minimizing the Kullback-Leibler divergence between the model's score distributions and the assumed priors of the underlying probabilities, estimated using a GAMSEL model. Columns $\Delta AUC$, $\Delta Brier$, $\Delta ICI$ and $\Delta KL$ display the differences between the AUC/Brier/ICI/KL values from models selected by minimizing the Kullback-Leibler divergence and those maximizing AUC. Negative values indicate a reduction in AUC/Brier/ICI/KL when prioritizing minimization of the Kullback-Leibler divergence over optimization of AUC/Brier/ICI. QR refers to the ratio of the difference between the 90th and 10th percentiles of the scores to the difference between the 90th and 10th percentiles of the assumed Beta distribution.
\end{minipage}
\end{table}
}

\paragraph{Distribution of Scores} The distribution of scores under each prior distribution as well as that estimated with the RF and the XGBoost on all datasets are reported when using a GLM to estimate the prior distribution shapes (Figs.~\ref{fig:real-glm-1-5} and \ref{fig:real-glm-6-10}), a GAM (Figs.~\ref{fig:real-gam-1-5} and \ref{fig:real-gam-6-10}), or a GAMSEL (Figs.~\ref{fig:real-gamsel-1-5} and \ref{fig:real-gamsel-6-10}).

\foreach \priormodelfile/\priormodelname in {glm/GLM,gam/GAM,gamsel/GAMSEL}
{
    \begin{figure}[H]
        \centering
        \includegraphics[width=.65\textwidth]{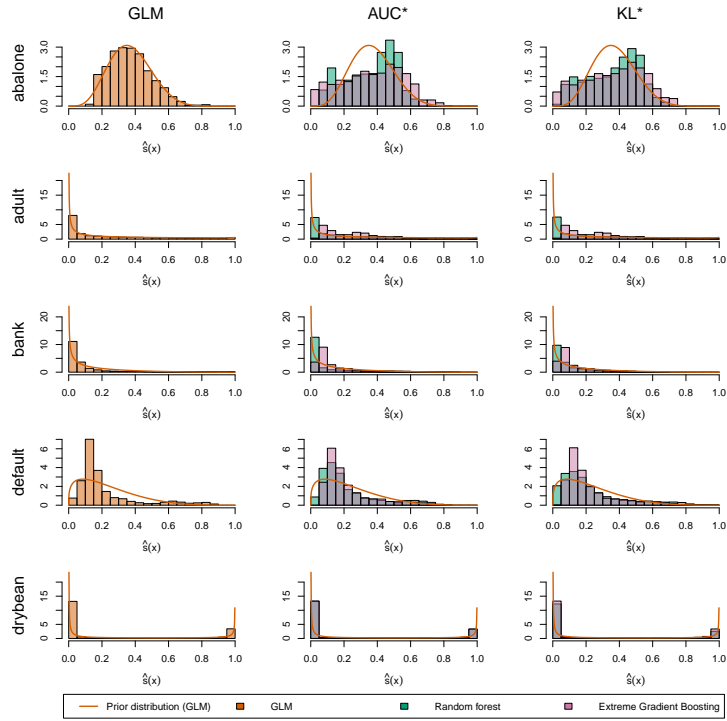}
        \caption{Estimated scores on the first five datasets: \textbf{\priormodelname} (left), models selected by AUC maximization (middle), and KL divergence minimization relative to prior assumptions (right).}
        \label{fig:real-\priormodelfile-1-5}
    \end{figure}
    
    \begin{figure}[H]
        \centering
        \includegraphics[width=.65\textwidth]{figures/real-\priormodelfile-6-10.pdf}
        \caption{Estimated scores on the last five datasets: \textbf{\priormodelname} (left), models selected by AUC maximization (middle), and KL divergence minimization relative to prior assumptions (right).}
        \label{fig:real-\priormodelfile-6-10}
    \end{figure}
}

\end{document}